\documentclass[a4paper,11pt]{article}

\usepackage{amsmath,amsfonts,bm}

\def\eqref#1{equation~\ref{#1}}
\def\1{\bm{1}}

\def\rf{{\textnormal{f}}}

\DeclareMathAlphabet{\mathsfit}{\encodingdefault}{\sfdefault}{m}{sl}
\SetMathAlphabet{\mathsfit}{bold}{\encodingdefault}{\sfdefault}{bx}{n}

\newcommand{\E}{\mathbb{E}}

\newcommand{\R}{\mathbb{R}}

\DeclareMathOperator{\sign}{sign}

\usepackage{geometry}
\usepackage[dvipsnames]{xcolor}

\usepackage[utf8]{inputenc} % allow utf-8 input
\usepackage[T1]{fontenc}    % use 8-bit T1 fonts
\usepackage{url}            % simple URL typesetting
\usepackage{booktabs}       % professional-quality tables
\usepackage{amsfonts}       % blackboard math symbols
\usepackage{nicefrac}       % compact symbols for 1/2, etc.
\usepackage{microtype}      % microtypography
\usepackage{xcolor}         % colors

\usepackage[parfill]{parskip}

\usepackage{microtype}
\usepackage{graphicx}
\usepackage{subfig}
\usepackage{booktabs} % for professional tables
\usepackage{tcolorbox}
\usepackage{wrapfig}

\usepackage{amsmath}
\usepackage{amssymb}
\usepackage{mathtools}
\usepackage{amsthm}
\usepackage{rotating}
\usepackage{comment}
\usepackage{enumerate}
\usepackage{enumitem}

\usepackage{color}
\definecolor{darkgreen}{rgb}{0.0,0.5,0.0}

\usepackage[pagebackref]{hyperref}
\hypersetup{
	colorlinks=true,        % false: boxed links; true: colored links
	linkcolor=blue,         % color of internal links
	filecolor=magenta,
	citecolor=darkgreen,         % color of links to bibliography
	urlcolor=blue           % color of external links
}
\renewcommand*{\backrefalt}[4]{%
    \ifcase #1 \footnotesize{(Not cited.)}%
    \or        \footnotesize{(Cited on page~#2)}%
    \else      \footnotesize{(Cited on pages~#2)}%
    \fi}

\usepackage[capitalize,noabbrev]{cleveref}

\usepackage{algorithm}
\usepackage{algorithmic}
\usepackage{authblk}
\usepackage{natbib}

\theoremstyle{plain}
\newtheorem{theorem}{Theorem}[section]
\newtheorem{proposition}[theorem]{Proposition}
\newtheorem{lemma}[theorem]{Lemma}
\newtheorem{corollary}[theorem]{Corollary}
\theoremstyle{definition}
\newtheorem{definition}[theorem]{Definition}
\newtheorem{assumption}[theorem]{Assumption}
\theoremstyle{remark}
\newtheorem{remark}[theorem]{Remark}
\usepackage{enumitem}
\newcommand{\tr}{\text{Tr}}

\DeclareMathOperator{\diag}{diag}

\def \lf{\left\lfloor}   
\def \rf{\right\rfloor}

\newcommand{\eqdef}{\coloneqq}

\newtcolorbox{mybox}[2][]{
  colframe = white, % Set the frame color to white (transparent)
  colback  = gray!7,
  #1
}

\def\toptitlebar{\hrule height1pt \vskip .25in} 
\def\bottomtitlebar{\vskip .22in \hrule height1pt \vskip .3in}

\title{
\toptitlebar
{{\center\baselineskip 18pt
                      {\Large\bf Unbiased and Sign Compression in Distributed Learning: Comparing Noise Resilience via SDEs}}
} 
\bottomtitlebar}
\date{}

\author[1]{\hspace{2.0cm} Enea Monzio Compagnoni}
\author[1]{Rustem Islamov}
\author[2]{\newline  Frank Norbert Proske}
\author[1]{Aurelien Lucchi}

\affil[1]{University of Basel, Switzerland}
\affil[2]{University of Oslo, Norway}

\begin{document}

\maketitle

\begin{abstract}

\noindent
Distributed methods are essential for handling machine learning pipelines comprising large-scale models and datasets. However, their benefits often come at the cost of increased communication overhead between the central server and agents, which can become the main bottleneck, making training costly or even unfeasible in such systems. Compression methods such as quantization and sparsification can alleviate this issue. Still, their robustness to large and heavy-tailed gradient noise, a phenomenon sometimes observed in language modeling, remains poorly understood. This work addresses this gap by analyzing Distributed Compressed SGD (DCSGD) and Distributed SignSGD (DSignSGD) using stochastic differential equations (SDEs). Our results show that DCSGD with unbiased compression is more vulnerable to noise in stochastic gradients, while DSignSGD remains robust, even under large and heavy-tailed noise. Additionally, we propose new scaling rules for hyperparameter tuning to mitigate performance degradation due to compression. These findings are empirically validated across multiple deep learning architectures and datasets, providing practical recommendations for distributed optimization.

\end{abstract}

%%%%%%%%%%%%%%%%%%%%%%%%%%%%%%%%%%%%%%%%%%%%%%%%%%%%%%%%%%%%%%%%%%%%%%%%%%%%
\section{Introduction} \label{sec:intro}
%%%%%%%%%%%%%%%%%%%%%%%%%%%%%%%%%%%%%%%%%%%%%%%%%%%%%%%%%%%%%%%%%%%%%%%%%%%%

Recent advancements in deep learning have been fueled by the development of larger, increasingly complex models on constantly growing datasets. However, this progress comes at the expense of extended training times and resources. Therefore, distributed training has gained popularity as a way to reduce training time \citep{dean2012large, abadi2016tensorflow}. In this framework, the data is distributed among several agents or machines that collaboratively train a model being orchestrated by a server. The objective function can be expressed as an average of $N$ functions: $ \min_{x\in\R^d} \left[f(x) \eqdef \frac{1}{N}\sum_{i=1}^N f_i(x)\right]$, where $f_i$ represents a loss over the local data of the $i$-th agent, and $x$ are the trainable parameters. Although computational resources are rapidly improving \citep{jouppi2017datacenter},  the synchronization between the server and agents is still a critical performance bottleneck and can significantly increase training time \citep{sapio2021scaling}. Among others, approaches such as communication compression \citep{seide20141, alistarh2018convergence, mishchenko2024distributed}, local computations \citep{gorbunov2021local, koloskova2020unified}, and asynchronous communication \citep{islamov2024asgrad, mishchenko2022asynchronous} are designed to boost the efficacy of distributed training.

We focus on algorithms that utilize lossy compression: They trade off some precision in the communication for reduced bandwidth usage, thereby speeding up the overall learning process. Despite the loss in precision, many compression algorithms are designed to ensure that the learning process converges to an optimal solution, often with guarantees on the convergence rate \citep{mishchenko2024distributed, richtarik2021ef21, gao2023econtrol}. Compression operators can be divided into two main categories: \textit{unbiased} (e.g., sparsification \citep{khirirat2018distributed} and quantization \citep{horvoth2022natural}) and \textit{biased} (e.g., sign \citep{bernstein2018signsgd, safaryan2021signsgd}, Top-$k$ \citep{alistarh2018convergence, beznosikov2023biased}, and low rank \citep{vogels2019powersgd, safaryan2021fednl, islamov2023distributed, qian2021basis}). While the first class is theoretically better understood \citep{khirirat2018distributed, horvath2023stochastic, mishchenko2024distributed, gorbunov2020unified, condat2024locodl} and the latter is often empirically
superior \citep{seide20141}, a theoretical understanding of how these two categories differ fundamentally remains unclear, particularly regarding their behavior w.r.t. the hyperparameters of the optimizer, or their robustness to large or heavy-tailed noise.

In this work, we address these questions by comparing \textit{unbiased} Distributed Compressed SGD (DCSGD) against Distributed SignSGD (DSignSGD), a popular \textit{biased} compression optimizer. While the class of unbiased compressors is widely used in the literature, we specifically focus on biased \textit{sign} compression due to its reported practical superiority \citep{chen2024symbolic, kunstner2024heavy}, communication efficiency \citep{bernstein2018signsgd} and connection to Adam \citep{balles2018dissecting}. As stochastic differential equations (SDEs) have become more and more successful in offering valuable insights into the dynamics of optimization algorithms \citep{li2017stochastic, jastrzkebski2017three,liu2018diffusion,hu2019global,bercher2020weak,zhu2020sharp,cui2020momentum,maulen2021continuous,wang2020asymptotic,lanconelli2022note,ayadi2021stochastic,soto2022sde,li2022uniform,wang2022two,bardi2022deep,chen2022approximation,kunin2023limiting,zhang2023stochastic,sun2023scheduling,li2023fast, gess2024stochastic,dambrine2024stochastic,maulen2024stochastic}, we utilize these continuous-time models to pursue our objective. SDEs can effectively model the dynamics of discrete-time stochastic optimizers in continuous time (see Figure \ref{fig:SDEs} for a graphical representation). Crucially, using SDEs allows us to leverage powerful results from Itô calculus, facilitating the derivation of convergence bounds, stationary distributions, and scaling rules \emph{with minimal mathematical effort}. This approach is especially useful for analyzing the intricate interactions between the optimization landscape, stochastic noise, and compression. Finally, another significant advantage of SDEs is that they enable a direct comparison between optimizers by making explicit the impact of hyperparameters and landscape features on their behavior \citep{compagnoni2024sde,li2017stochastic,Malladi2022AdamSDE,orvieto2019continuous}.

\begin{figure}[ht]
   \centering
   \vspace{-0.4cm}
    \subfloat{{\includegraphics[width=0.45\linewidth]{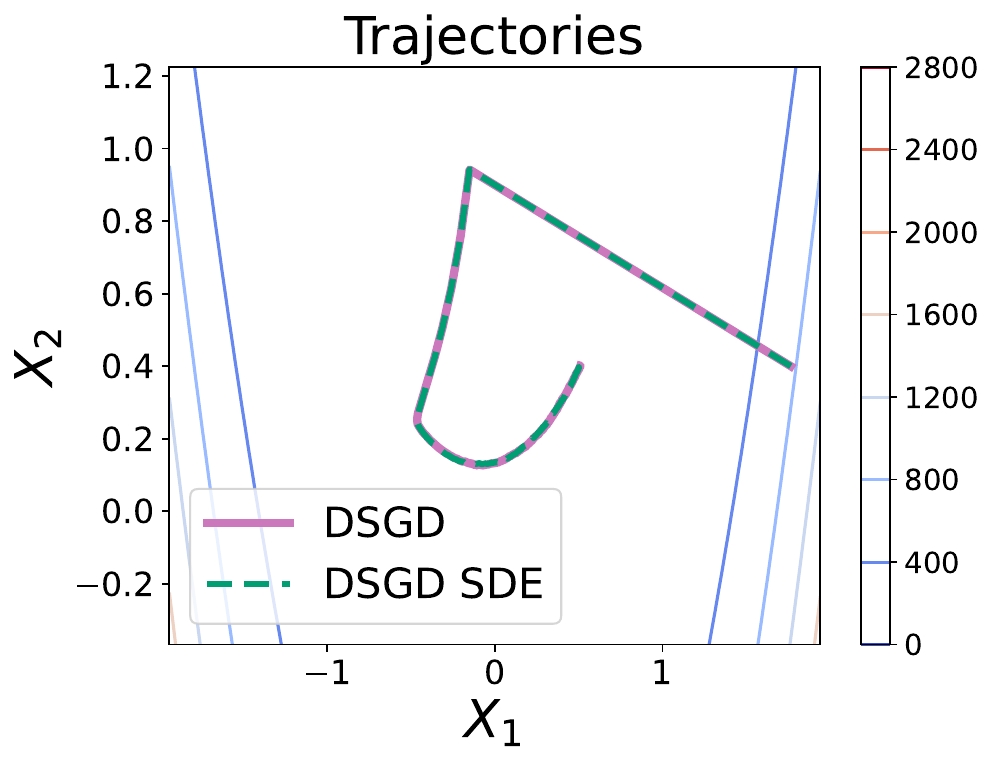} }}%
    \hspace{1.cm}\subfloat{{\includegraphics[width=0.45\linewidth]{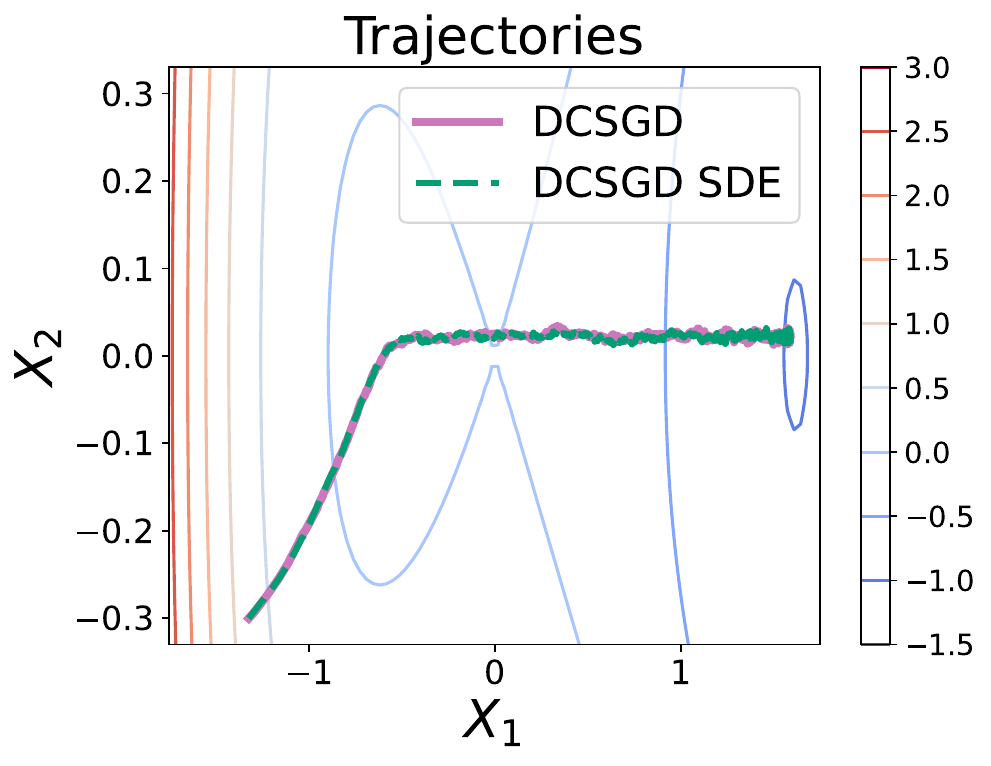} }} \\
    \subfloat{{\includegraphics[width=0.45\linewidth]{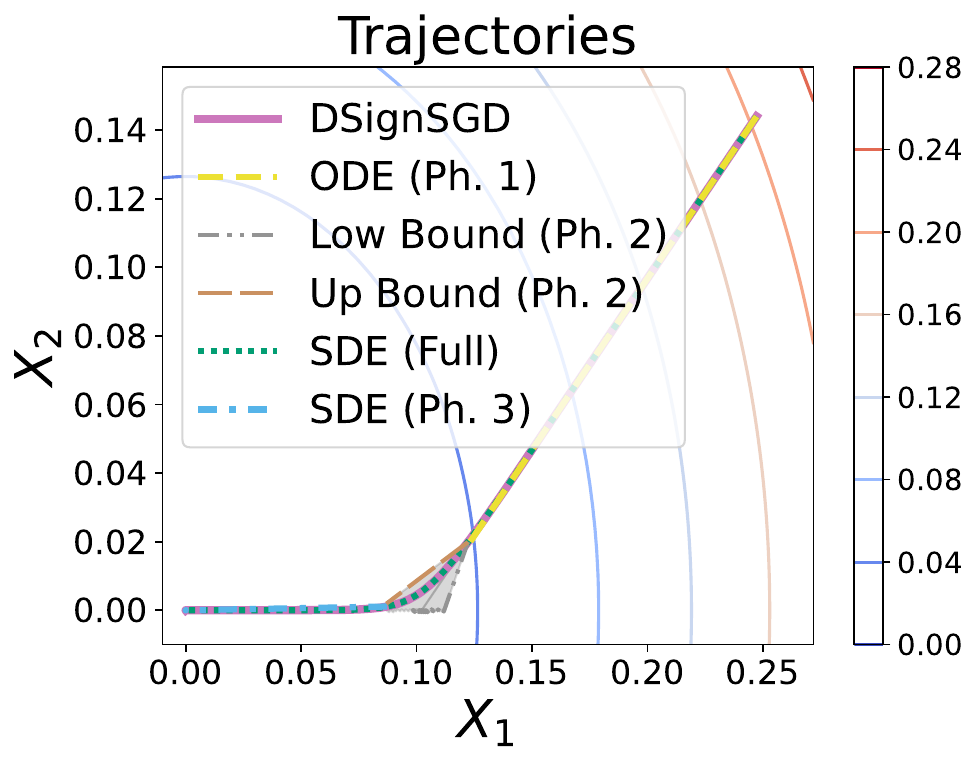} }}%
    \hspace{1.cm}\subfloat{{\includegraphics[width=0.45\linewidth]{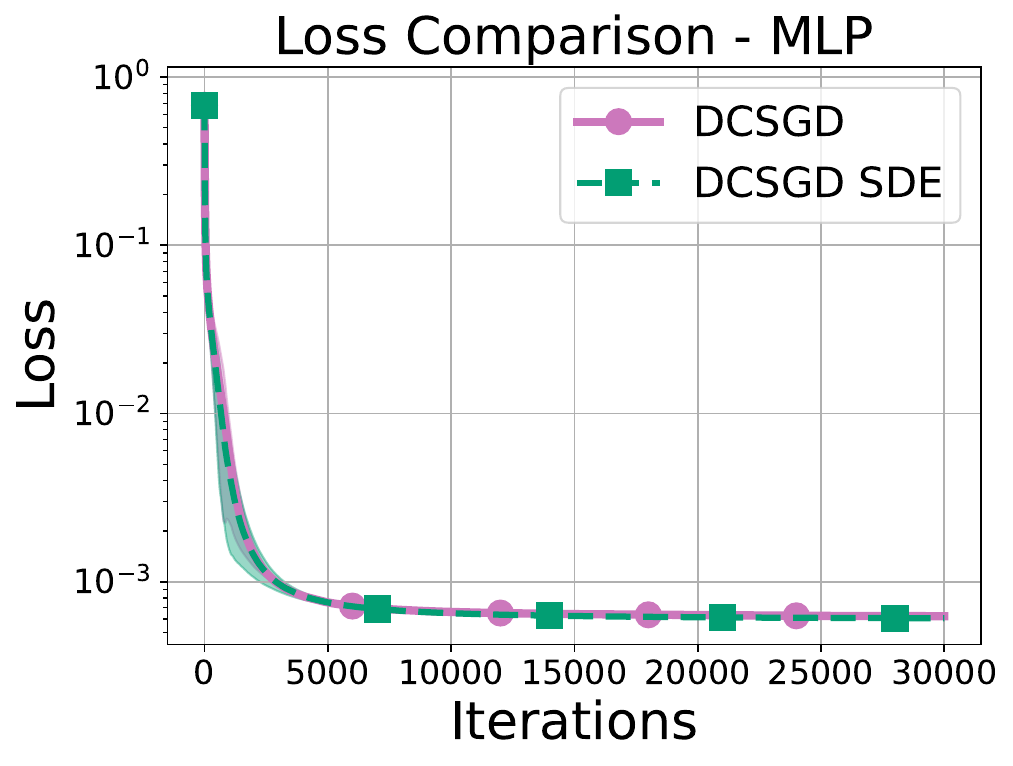} }}%
    \caption{Empirical validation that the trajectories of the SDEs match those of the respective algorithm averaged over 500 runs - DSGD (Theorem~\ref{thm:DSGD_Theorem}) on a Rosenbrock function ({\bf Upper-Left}); DCSGD (Theorem~\ref{thm:DCSGD_Theorem}) with Rand-$k$ on an Embedded Saddle ({\bf Upper-Right}); DSignSGD on a Convex Quadratic: As per Theorem~\ref{thm:DSignSGD_Theorem}, the dynamics of DSignSGD can be partitioned into three phases --- Not only the ``Full'' SDE is a faithful model for DSignSGD through the whole dynamics, but so are the ODE of Phase~1 and the SDE of Phase~3 in their respective phases. Importantly, the bound that characterizes Phase~2 captures the dynamics as prescribed ({\bf Bottom-Left}); The SDEs and the optimizers move at the same speed — DCSGD on an MLP ({\bf Bottom-Right}). For details, see Appendix \ref{app:Exper}.}
    \label{fig:SDEs}%
\end{figure}

\textbf{Contributions.} We identify the following as key ones:
\begin{enumerate}[leftmargin=*, topsep=0.3pt,, itemsep=0.3pt]
    \item We derive the first SDEs for DSGD, DCSGD, and DSignSGD under general assumptions and compare their behavior in terms of expected loss, expected gradient norm, and stationary distribution. Importantly, we discover that \textit{sign} and \textit{unbiased} compression interact differently with gradient noise;
    \item For SignSGD, we prove that \textit{sign} compression causes the noise level, e.g. standard deviation or scale, to inversely affect the convergence rate of both the loss and the iterates. This is in contrast with DCSGD for which the noise level plays no role. Additionally, the noise level has a linear impact on the asymptotic expected loss and covariance of the iterates while this is quadratic for DCSGD;
    \item We show that heavy-tailed noise marginally affects the performance of DSignSGD, which remains robust even to noise of infinite expected value. Under the same conditions, DCSGD fails to converge;
    \item We derive novel scaling rules for DCSGD and DSignSGD: These rules provide intuitive and actionable guidelines to adjust hyperparameters, e.g. to contrast the performance degradation due to compression, or adapt to newly available hardware;
    \item We empirically verify every theoretical insight and prediction. Our experiments are conducted on a variety of deep learning architectures (MLP, ResNet, ViT, GPT2) and datasets (Breast Cancer Wisconsin, MNIST, CIFAR-10, FineWeb-Edu).
\end{enumerate}

%%%%%%%%%%%%%%%%%%%%%%%%%%%%%%%%%%%%%%%%%%%%%%%%%%%%%%%%%%%%%%%%%%%%%%%%%%%%
\section{Related work}
\label{sec:RelatedWorks}

%%%%%%%%%%%%%%%%%%%%%%%%%%%%%%%%%%%%%%%%%%%%%%%%%%%%%%%%%%%%%%%%%%%%%%%%%%%%

\textbf{SDE Approximations and Applications.} In \citep{li2017stochastic}, a formal theoretical framework was proposed to derive SDEs that accurately capture the stochastic nature inherent in optimization methods commonly used in machine learning. Since then, SDEs have been applied in various areas of machine learning, including \emph{stochastic optimal control} to optimally adjust both stepsize \citep{li2017stochastic,li2019stochastic} and batch size \citep{zhao2022batch}. Importantly, they have been used to characterize \emph{convergence bounds} and \emph{stationary distributions} \citep{compagnoni2023sde,compagnoni2024sde,compagnoni2025adaptive}, \emph{scaling rules} \citep{jastrzkebski2017three,Malladi2022AdamSDE,compagnoni2025adaptive}, and provided insights in the context of \emph{implicit regularization} \citep{smith2021origin,compagnoni2023sde}.

\textbf{Two Classes of Compression.} The current theory focuses \textit{either} on unbiased \citep{condat2022ef,PD2025, mishchenko2024distributed, islamov2021distributed} \textit{or} biased \citep{gao2023econtrol, fatkhullin2024momentum} compression without discussing the conceptual differences between the two classes. However, biased compressors typically outperform their unbiased counterparts in practical applications \citep{seide20141}. Therefore, there is a gap between theory and practice which we aim to reduce in this paper.

\textbf{Heavy-tailed Noise.} Recent empirical evidence suggests that the noise in several deep learning setups follows a heavy-tailed distribution \citep{simsekli2019tail, zhang2020adaptive, gurbuzbalaban2021heavy, kunstner2024heavy}. In contrast, previous studies mostly focused on more restricted bounded variance assumptions. Therefore, there is a growing interest in analyzing the convergence of algorithms under heavy-tailed noise  \citep{devlin2018bert, sun2023distributed, yang2022taming, gorbunov2023high}. Earlier works \citep{khirirat2023clip21, li2023convergenceandprivacy, yu2023smoothed} combined compression and clipping to make the algorithm communication-efficient and robust to heavy-tailed noise: We show that the \textit{sign} compressor alone effectively addresses both issues without introducing additional hyperparameters.

%%%%%%%%%%%%%%%%%%%%%%%%%%%%%%%%%%%%%%%%%%%%%%%%%%%%%%%%%%%%%%%%%%%%%%%%%%%%
\section{Formal Statements \& Insights Through SDEs}\label{sec:Insights}
%%%%%%%%%%%%%%%%%%%%%%%%%

\vspace{0.2cm}

This section provides the formulations of the SDEs of DSGD (Theorem \ref{thm:DSGD_Theorem}), DCSGD (Theorem \ref{thm:DCSGD_Theorem}) and DSignSGD (Theorem \ref{thm:DSignSGD_Theorem}): We use these models to derive convergence rates, scaling rules, and stationary distributions of the respective optimizers. Given the technical complexity of the analysis, the formal statements and proofs are provided in the appendix for further reference.

\vspace{0.1cm}

\textbf{Assumptions and Notation.} In this section, we assume that the stochastic gradient of the $i$-th agent is given by $\nabla f_{\gamma_i}(x) = \nabla f(x) + Z_i(x)$, where $Z_i(x)$ denotes the gradient noise and $Z_i(x)$ is independent of $Z_j(x)$ for $i \neq j$. If $Z_i(x) \in L^1(\R^d)$, we assume $\E[Z_i(x)] = 0$, and if $Z_i(x) \in L^2(\R^d)$, we assume $Cov(Z_i(x)) = \Sigma_i(x)$\footnote{We omit the size of the batch $\gamma$ unless relevant.} s.t. $\sqrt{\Sigma_i(x)}$ is bounded, Lipschitz, satisfies affine growth, and together with its derivatives, it grows at most polynomially fast (Definition 2.5 in \cite{Malladi2022AdamSDE}). Importantly, we assume that all $Z_i(x)$ have a smooth and bounded probability density function whose derivatives are all integrable: A common assumption in the literature is for $Z_i(x)$ to be Gaussian\footnote{See \cite{jastrzkebski2017three} for the justification why this might be the case.} \cite{ahn2012bayesian,chen2014stochastic,mandt2016variational,stephan2017stochastic,zhu2019anisotropic,wu2020noisy,Xie2021} while our assumption allows for heavy-tailed distributions such as the Student's t. To derive the stationary distribution near the optimum, we approximate the loss function as a quadratic convex function $f(x) = \frac{1}{2} x^{\top} H x$, a standard approach in the literature \citep{ge2015escaping,levy2016power,jin2017escape,poggio2017theory,mandt2017stochastic,compagnoni2023sde}. 

About notation, $n_i$ is the number of data points in the local dataset of the $i$-th agent, $\eta > 0$ is the learning rate, and the batches $\{ \gamma_k \}$ have size $B \geq 1$ and are modeled as i.i.d.~random variables uniformly distributed over $\{ 1, \dots, n_i \}$. Finally, $W_t$ is the Brownian motion.

The following definition formalizes in which sense an SDE can be ``reliable'' in modeling an optimizer.

\vspace{0.1cm}

\begin{definition}[Weak Approximation]\label{def:weak_approximation}
A continuous-time stochastic process $\{X_t\}_{ t \in [0, T]}$ is an order $\alpha$ weak approximation of a discrete stochastic process $\{x_k\}_{k=0}^{\lf T/\eta \rf}$ if for every polynomial growth function $g$, there exists a positive constant $C$, independent of the stepsize $\eta$, such that $$ \max _{k=0, \ldots, \lf T/\eta \rf}\left|\E g\left(x_k\right)-\E g\left(X_{k \eta}\right)\right| \leq C \eta^\alpha.$$
\end{definition}

Rooted in the numerical analysis of {\small SDEs \cite{mil1986weak}} this definition quantifies the discrepancy between the continuous-time model {\small $X_t$} and discrete-time process {\small $x_k$}. 

%%%%%%%%%%%%%%%%%%%%%%%%%%%%%%%%%%%%%%%%%%%%%%%%%%%%%%%%%%%%%%%%%%%%%%%%%%%%
\subsection{Distributed SGD}\label{sec:DSGD}
%%%%%%%%%%%%%%%%%%%%%%%%%

In this section, we derive an SDE model for Distributed SGD whose update rule is

\begin{equation} \label{eq:DUSGD_Discr_Update}
    x_{k+1} = x_{k} - \frac{\eta}{N}  \sum_{i=1}^{ N} \nabla f_{\gamma_i} (x_k).
\end{equation}

Though not surprising, the following results serve as a reference point for analyzing other optimizers in the subsequent sections.
The first shows the SDE of DSGD which we validate in Figure~\ref{fig:SDEs} on a simple landscape.
\begin{theorem}[Informal Statement of Theorem \ref{thm:DSGD_SDE}] 
\label{thm:DSGD_Theorem}
The SDE of DSGD is
    \begin{equation}
    d X_t = - \nabla f(X_t) dt + \sqrt{\frac{\eta}{N B}} \sqrt{\hat{\Sigma}(X_t)} dW_t,
\end{equation}
where $\hat{\Sigma}(x)\eqdef  \frac{1}{N} \sum_{i=1}^{N} \Sigma_i(x)$ is the average of the covariance matrices of the $N$ agents.
\end{theorem}

Effectively, the SDE above extends the single-node case presented by \cite{li2017stochastic}, where the batch size $B$ is replaced by $B N$. Using the SDE established in Theorem~\ref{thm:DSGD_Theorem}, we derive the convergence rate of DSGD for smooth functions that satisfy the PL condition \cite{karimi2016linear}.

\begin{theorem}\label{thm:DSGD_Convergence}
    If $f$ is $\mu$-PL, $L$-smooth, $ \tr(\Sigma_i(x)) < \mathcal{L}_{\sigma_i}$, $\overline{\mathcal{L}}_{\sigma} \eqdef \frac{1}{N} \sum_{i=1}^{N} \mathcal{L}_{\sigma_i}$, and $S_t:=f(X_t) - f(X_*)$
    \begin{equation}
        \E \left[ S_t \right] \leq  S_0 e^{- 2\mu t}  + (1 - e^{- 2\mu t}) \frac{\eta L \overline{\mathcal{L}}_{\sigma}} {4 \mu B N} .
    \end{equation}
\end{theorem}
For the general smooth non-convex functions the convergence guarantees are given in the next theorem.
\begin{theorem}\label{thm:DSGD_Convergence_Smooth}
    If $f$ is $L$-smooth, we use a learning rate scheduler $\eta_t$ such that $\phi^i_t = \int_0^t (\eta_s)^i ds$, $\phi^1_t \overset{t \rightarrow \infty}{\rightarrow} \infty$, $\frac{\phi^2_t}{\phi^1_t} \overset{t \rightarrow \infty}{\rightarrow} 0$, $\overline{\mathcal{L}}_{\sigma}\eqdef \frac{1}{N} \sum_{i=1}^{N} \mathcal{L}_{\sigma_i}$, and $\tilde{t}$ distributed as $\frac{\eta_{\tilde{t}}}{\phi^1_t}$,
    {\small 
    \begin{equation}
        \E \left[ \lVert \nabla f(X_{\tilde{t}}) \rVert_2^2\right] \leq \frac{f(X_0) - f(X_*)}{\phi^1_t}  + \frac{\eta L \overline{\mathcal{L}}_{\sigma} }{2 B  N} \frac{\phi^2_t}{\phi^1_t} \overset{t \rightarrow \infty}{\rightarrow} 0.
    \end{equation}}
\end{theorem}

\vspace{-0.1cm}

\paragraph{Observations on convergence guarantees:}
\begin{enumerate}
    \item In Theorem~\ref{thm:DSGD_Convergence}, the decay is $e^{-2 \mu t}$, as in SGD;
    \item In Theorem~\ref{thm:DSGD_Convergence}, the asymptotic expected loss scales inversely to $N$, i.e. DCSGD attains a linear speedup with the number of agents $N$. Moreover, the stochastic term is proportional to the condition number $\frac{L}{\mu}$ of the Hessian of the loss, and scales with the average variance $\overline{\mathcal{L}}_{\sigma}$ of the gradient noise which is also observed in earlier works \citep{garrigos2023handbook};
    \item Analogous conclusions hold in Theorem~\ref{thm:DSGD_Convergence_Smooth}.
\end{enumerate}

\vspace{-0.1cm}

\subsection*{Scaling Rules: Preserving DSGD Performance}

After an extensive hyperparameter tuning phase of a machine learning model, it is common to need adjustments to the hyperparameters to accommodate new scenarios. For instance, when training LLMs, larger batch sizes may be desirable to fully utilize newly available and larger GPUs, without the need to repeat the fine-tuning process. Scaling rules offer theoretically grounded formulas that allow changes to some hyperparameters while adjusting others to maintain specific performance metrics. These rules are not \textit{strict} but serve as guidelines to \textit{narrow down} the hyperparameter search space. Common scaling rules include the \textit{linear} rule for SGD \citep{jastrzkebski2017three} that prescribes that the ratio of learning rate and batch size must be kept constant and the far more complex \textit{square-root} rule of Adam \citep{Malladi2022AdamSDE}. As we demonstrate next, we extend the \textit{linear} scaling rule of SGD to the distributed setting, thereby incorporating the number of agents.
To establish this scaling rule, we seek a functional relationship between the hyperparameters, ensuring that DSGD with a learning rate of $\eta$, batch size $B$, and $N$ agents achieves the same performance as with a learning rate of $\kappa \eta$, batch size $\delta B$, and $\alpha N$ agents. The next result shows the exact dependencies. 

\begin{proposition}\label{prop:DSGD_ScalingLaws}
   The scaling rule to preserve the performance independently of $\delta$, $\kappa$, and $\alpha$ is $\frac{\kappa}{\alpha \delta} = 1$.
\end{proposition}

\vspace{-0.1cm}

%%%%%%%%%%%%%%%%%%%%%%%%%%%%%%%%%%%%%%%%%%%%%%%%%%%%%%%%%%%%%%%%%%%%%%%%%%%%
\subsection{Distributed Compressed SGD}\label{sec:DCSGD}
%%%%%%%%%%%%%%%%%%%%%%%%%

Next, we study Distributed Compressed SGD whose update rule has a form as
\begin{equation} \label{eq:DCSGD_Discr_Update}
    x_{k+1} = x_{k} - \frac{\eta}{N}  \sum_{i=1}^{ N} \mathcal{C}_{\xi_i} \left( \nabla f_{\gamma_i} (x_k) \right),
\end{equation}
where the stochastic compressors $\mathcal{C}_{\xi_i}$ are independent for different $i$ and satisfy $(i)$ $\E_{\xi_i} \left[ \mathcal{C}_{\xi_i} (x) \right] = x$ and $(ii)$ $\E_{\xi_i} \left[ \lVert \mathcal{C}_{\xi_i} (x) - x \rVert_2^2 \right] \leq \omega_i  \lVert x \rVert_2^2$ for some compression rates $\omega_i\ge 0$. The following result shows the SDE of DCSGD, which we believe to be a novel addition to the literature and reveals the unique manner in which gradient noise and \textit{unbiased} compression influence the dynamics of DCSGD --- See Figure~\ref{fig:SDEs} for its validation on a simple landscape and MLP training with \textit{Rand-$k$}.

\begin{theorem}[Informal Statement of Theorem \ref{thm:DCSGD_SDE}] 
\label{thm:DCSGD_Theorem}
The SDE of DCSGD is
\begin{equation}
    d X_t = - \nabla f(X_t) dt + \sqrt{\frac{\eta}{N}} \sqrt{\Tilde{\Sigma}(X_t)} dW_t,
\end{equation}
where for $\Phi_{\xi_i,\gamma_i}(x) := \mathcal{C}_{\xi_i} \left( \nabla f_{\gamma_i} (x) \right) - \nabla f_{\gamma_i}(x)$
\small 
\begin{equation}
    \Tilde{\Sigma}(x) = \frac{1}{ N} \sum_{i=1}^{ N} \left( \E_{\xi_i \gamma_i} \left[ \Phi_{\xi_i,\gamma_i}(x)\Phi_{\xi_i,\gamma_i}(x)^{\top} \right] + \Sigma_i(x) \right).
\end{equation}
\end{theorem}

The covariance matrix for DCSGD consists of the covariance matrix of DSGD plus an additional component due to the compression, e.g. for $k$-Sparsification, { \small $\Tilde{\Sigma}(x) = \left( \frac{d}{k} - 1  \right)\left(\nabla f(x)\nabla f(x)^{\top} +  \hat{\Sigma}(x)\right) + \hat{\Sigma}(x)$}. 

We derive convergence rates for the loss value and gradient norm by leveraging the SDE derived in Theorem~\ref{thm:DCSGD_Theorem}: These recover the best known results in the literature \citep{khirirat2018distributed, li2020unified,PD2025}, thus testifying that SDEs provide the community with a powerful alternative technique that allows for a \textit{precise analysis} of optimizers. We start with the convergence guarantees for PL functions. 

\begin{theorem}\label{thm:DCSGD_Convergence}
    If $f$ is $\mu$-PL, $L$-smooth, $\overline{\omega} = \frac{\sum_{i=1}^{N} \omega_i}{N} $,  $ \tr(\Sigma_i(x)) < \mathcal{L}_{\sigma_i}$, $\overline{\mathcal{L}}_{\sigma} := \frac{ \sum_{i=1}^{N} \mathcal{L}_{\sigma_i}}{N}$, $\overline{\omega \mathcal{L}_\sigma}:=\frac{\sum_{i=1}^{N} \omega_i \mathcal{L}_{\sigma_i}}{N} $, $S_t:=f(X_t) - f(X_*)$, and $\Delta:= 1 -  \frac{\eta L^2 \overline{\omega}}{2 \mu N}$, then
\small
\begin{equation}
     \E \left[ S_t\right] \leq S_0 \textcolor{PineGreen}{e^{- 2 \mu \Delta t}} + \left(1 -  \textcolor{PineGreen}{e^{- 2 \mu \Delta t}}\right) \frac{\eta L \left( \overline{\mathcal{L}}_{\sigma} + \textcolor{purple}{\overline{\omega \mathcal{L}_\sigma}} \right)}{4 \mu B N \Delta}.
\end{equation}
\end{theorem}

The next theorem offers a \textit{new} and \textit{general} condition on the learning rate scheduler to ensure the convergence of DCSGD for the general non-convex case.
\begin{theorem}
    If $f$ is $L$-smooth and the learning rate scheduler $\eta_t$ is such that $\phi^i_t = \int_0^t (\eta_s)^i ds$, $\phi^1_t \overset{t \rightarrow \infty}{\rightarrow} \infty$, $\frac{\phi^2_t}{\phi^1_t} \overset{t \rightarrow \infty}{\rightarrow} 0$, $\eta_t < \frac{2 N}{\eta L \overline{\omega}}$, and $S_0:=f(X_0) - f(X_*)$ then, $\E \left[\lVert \nabla f(X_{\Tilde{t}}) \rVert_2^2 \right]$ is smaller than
\begin{equation}
    \frac{1}{ 1  - \frac{ \eta L \overline{\omega}}{2 N} \frac{\phi^2_t}{\phi^1_t} } \left( \frac{S_0}{\phi^1_t}  + \frac{\phi^2_t}{\phi^1_t} \frac{\eta L \left( \overline{\mathcal{L}}_{\sigma} +  \textcolor{purple}{\overline{\omega \mathcal{L}_\sigma}} \right)}{2 B N}  \right) \overset{t \rightarrow \infty}{\rightarrow} 0,
\end{equation}
where $\Tilde{t}$, is a random time with distribution $\frac{\eta_{\Tilde{t}} - \frac{\eta L \overline{\omega}}{2 N} (\eta_{\Tilde{t}})^2}{\phi^1_t - \frac{\eta L \overline{\omega}}{2 N} \phi^2_t}$.
\end{theorem}

\begin{figure*}%
    \vspace{0.5cm}
    \begin{tabular}{cc}
    
    \hspace{-0.3cm}\includegraphics[width=0.53\linewidth]{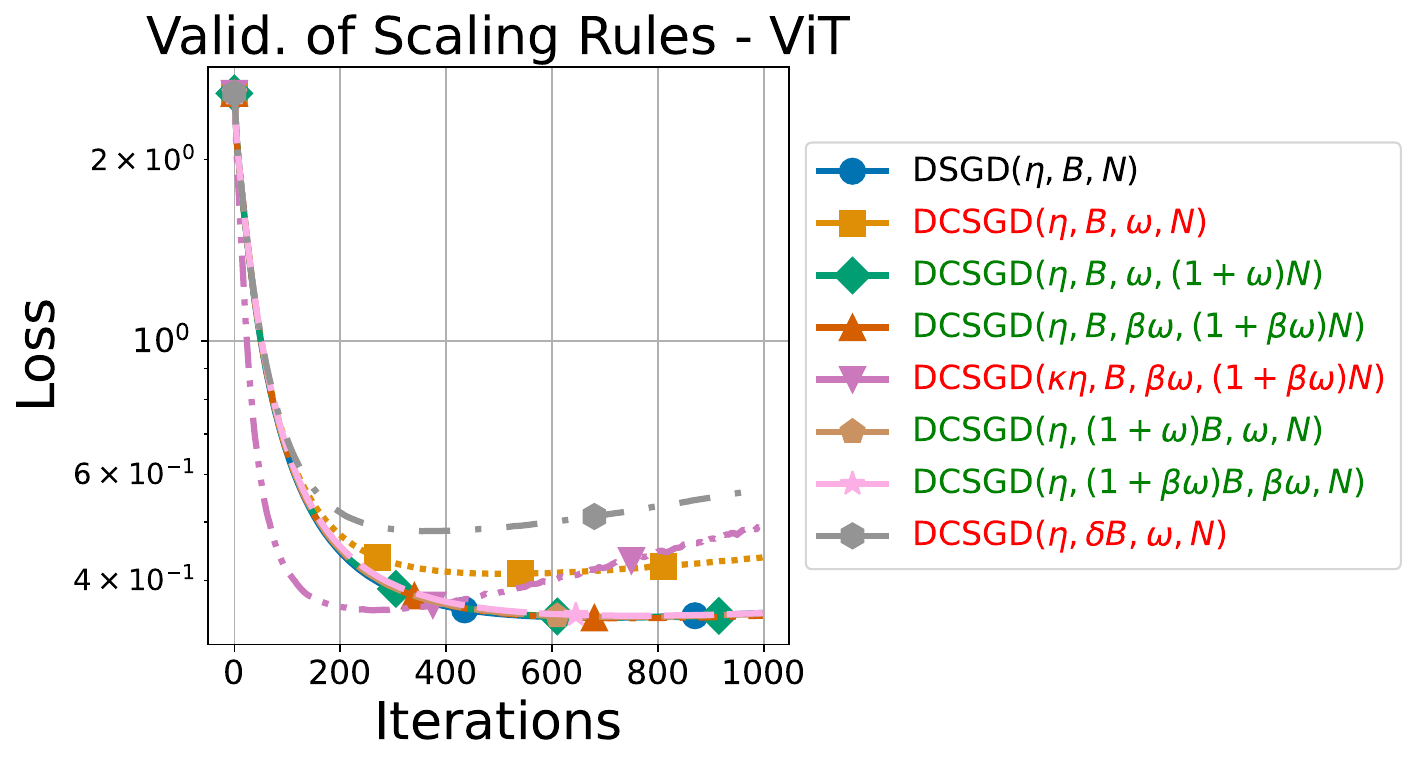}  & 
        \hspace{-0.4cm}
       \includegraphics[width=0.53\linewidth]{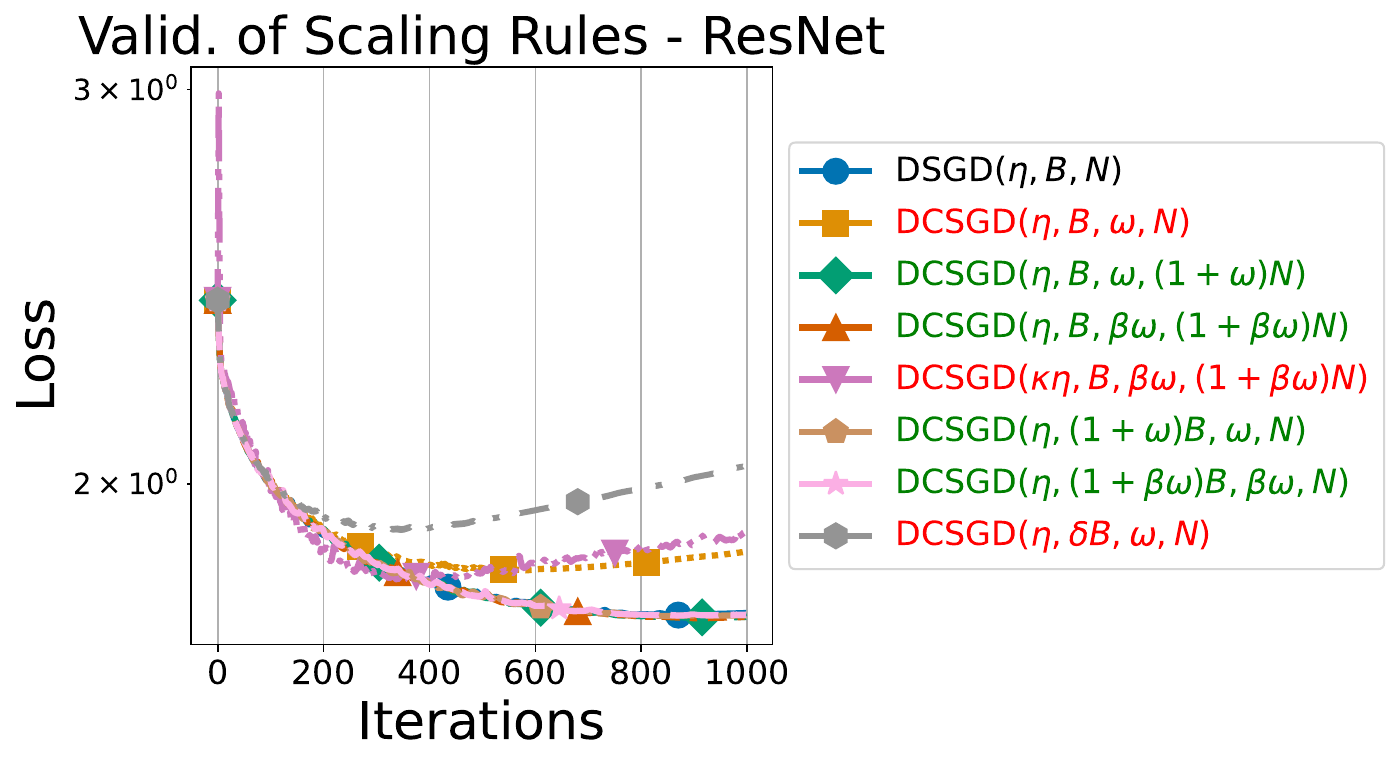}
    \end{tabular}
    \vspace{0.1cm}
    \caption{Validation of Scaling Rules: Consistently with Prop.~\ref{prop:DCSGD_RecoverLaws_Main},  DCSGD run with hyperparameters that follow the scaling rules listed in Table \ref{tab:DCSGD_ScalingLaws} (marked in green in the legends) recover the performance of DSGD$(\eta, B, N)$. Those that do not (marked in red) fail to do so. On the left, we plot the training loss of a ViT for \textit{some} rules while on the right we do the same for a ResNet. Details are in Appendix \ref{app:Exper}.
    }%
    \label{fig:DCSGD_ScalLaw}%
\end{figure*}

\paragraph{Observations on convergence guarantees:}
\begin{enumerate}
    \item The \textcolor{PineGreen}{decay $e^{- 2 \mu \Delta t}$} of DCSGD is strictly slower than that of DSGD: $\Delta$ crucially depends on the average rate of compression $\overline{\omega}$, the condition number, and the number of agents. Specifically, \textbf{larger compression implies a slower convergence in comparison with DSGD}, which is exacerbated for ill-conditioned landscapes;
    \item $\Delta$ needs to be positive to ensure convergence, which imposes limitations on the hyperparameters. For example, $\eta < \frac{2 \mu N}{L^2 \overline{\omega}}$: More agents allow for a larger learning rate, but a larger compression rate or an ill-conditioned landscape restricts it. \textbf{Violating such prescriptions might lead to divergence} (See left of Figure~\ref{fig:DCSGD_Insights_App} for empirical validation);
    \item DCSGD enjoys a linear speedup: The asymptotic loss level in Theorem~\ref{thm:DCSGD_Convergence}
    \textbf{scales inversely to the number of agents $N$} and has an \textcolor{purple}{additional term $\overline{\omega \mathcal{L}_\sigma}$} w.r.t. DSGD (Theorem~\ref{thm:DSGD_Convergence}) which quantifies the nonlinear interaction between the compression and gradient noise. See the center-left and -right of Figure \ref{fig:DCSGD_Insights_App} for empirical validation).
\end{enumerate}

\subsection*{Scaling Rules: Recovering DSGD Performance}

\vspace{0.2cm}

As previously noted, when using the same learning rate $\eta$, batch size $B$, and $N$ agents, DCSGD is slower and less optimal than DSGD. To address this, we propose deriving novel, actionable, and interpretable scaling rules to adjust the hyperparameters of DCSGD to \textit{recover} the asymptotic performance of DSGD. The following result shows these rules under the simplifying assumption that $\mathcal{L}_{\sigma_{i}} = \mathcal{L}_{\sigma} $, $\omega_i=\omega$, for $i\in[N]$, and $N \gg 1$. We defer to Proposition~\ref{prop:DCSGD_RecoverLaws} for the general cases. We validate some of these rules in Figure \ref{fig:DCSGD_ScalLaw} for a ViT and a ResNet. See Appendix \ref{ScalingRulesGPT2} for the validation on a 124M GPT2 model.

\vspace{0.2cm}

\begin{proposition}\label{prop:DCSGD_RecoverLaws_Main}
Let DCSGD$(\kappa \eta, \delta B, \beta \omega, \alpha N)$ run with batch size $ \delta B$, learning rate $\kappa \eta$, compression rates $\beta \omega$, and $\alpha N$ agents. Table \ref{tab:DCSGD_ScalingLaws} shows scaling rules to recover the asymptotic performance of DSGD$(\eta, B, N)$:\footnote{For practical reasons, we only show those involving two hyperparameters while the other two are kept constant.}

\vspace{0.2cm}

\begin{table}[!ht]
    \centering
    \begin{tabular}{|c|c|}
        \hline
        \textbf{Scaling Rule} & \textbf{Implication}\\
        \hline
        $\alpha = 1 + \beta \omega$ & CR $\uparrow \implies$ Agents $\uparrow$\\
        \hline
        $\alpha = \kappa(1 + \omega)$ & LR $\uparrow \implies$ Agents $\uparrow$\\
        \hline
        $\alpha = \frac{1 + \omega}{\delta}$ & BS $\downarrow \implies$ Agents $\uparrow$\\
        \hline
        $\kappa = \frac{1}{1 + \beta \omega}$ & CR $\uparrow \implies$ LR $\downarrow$\\
        \hline
        $\delta = 1 + \beta \omega$ & CR $\uparrow \implies$ BS $\uparrow$\\
        \hline
        $\kappa = \frac{\delta}{1 +  \omega}$ & BS $\uparrow \implies$ LR $\uparrow$\\
        \hline
    \end{tabular}
    \vspace{0.2cm}
    \caption{Scaling Rules to \textit{recover} DSGD performance. For example, compression can be countered by increasing the number of agents (CR = Compression Rate, LR = Learning Rate, and BS = Batch Size).}
    \label{tab:DCSGD_ScalingLaws}
\end{table}

\end{proposition}

\paragraph{Observations on scaling rules:}

\begin{enumerate}
    \item In the absence of compression, the scaling rules all reduce to that of DSGD (See Proposition~\ref{prop:DSGD_ScalingLaws});
    \item For example, to achieve comparable performance between DSGD and DCSGD with compression factor $\omega$, the number of agents can be \textit{increased} to $(1 + \omega)N$. Alternatively, one can further \textit{increase} compression ($\textcolor{OrangeRed}{\beta}>1$) and compensate by \textit{increasing} the number of agents to $(1 + \textcolor{OrangeRed}{\beta}\omega)N$. Similarly, one can \textit{decrease} the learning rate to $\frac{\eta}{1+\omega}$, or \textit{increase} the batch size to $B(1 + \omega)$.
\end{enumerate}

\newpage

%%%%%%%%%%%%%%%%%%%%%%%%%%%%%%%%%%%%%%%%%%%%%%%%%%%%%%%%%%%%%%%%%%%%%%%%%%%%
\subsection{Distributed SignSGD}\label{sec:DSignSGD}
%%%%%%%%%%%%%%%%%%%%%%%%%

Now we turn to derive an SDE for DSignSGD, a \textit{biased} compression method with update rule
\vspace{-0.1cm}
\begin{equation} \label{eq:DSignSGD_Discr_Update}
    x_{k+1} = x_{k} - \frac{\eta}{N}  \sum_{i=1}^{N} \sign(\nabla f_{\gamma_i} (x_k)).
\end{equation}

\vspace{-0.2cm}

This derivation reveals how \textit{sign} compression interacts with the gradient noise in determining the dynamics of DSignSGD. In particular, we focus on the role of the \textbf{level} of the gradient noise, e.g. standard deviation or scale, and of the \textbf{fatness} of the tails of its distribution. See \cite{compagnoni2025adaptive} for a comparison between SignSGD and RMSprop, Adam, and AdamW in the single-node case.

\begin{theorem}[Informal Statement of Theorem \ref{thm:DSignSGD_SDE}]
\label{thm:DSignSGD_Theorem}
The SDE of DSignSGD is
{\small \begin{align}
    d X_t & = - \frac{1}{N} \sum_{i=1}^{N} \left( 1 - 2 \mathbb{P}(\nabla f_{\gamma_i} (X_t) <0) \right) dt + \sqrt{\frac{\eta}{ N}}\sqrt{\overline{\Sigma}(X_t)} dW_t,
\end{align}}\\
where $\overline{\Sigma}(x) := \frac{\sum_{i=1}^{N} \overline{\Sigma_i}(x)}{N} $, $ \overline{\Sigma_i}(x) = \E[\xi_{\gamma_i}(x)\xi_{\gamma_i}(x)^\top]$, and $\xi_{\gamma_i}(x):= \sign (\nabla f_{\gamma_i}(x)) - 1 + 2 \mathbb{P}(\nabla f_{\gamma_i}(x)<0)$ is the noise around $ \sign \left(\nabla f_{\gamma_i}(x) \right)$.
\end{theorem}

For interpretability reasons, we present a corollary with a flexible gradient noise assumption that interpolates between the Cauchy (\textbf{heavy}-tailed) and Gaussian (\textbf{light}-tailed) distributions: {\small $\nabla f_{\gamma_i}(x) = \nabla f(x) + \sqrt{\frac{\Sigma_i}{B}} Z_i$}, {\small $Z_i \sim t_{\nu}(0, I_d) $}, $\nu$ are the degrees of freedom, and scale matrices {\small $\Sigma_i= \diag(\sigma_{1,i}^2, \cdots, \sigma_{d,i}^2)$}. This allows us to \textbf{parse the dynamics of DSignSGD into three distinct phases} depending on the size of the signal-to-noise ratios {\small $Y^i_t := \sqrt{B} \Sigma_i^{-\frac{1}{2}} \nabla f(X_t)$}. This is visually supported in the bottom-right of Figure~\ref{fig:SDEs} on a convex quadratic function.

The following results involve several quantities, highlighted using colors for clarity: \textit{\textcolor{Rhodamine}{Pink}} indicates dependence on the degrees of freedom \(\textcolor{Rhodamine}{\nu}\), related to the concept of \textit{\textcolor{Rhodamine}{fatness}} of the noise, while \textit{\textcolor{NavyBlue}{blue}} corresponds to the \textit{\textcolor{NavyBlue}{level}} of noise.

\begin{proposition} \label{prop:HSignSGD_three_phases}
For some constants $\mathbf{q}_{\mathbf{{\textcolor{Rhodamine}{\nu}}}}^{+}$, $\mathbf{q}_{\mathbf{{\textcolor{Rhodamine}{\nu}}}}^{-}$, $m_{\mathbf{{\textcolor{Rhodamine}{\nu}}}}$, $\ell_{\mathbf{{\textcolor{Rhodamine}{\nu}}}}$, and $\psi_{\mathbf{{\textcolor{Rhodamine}{\nu}}}}$ that depend on the degrees of freedom $\mathbf{{\textcolor{Rhodamine}{\nu}}}$,\footnote{See Proposition~\ref{prop:HSignSGD_three_phases_App} for their definition.} the dynamics of DSignSGD has three phases:

    \begin{enumerate} 
        \item \textbf{Phase~1:} If $ \left\vert Y^i_t  \right\vert>\psi_{\mathbf{{\textcolor{Rhodamine}{\nu}}}}$, the SDE coincides with the ODE of SignGD:
        \begin{equation} \label{sde:HSignSGD_ph1}
            d X_t = - \sign ( \nabla f(X_t)) dt;
        \end{equation}
        \item \textbf{Phase~2:}  If $ 1 < \left\vert Y^i_t \right\vert< \psi_{\mathbf{{\textcolor{Rhodamine}{\nu}}}}$ and $\overline{Y_t}:=\frac{\sum_{i=1}^N Y^i_t}{N}$:
        {\small $\mathbb{P}\left[ \lVert X_t - \E \left[ X_t \right] \rVert^2_2 > a \right] \leq \frac{\eta}{a} \left[ d - \frac{\sum_{i=1}^N \lVert  m_{\mathbf{{\textcolor{Rhodamine}{\nu}}}} Y^i_t + \mathbf{q}_{\mathbf{{\textcolor{Rhodamine}{\nu}}}}^{-} \rVert^2_2}{N} \right]$;}
        \item \textbf{Phase~3:} If {\small$ \left\vert Y^i_t \right\vert<1$}, the SDE is
        {\small
        \begin{align} \label{sde:HSignSGD_ph3}
            d X_t & = -  \ell_{\mathbf{{\textcolor{Rhodamine}{\nu}}}} \left(\frac{\sqrt{B}}{N} \sum_{i=1}^N \textcolor{NavyBlue}{\Sigma_i^{-\frac{1}{2}}} \right) \nabla f(X_t) dt  + \sqrt{\frac{\eta}{N}} \sqrt{I_d - \frac{\ell_{\mathbf{{\textcolor{Rhodamine}{\nu}}}}^2 B}{N} \sum_{i=1}^N \diag \left(\textcolor{NavyBlue}{\Sigma_i^{-\frac{1}{2}}} \nabla f(X_t) \right)^2} d W_t. \nonumber
        \end{align}
}

    \end{enumerate}
\end{proposition}

\begin{figure*}[ht]%
    \hspace{-0.6cm}
    \begin{tabular}{cc}
       \includegraphics[width=0.5\linewidth]{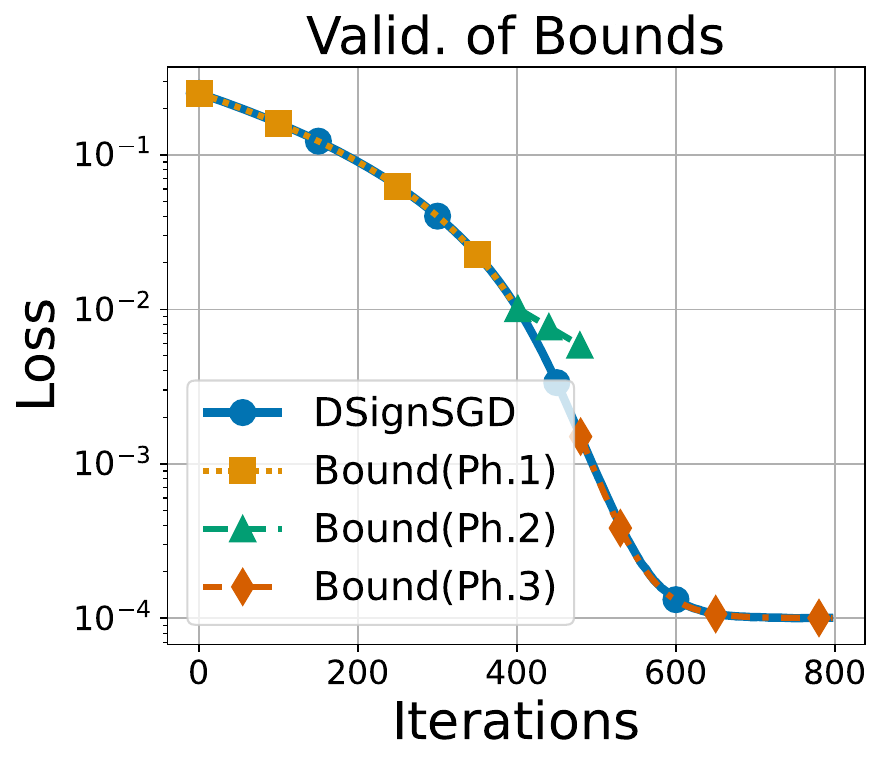}  & 
       \hspace{-0.6cm}
       \includegraphics[width=0.57\linewidth]{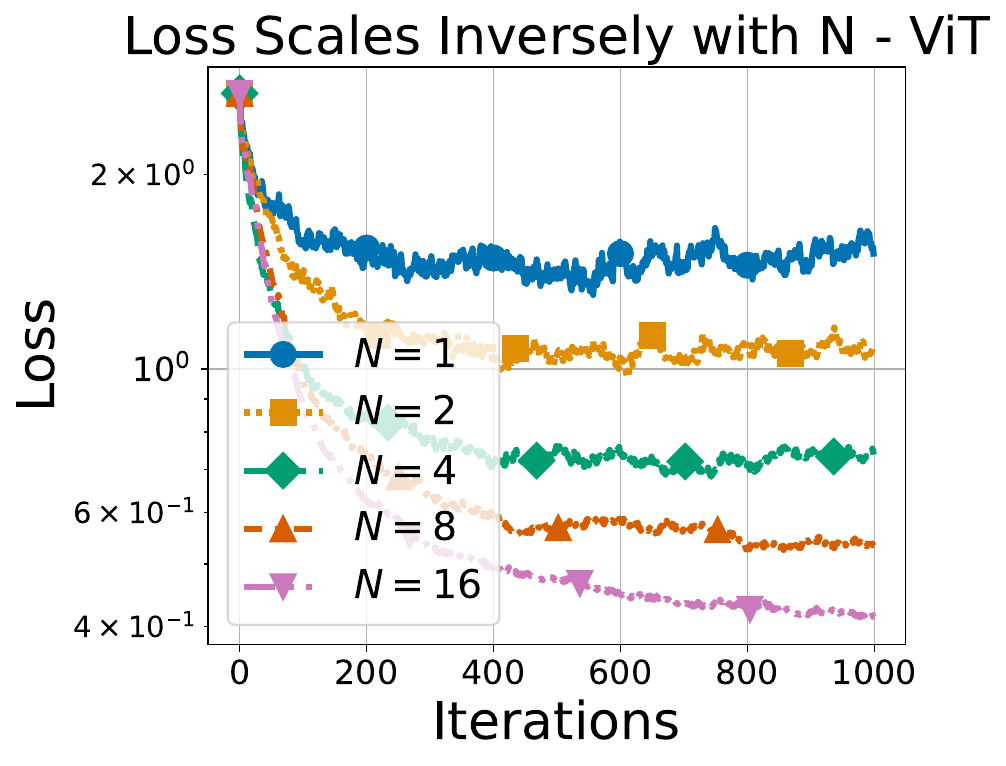}
    \end{tabular}
    
    \caption{Validation of Bounds: As prescribed by Theorem~\ref{thm:DSignSGD_Convergence}, the bounds match or dominate the empirical loss of DSignSGD on a quadratic convex function in all three phases ({\bf Left}); As per Theorem~\ref{thm:DSignSGD_Convergence}, DSignSGD achieves linear speedup: More agents imply lower loss ({\bf Right}); 
    }
\label{fig:DSignSGD_Insights}%
\end{figure*}

\vspace{-0.1cm}

\textbf{Observation on SDEs:}
\begin{enumerate}
    \item The behavior of DSignSGD depends on the size of the signal-to-noise ratios;
    \item In Phase~2 and Phase~3, the inverse of the \textcolor{NavyBlue}{level} of the noise $\textcolor{NavyBlue}{\Sigma_i^{-\frac{1}{2}}}$ premultiplies the gradient, thus \textbf{affecting the rate of descent}: The \textcolor{NavyBlue}{larger} the scale, the slower the dynamics. This is not the case for DCSGD where the $\Sigma_i$ \textbf{only influence the diffusion term};
    \item The degrees of freedom $\mathbf{{\textcolor{Rhodamine}{\nu}}}$ of the Student's t parametrize the \textcolor{Rhodamine}{fatness} of the tails of the noise distribution: The smaller $\mathbf{{\textcolor{Rhodamine}{\nu}}}$, the \textcolor{Rhodamine}{fatter} the tails and the smaller $m_{\mathbf{{\textcolor{Rhodamine}{\nu}}}}$ and $\ell_{\mathbf{{\textcolor{Rhodamine}{\nu}}}}$\footnote{For example, $\ell_{1} = \frac{2}{\pi}$, $\ell_{2} = \frac{1}{\sqrt{2}}$, and $\ell_{\infty} = \sqrt{\frac{2}{\pi}}$.} --- \textbf{\textcolor{Rhodamine}{Fatter} tails} imply that the average dynamics of $X_t$ is \textbf{slower} and exhibits \textbf{more variance}.
\end{enumerate}

To better understand the role of the noise, we need to study how its \textcolor{NavyBlue}{level} and \textcolor{Rhodamine}{fatness} affect the dynamics of the expected loss. The tightness of these bounds is empirically verified on the left of Figure~\ref{fig:DSignSGD_Insights}. 

\begin{theorem}\label{thm:DSignSGD_Convergence}
Let $f$ be $\mu$-strongly convex, $\tr(\nabla^2 f(x)) < \mathcal{L}_{\tau}$, $\Sigma_i \leq \sigma_{\text{max}, i}^2$, $S_t:=f(X_t) - f(X_*)$, and $\sigma_{\mathcal{H},j}$ be the harmonic mean of $\{(\sigma_{\text{max}, i})^j \}$. Then,
\begin{enumerate}
    \item In \textbf{Phase~1}, $S_t \leq \frac{1}{4} \left( \sqrt{\mu}t - 2 \sqrt{S_0}\right)^2$:  DSignSGD stays in this phase for at most $t_* = 2 \sqrt{\frac{S_0}{\mu}}$;
    \item In \textbf{Phase~2}, for $\Delta:=  m_{\mathbf{{\textcolor{Rhodamine}{\nu}}}} \sqrt{B}\textcolor{NavyBlue}{\sigma_{\mathcal{H},1}^{-1}} + \frac{\eta B \mu m_{\mathbf{{\textcolor{Rhodamine}{\nu}}}}^2}{2 N} \sigma_{\mathcal{H},2}^{-1}$, \\
    $$ \E[S_t] \leq S_0 e^{- 2 \mu \Delta t} +  \frac{\eta (\mathcal{L}_{\tau} - \mu d \hat{q}_{\mathbf{{\textcolor{Rhodamine}{\nu}}}}^2)}{2 N} \frac{1}{ 2 \mu \Delta} \left(1 - e^{- 2 \mu \Delta t}\right);$$
    \item In \textbf{Phase~3}, for $\Delta:=  \ell_{\mathbf{{\textcolor{Rhodamine}{\nu}}}} \sqrt{B} \textcolor{NavyBlue}{\sigma_{\mathcal{H},1}^{-1}} + \frac{\eta B \mu \ell_{\mathbf{{\textcolor{Rhodamine}{\nu}}}}^2}{2 N} \sigma_{\mathcal{H},2}^{-1}$, \\
    $$ \E[S_t] \leq S_0 e^{- 2 \mu \Delta t} +  \frac{\eta \mathcal{L}_{\tau}}{2 N} \frac{1}{ 2 \mu \Delta} \left(1 - e^{- 2 \mu \Delta t}\right).$$
\end{enumerate}
\end{theorem}

\paragraph{Observations on Convergence - PL:}
\begin{enumerate}
    \item The dynamics of Phase~1 ensures a steady decrease of $S_t$ independently of the noise, which triggers the emergence of Phase~2;
    \item During Phase~2 and Phase~3, the exponential decay of the loss strongly \textbf{depends on the noise distributional properties}: It scales inversely to the noise \textcolor{NavyBlue}{level $\sigma_{\mathcal{H},1}$} and proportionally to the \textcolor{Rhodamine}{fatness}  of the tails $\ell_{\mathbf{{\textcolor{Rhodamine}{\nu}}}}$, meaning that \textcolor{NavyBlue}{larger} noise and \textcolor{Rhodamine}{fatter} tails imply a slower convergence;
    \item The asymptotic loss level \textbf{achieves a linear speedup with $N$}, scales inversely to $\ell_{\mathbf{{\textcolor{Rhodamine}{\nu}}}}$ and proportionally to $\textcolor{NavyBlue}{\sigma_{\mathcal{H},1}}$: More agents imply lower loss (see right of Figure~\ref{fig:DSignSGD_Insights}) while \textcolor{Rhodamine}{fatter} tails and \textcolor{NavyBlue}{larger} noise imply larger loss (Section \ref{sec:RoleofNoise} for a discussion and empirical validation).
\end{enumerate}

The next theorem shows a general condition on the learning rate scheduler to ensure the convergence of DSignSGD for the general non-convex case. Interestingly, it sheds light on how DSignSGD reduces different gradient norms ($L^{\textcolor{purple}{\mathbf{1}}}$ and $L^{\textcolor{orange}{\mathbf{2}}}$) across different phases.

\newpage

\vspace{0.5cm}

\begin{figure*}[ht]%
    \hspace{-0.6cm}
    \begin{tabular}{cc}
       \includegraphics[width=0.5\linewidth]{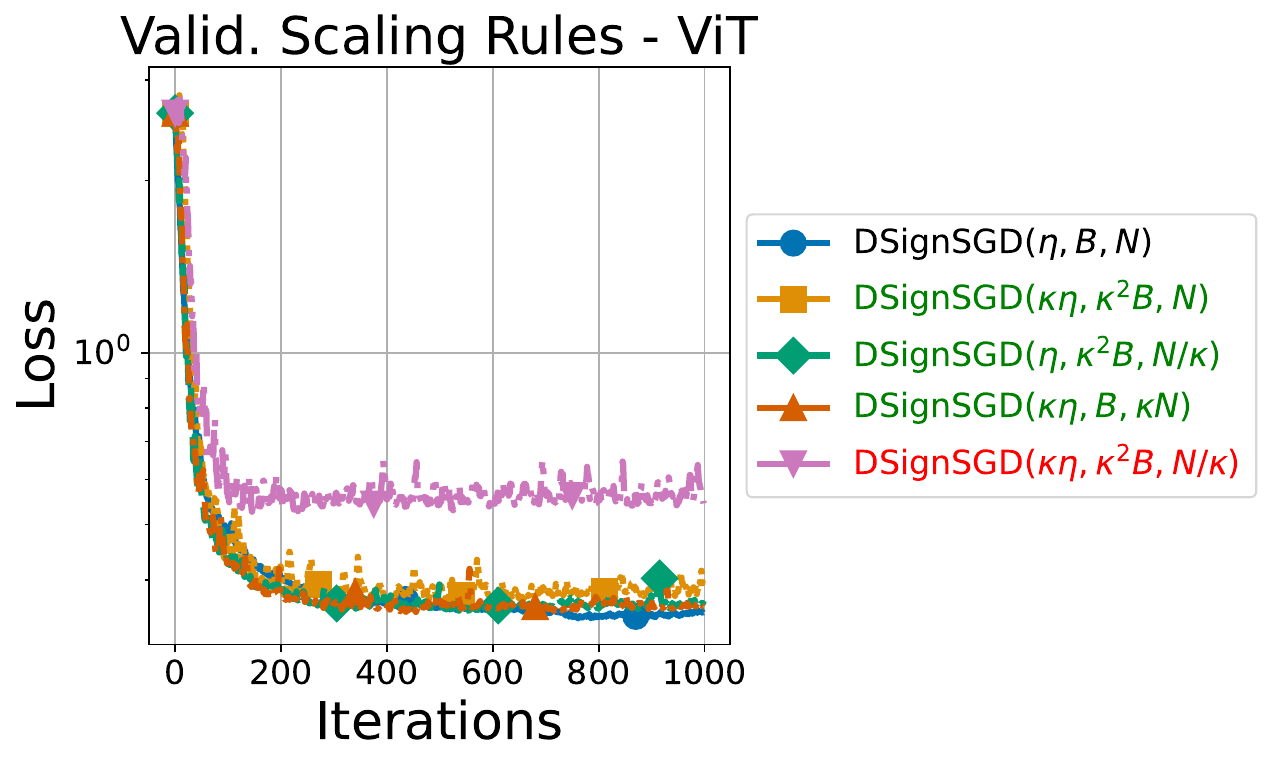}  & 
       \hspace{-0.6cm}
       \includegraphics[width=0.5\linewidth]{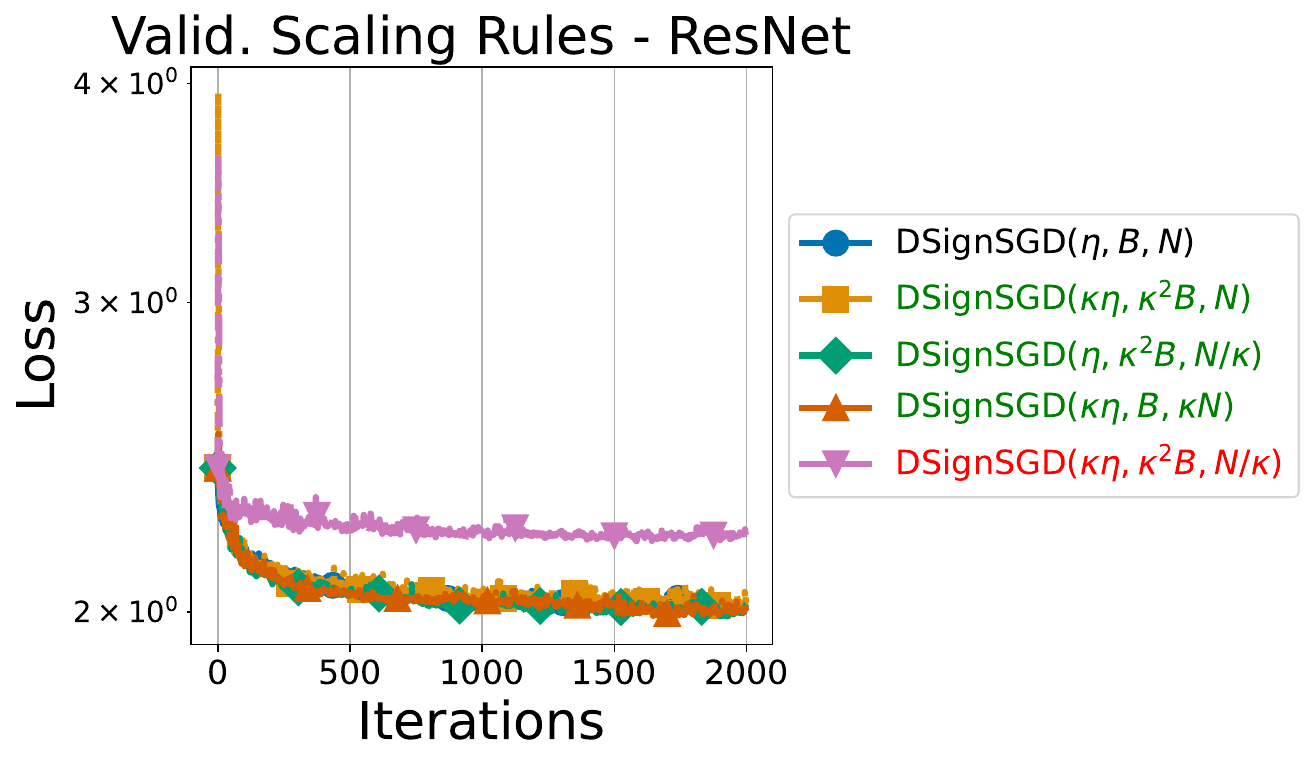}
    \end{tabular}
    \vspace{0.1cm}
    \caption{Validation of Scaling Rules: Consistently with Proposition~\ref{prop:DSignSGD_ScalingLaws}, DSignSGD run with hyperparameters that follow our scaling rule (in green in the legends) recover the performance of DSignSGD$(\eta, B, N)$. The one that does not (in red) fails to do so. On the left, we plot the training loss of a ViT for \textit{some} rules while on the right we do the same for a ResNet. Details in Appendix \ref{app:Exper}.
    }%
    \label{fig:DSignSGD_ScalLaw}%
\end{figure*}

\begin{theorem}\label{thm:DSignSGD_Conv_LSmooth}
    Let $f$ be $L$-smooth, the learning rate scheduler $\eta_t$ s.t. $\phi^i_t = \int_0^t (\eta_s)^i ds$, $\phi^1_t \overset{t \rightarrow \infty}{\rightarrow} \infty$, $\frac{\phi^2_t}{\phi^1_t} \overset{t \rightarrow \infty}{\rightarrow} 0$, $\Sigma_i \leq \sigma_{\text{max}, i}^2$, and $S_0:=f(X_0) - f(X_*)$. Then,
    \vspace{0.15cm}
    \begin{enumerate}
        \item In \textbf{Phase~1}, $\lVert \nabla f\left(X_{\Tilde{t}^1}\right)\rVert_{\textcolor{purple}{\mathbf{1}}} \leq \frac{S_0}{\phi_t^1} \overset{t \rightarrow \infty}{\rightarrow} 0$;
        \vspace{0.15cm}
        \item In \textbf{Phase~2}, 
        \begin{align}
        & \E \lVert \nabla f\left(X_{\Tilde{t}^{(1,2)}}\right)\rVert_{\textcolor{orange}{\mathbf{2}}}^2 + \frac{\hat{q}_{\mathbf{{\textcolor{Rhodamine}{\nu}}}}  \textcolor{NavyBlue}{\sigma_{\mathcal{H},1}}}{m_{\mathbf{{\textcolor{Rhodamine}{\nu}}}} \sqrt{B} } \E \lVert \nabla f\left(X_{\Tilde{t}^{(2,2)}}\right)\rVert_{\textcolor{purple}{\mathbf{1}}} \le \frac{ \textcolor{NavyBlue}{\sigma_{\mathcal{H},1}}}{\phi^1_t m_{\mathbf{{\textcolor{Rhodamine}{\nu}}}} \sqrt{B}} \left( S_0 + \frac{\eta L d \phi_t^2}{2 N} \right) \overset{t \rightarrow \infty}{\rightarrow} 0;
    \end{align}
    \vspace{0.15cm}
        \item In \textbf{Phase~3}, $  \E \lVert \nabla f\left(X_{\Tilde{t}^{3}}\right)\rVert_{\textcolor{orange}{\mathbf{2}}}^2$ is smaller than
        \begin{equation}
         \frac{ \textcolor{NavyBlue}{\sigma_{\mathcal{H},1}}}{\phi^1_t \ell_{\mathbf{{\textcolor{Rhodamine}{\nu}}}} \sqrt{B}} \left( S_0 + \frac{\eta L d \phi_t^2}{2 N} \right) \overset{t \rightarrow \infty}{\rightarrow} 0.
    \end{equation}
    \end{enumerate}
Above, $\Tilde{t}^1$, $\Tilde{t}^{(1,2)}$, $\Tilde{t}^{(2,2)}$, and $\Tilde{t}^3$ are random times with distribution $\frac{\eta_t}{\phi^1_t}$.
\end{theorem}

\vspace{0.15cm}

\paragraph{Observations on Convergence - Non-convex:}

\vspace{0.15cm}

\begin{enumerate}
    \item DSignSGD \textbf{implicitly} minimizes the $L^{\textcolor{purple}{\mathbf{1}}}$ norm of the gradient in Phase~1, a linear combination of $L^{\textcolor{purple}{\mathbf{1}}}$ and $L^{\textcolor{orange}{\mathbf{2}}}$ norm in Phase~2, and transitions to optimizing the $L^{\textcolor{orange}{\mathbf{2}}}$ norm in Phase~3;
    \vspace{0.15cm}
    \item \textcolor{NavyBlue}{Large} and \textcolor{Rhodamine}{fat} noise slow down the convergence;
    \vspace{0.15cm}
    \item Note that \citep{safaryan2021signsgd} derived convergence guarantees for a mixture of $L^{\textcolor{purple}{\mathbf{1}}}$ and $L^{\textcolor{orange}{\mathbf{2}}}$ norms. This mixture reduces to a rescaled $L^{\textcolor{purple}{\mathbf{1}}}$ norm when the gradients are large (similarly to our Phase~1) and to a rescaled $L^{\textcolor{orange}{\mathbf{2}}}$ norm when the gradients are small (as in our Phase~ 3). However, we highlight that our rates reveal \textit{exactly} how all parameters affect the rate while in \citep{safaryan2021signsgd} these dependencies are \textit{hidden} in the mixed norm.
\end{enumerate}

We conclude by observing that \textit{not all biased compressors} can handle fat noise, e.g. Top-$k$ fails as well, while a slight modification is promising --- See Figure~\ref{fig:Topk}.

\subsection*{Scaling Rules: Preserving Performance}

Under the simplifying assumption that $N \gg 1$, the following result provides intuitive scaling rules for DSignSGD, while Proposition~\ref{prop:DSignSGD_ScalingLaws} presents the general cases. We validate some rules in Figure~\ref{fig:DSignSGD_ScalLaw} on a ViT and a ResNet. See Appendix \ref{ScalingRulesGPT2} for the validation on a 124M GPT2.
\begin{proposition}\label{prop:DSignSGD_ScalingLaws_Main}
Let the batch size be {\small $ \delta B$}, learning rate {\small $\kappa \eta$}, {\small $\alpha N$} agents, and {\small $N \gg 1$}. The scaling rule to preserve the performance indep. of {\small $\delta$}, {\small $\kappa$}, {\small $\alpha$}, is {\small $\frac{\kappa}{\alpha \sqrt{\delta}}=1$}.
\end{proposition}

\paragraph{Observations:}

\begin{enumerate}
    \item If $\alpha=1$, this rule coincides with Adam's \cite{compagnoni2025adaptive};
    \item For example, while preserving the performance of DSignSGD run with $\eta$, $B$, and $N$, one can increase the batch size to $\kappa^2 B$, the learning rate to $\kappa \eta$, and keep $N$ agents. Alternatively, keep the learning rate to $\eta$ but increase the batch size to $\kappa^2 B$ and decrease the number of agents down to $\frac{N}{\kappa}$.
\end{enumerate}

\textbf{Stationary Distribution.}
The stationary distribution of a process characterizes its behavior at convergence. Proposition~\ref{prop:HSignSGD_StaDistr_App}, shows that of DSignSGD: The main takeaway is that \textbf{the covariance matrix of the iterates scales linearly in the noise level}. In contrast, Proposition~\ref{prop:DCSGD_StaDistr_App} shows that the scaling is \textbf{quadratic} for DCSGD with $k$-Sparsification. These findings are novel and are empirically validated in Figure \ref{fig:StatDistr}.

%%%%%%%%%%%%%%%%%%%%%%%%%%%%%%%%%%%%%%%%%%%%%%%%%%%%%%%%%%%%%%%%%%%%%%%%%%%%
\section{Heavy and Large Noise}\label{sec:RoleofNoise}
%%%%%%%%%%%%%%%%%%%%%%%%%

\begin{figure}[ht]%
   \centering

    \subfloat{{\includegraphics[width=0.45\linewidth]{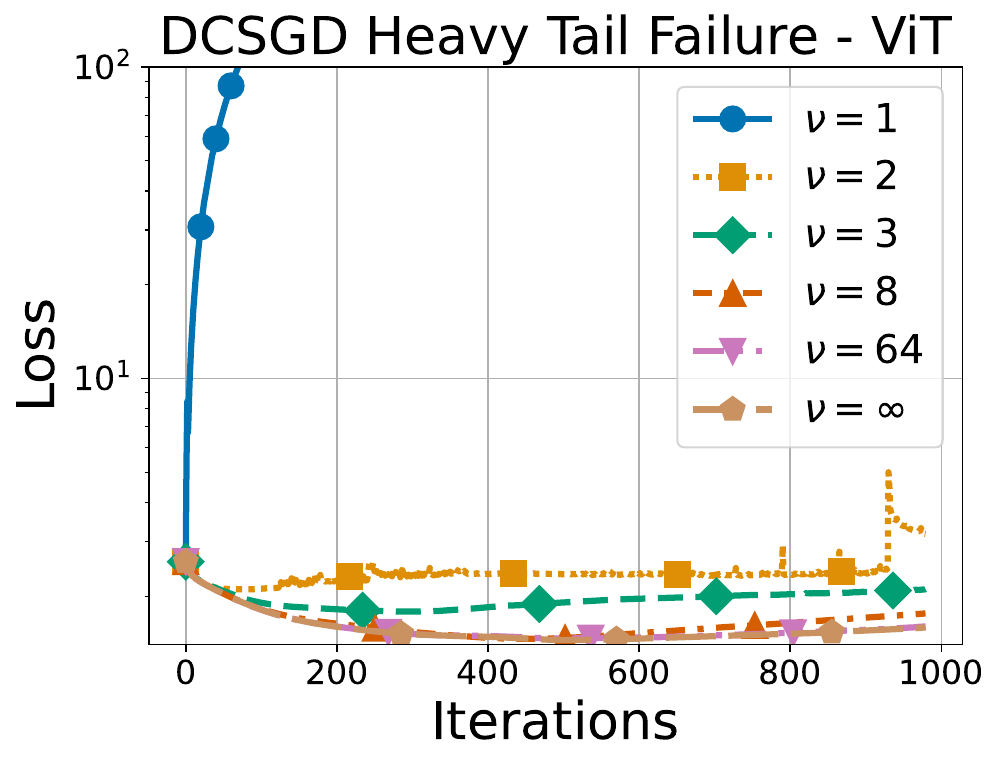} }}%
    \hspace{1.cm}\subfloat{{\includegraphics[width=0.45\linewidth]{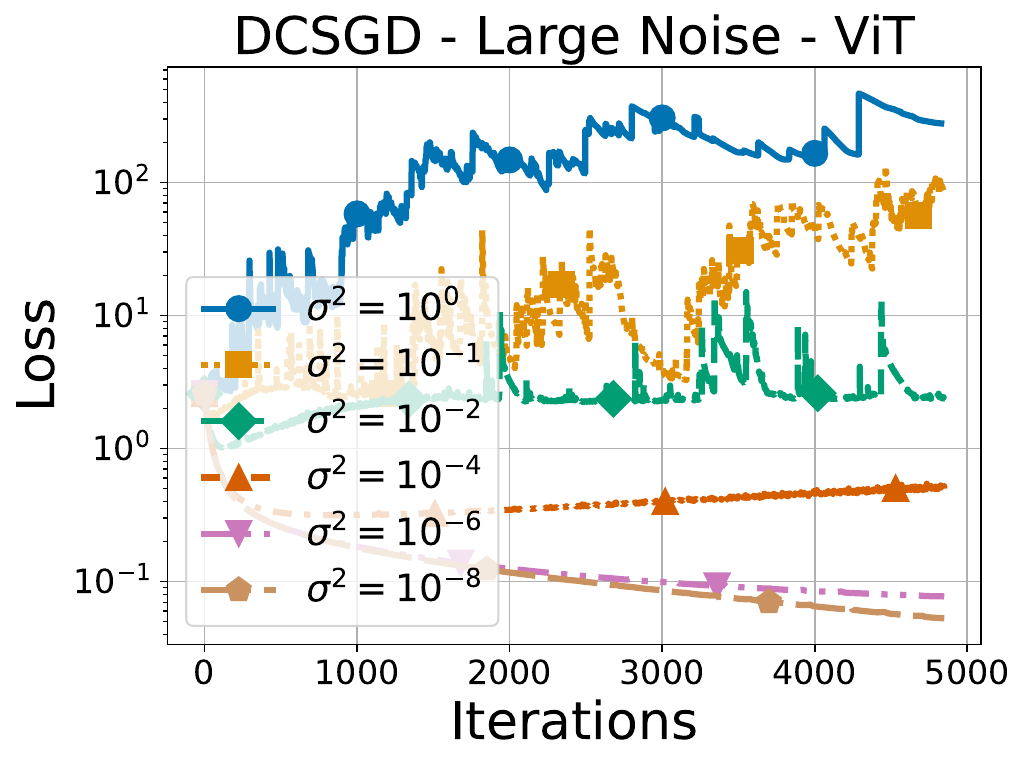} }} \\
    \subfloat{{\includegraphics[width=0.45\linewidth]{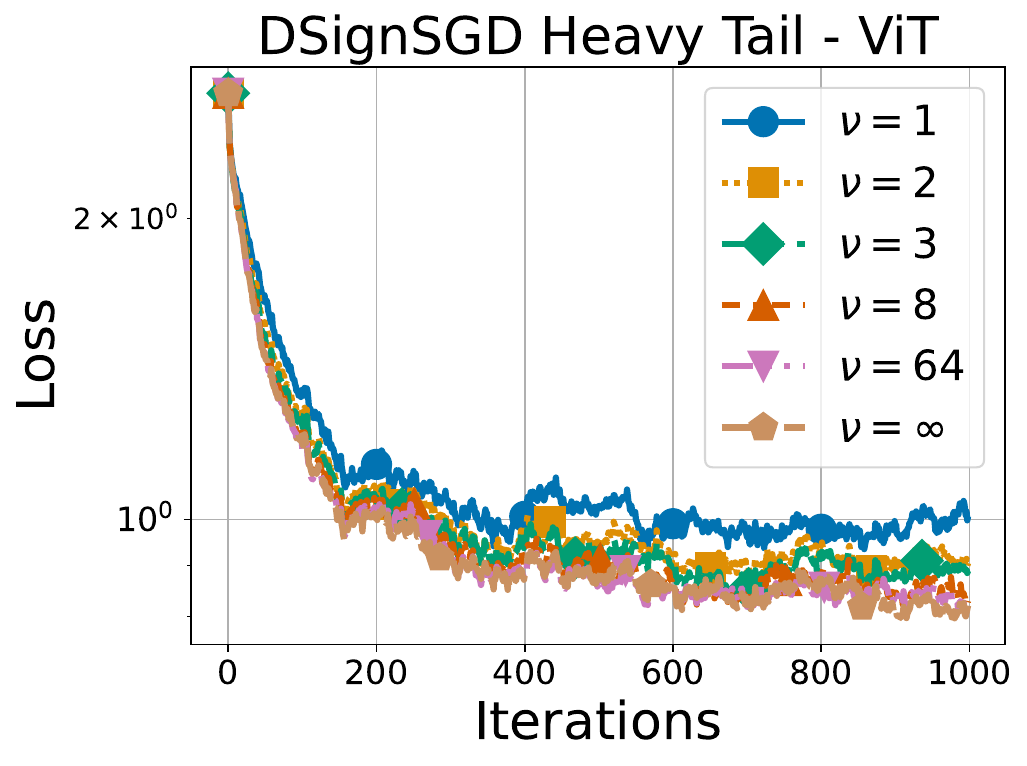} }}%
    \hspace{1.cm}\subfloat{{\includegraphics[width=0.45\linewidth]{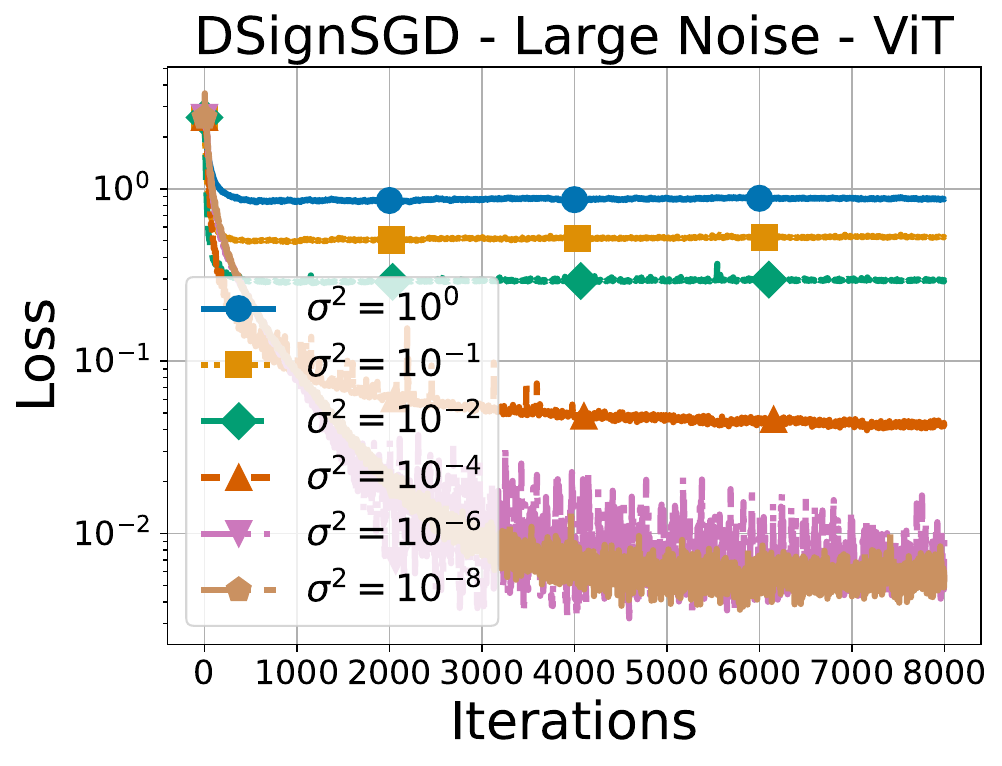} }}%
    \caption{Empirical validation of the insights derived from Theorem~\ref{thm:DCSGD_Convergence} and Theorem~\ref{thm:DSignSGD_Convergence}: i) DCSGD cannot handle \textcolor{Rhodamine}{fat} noise - The loss diverges if $\nu=1$ and is non-stationary if $\nu=2$ ({\bf Upper-Left}); ii) The loss diverges more and more for \textcolor{NavyBlue}{larger} noise ({\bf Upper-Right}); DSignSGD converges even when the noise is \textcolor{Rhodamine}{fat}, although \textcolor{Rhodamine}{fatter} noise implies less optimality ({\bf Bottom-Left}); DSignSGD never diverges even when noise becomes increasingly \textcolor{NavyBlue}{larger} ({\bf Bottom-Right}).}%
    \label{fig:ComparisonLargeNoise_ViT}%
\end{figure}

In this section, we recap our findings regarding the behavior of D(C)SGD and DSignSGD w.r.t. how \textcolor{Rhodamine}{fat}, i.e. how \textcolor{Rhodamine}{heavy-tailed} the noise is, and its \textcolor{NavyBlue}{level}, i.e. its \textcolor{NavyBlue}{standard deviation or scale}. We validate our results as we inject Gaussian or Student's t-distributed noise on the full gradient of a ViT trained on MNIST.

Theoretically and practically, \textbf{DCSGD diverges if the noise is fat}, i.e., does not admit bounded first or second moments (upper-left of Figure \ref{fig:ComparisonLargeNoise_ViT}). Provided that the noise admits a finite second moment, as per Theorem~\ref{thm:DCSGD_Convergence}, DCSGD converges to an asymptotic loss level that scales \textbf{quadratically} with the noise level $\sigma$: The upper-right of Figure \ref{fig:ComparisonLargeNoise_ViT} shows this on a ViT while Figure \ref{fig:ComparisonLargeNoise_Quadratics} validates the tightness of the bounds derived in Theorem~\ref{thm:DCSGD_Convergence} on a quadratic convex function for several noise levels.

In contrast, Theorem \ref{thm:DSignSGD_Convergence} shows that \textbf{DSignSGD converges even if the noise has an unbounded expected value}. In particular, the fatness of the tails influences both the convergence speed and the asymptotic loss level: \textbf{Fatter and larger} noise implies a \textbf{slower convergence} to a \textbf{larger} asymptotic level (bottom-left of Figure \ref{fig:ComparisonLargeNoise_ViT}). Additionally, the asymptotic loss level of DSignSGD scales (approximately) \textbf{linearly} with the noise level: The bottom-right of Figure \ref{fig:ComparisonLargeNoise_ViT} show this on a ViT while Figure \ref{fig:ComparisonLargeNoise_Quadratics} demonstrates the tightness of the bounds derived in Theorem \ref{thm:DSignSGD_Convergence} on a quadratic convex function for several noise levels.

%%%%%%%%%%%%%%%%%%%%%%%%%%%%%%%%%%%%%%%%%%%%%%%%%%%%%%%%%%%%%%%%%%%%%%%%%%%%
\section{Conclusions}\label{sec:Conclusions}
%%%%%%%%%%%%%%%%%%%%%%%%%

We derived the first formal SDE models for DSGD, DCSGD, and DSignSGD, enabling us to elucidate the complex and different ways in which \textit{unbiased} and \textit{sign} compression interact with gradient noise. We started by showcasing the \textit{tightness} of our analysis as we recovered and empirically validated the \textit{best known} convergence results for DSGD and DCSGD: 1) We quantified how \textit{unbiased} compression \textbf{slows down} the convergence of DCSGD w.r.t. DSGD, and showed that the noise level \textbf{does not impact} the convergence speed; 2) Unbiased compression and noise level interact nonlinearly by negatively affecting the asymptotic loss level of DCSGD w.r.t. DSGD. For DSignSGD, we 3) proved that \textit{sign} compression implies that noise level \textbf{does influence} the speed of convergence as \textbf{larger noise slows it down}; 4) While the asymptotic loss level of DCSGD scales \textbf{quadratically} in the noise level, that of DSignSGD does so \textbf{linearly}; 5) DSignSGD is \textbf{resilient to heavy-tailed} noise and converges even when this has an unbounded expected value. Much differently, an \textbf{unbounded variance} of the noise is already enough for \textbf{DCSGD to diverge}. 6) Importantly, we prove that DSignSGD achieves \textbf{linear speedup}; 7) Finally, we derive \textbf{novel scaling rules} for DCSGD and DSignSGD, providing intuitive and actionable guidelines for selecting hyperparameters. These rules ensure that the performance of the algorithms is preserved, even allowing DCSGD to \textit{recover} the performance of its \textit{uncompressed} counterpart and DSignSGD to \textit{preserve} it. Finally, we verify our results on a variety of deep learning architectures and datasets.

\textbf{Future work.} Our analysis can be extended to other practical optimizers, such as Top-$k$ or DSignSGD with majority vote \cite{bernstein2018signsgdd}. Moreover, the insights derived from our SDE analysis provide a foundation for developing new optimization algorithms that integrate the strengths of current methods while addressing their limitations. Finally, it is possible to extend most of our results to the \textit{heterogeneous} federated setting, up to some adjustments in the regularity of the local loss functions.

\section{Acknowledgments}
Enea Monzio Compagnoni, Rustem Islamov, and Aurelien Lucchi acknowledge the financial support of the Swiss
National Foundation, SNF grant No 207392. Frank Norbert Proske acknowledges the financial support of the Norwegian Research Council (project No 274410) and MSCA4Ukraine (project No 101101923).

\clearpage
\bibliography{references}

\begin{thebibliography}{123}
\providecommand{\natexlab}[1]{#1}
\providecommand{\url}[1]{\texttt{#1}}
\expandafter\ifx\csname urlstyle\endcsname\relax
  \providecommand{\doi}[1]{doi: #1}\else
  \providecommand{\doi}{doi: \begingroup \urlstyle{rm}\Url}\fi

\bibitem[Abadi et~al.(2016)Abadi, Barham, Chen, Chen, Davis, Dean, Devin, Ghemawat, Irving, Isard, et~al.]{abadi2016tensorflow}
Mart{\'\i}n Abadi, Paul Barham, Jianmin Chen, Zhifeng Chen, Andy Davis, Jeffrey Dean, Matthieu Devin, Sanjay Ghemawat, Geoffrey Irving, Michael Isard, et~al.
\newblock {TensorFlow}: a system for {Large-Scale} machine learning.
\newblock In \emph{12th USENIX symposium on operating systems design and implementation (OSDI 16)}, 2016.

\bibitem[Ahn et~al.(2012)Ahn, Korattikara, and Welling]{ahn2012bayesian}
Sungjin Ahn, Anoop Korattikara, and Max Welling.
\newblock Bayesian posterior sampling via stochastic gradient {F}isher scoring.
\newblock \emph{arXiv preprint arXiv:1206.6380}, 2012.

\bibitem[Alistarh et~al.(2018)Alistarh, De~Sa, and Konstantinov]{alistarh2018convergence}
Dan Alistarh, Christopher De~Sa, and Nikola Konstantinov.
\newblock The convergence of stochastic gradient descent in asynchronous shared memory.
\newblock In \emph{Proceedings of the 2018 ACM Symposium on Principles of Distributed Computing}, 2018.

\bibitem[An et~al.(2020)An, Lu, and Ying]{an2020stochastic}
Jing An, Jianfeng Lu, and Lexing Ying.
\newblock Stochastic modified equations for the asynchronous stochastic gradient descent.
\newblock \emph{Information and Inference: A Journal of the IMA}, 2020.

\bibitem[Ankirchner and Perko(2024)]{ankirchner2024comparison}
Stefan Ankirchner and Stefan Perko.
\newblock A comparison of continuous-time approximations to stochastic gradient descent.
\newblock \emph{Journal of Machine Learning Research}, 2024.

\bibitem[Aviv et~al.(2021)Aviv, Hakimi, Schuster, and Levy]{aviv2021learning}
Rotem~Zamir Aviv, Ido Hakimi, Assaf Schuster, and Kfir~Y Levy.
\newblock Learning under delayed feedback: Implicitly adapting to gradient delays.
\newblock \emph{ICML}, 2021.

\bibitem[Ayache et~al.(2023)Ayache, Dassari, and El~Rouayheb]{ayache2023walk}
Ghadir Ayache, Venkat Dassari, and Salim El~Rouayheb.
\newblock Walk for learning: A random walk approach for federated learning from heterogeneous data.
\newblock \emph{IEEE Journal on Selected Areas in Communications}, 2023.

\bibitem[Ayadi and Turinici(2021)]{ayadi2021stochastic}
Imen Ayadi and Gabriel Turinici.
\newblock Stochastic runge-kutta methods and adaptive sgd-g2 stochastic gradient descent.
\newblock In \emph{2020 25th International Conference on Pattern Recognition (ICPR)}, 2021.

\bibitem[Balles and Hennig(2018)]{balles2018dissecting}
Lukas Balles and Philipp Hennig.
\newblock Dissecting {Adam}: The sign, magnitude and variance of stochastic gradients.
\newblock In \emph{International Conference on Machine Learning}, 2018.

\bibitem[Bardi and Kouhkouh(2022)]{bardi2022deep}
Martino Bardi and Hicham Kouhkouh.
\newblock Deep relaxation of controlled stochastic gradient descent via singular perturbations.
\newblock \emph{arXiv preprint arXiv:2209.05564}, 2022.

\bibitem[Bercher et~al.(2020)Bercher, Gonon, Jentzen, and Salimova]{bercher2020weak}
Aritz Bercher, Lukas Gonon, Arnulf Jentzen, and Diyora Salimova.
\newblock Weak error analysis for stochastic gradient descent optimization algorithms.
\newblock \emph{arXiv preprint arXiv:2007.02723}, 2020.

\bibitem[Bernstein et~al.(2018)Bernstein, Wang, Azizzadenesheli, and Anandkumar]{bernstein2018signsgd}
Jeremy Bernstein, Yu-Xiang Wang, Kamyar Azizzadenesheli, and Animashree Anandkumar.
\newblock Sign{SGD}: Compressed optimisation for non-convex problems.
\newblock In \emph{Proceedings of the 35th International Conference on Machine Learning}, 2018.

\bibitem[Bernstein et~al.(2019)Bernstein, Zhao, Azizzadenesheli, and Anandkumar]{bernstein2018signsgdd}
Jeremy Bernstein, Jiawei Zhao, Kamyar Azizzadenesheli, and Anima Anandkumar.
\newblock sign{SGD} with majority vote is communication efficient and fault tolerant.
\newblock In \emph{International Conference on Learning Representations}, 2019.
\newblock URL \url{https://openreview.net/forum?id=BJxhijAcY7}.

\bibitem[Beznosikov et~al.(2023)Beznosikov, Horv{\'a}th, Richt{\'a}rik, and Safaryan]{beznosikov2023biased}
Aleksandr Beznosikov, Samuel Horv{\'a}th, Peter Richt{\'a}rik, and Mher Safaryan.
\newblock On biased compression for distributed learning.
\newblock \emph{Journal of Machine Learning Research}, 2023.

\bibitem[Bradbury et~al.(2018)Bradbury, Frostig, Hawkins, Johnson, Leary, Maclaurin, Necula, Paszke, Vander{P}las, Wanderman-{M}ilne, and Zhang]{jax2018github}
James Bradbury, Roy Frostig, Peter Hawkins, Matthew~James Johnson, Chris Leary, Dougal Maclaurin, George Necula, Adam Paszke, Jake Vander{P}las, Skye Wanderman-{M}ilne, and Qiao Zhang.
\newblock {JAX}: composable transformations of {P}ython+{N}um{P}y programs, 2018.
\newblock URL \url{http://github.com/google/jax}.

\bibitem[Cattaneo et~al.(2024)Cattaneo, Klusowski, and Shigida]{OTIBA2024}
Matias~D. Cattaneo, Jason~Matthew Klusowski, and Boris Shigida.
\newblock On the implicit bias of adam, 2024.
\newblock URL \url{https://openreview.net/forum?id=ZA9XUTseA9}.

\bibitem[Chen et~al.(2022)Chen, Lu, and Xu]{chen2022approximation}
Peng Chen, Jianya Lu, and Lihu Xu.
\newblock Approximation to stochastic variance reduced gradient langevin dynamics by stochastic delay differential equations.
\newblock \emph{Applied Mathematics \& Optimization}, 2022.

\bibitem[Chen et~al.(2014)Chen, Fox, and Guestrin]{chen2014stochastic}
Tianqi Chen, Emily Fox, and Carlos Guestrin.
\newblock Stochastic gradient hamiltonian monte carlo.
\newblock In \emph{International conference on machine learning}, pages 1683--1691. PMLR, 2014.

\bibitem[Chen et~al.(2024)Chen, Liang, Huang, Real, Wang, Pham, Dong, Luong, Hsieh, Lu, et~al.]{chen2024symbolic}
Xiangning Chen, Chen Liang, Da~Huang, Esteban Real, Kaiyuan Wang, Hieu Pham, Xuanyi Dong, Thang Luong, Cho-Jui Hsieh, Yifeng Lu, et~al.
\newblock Symbolic discovery of optimization algorithms.
\newblock \emph{Advances in neural information processing systems}, 2024.

\bibitem[Compagnoni et~al.(2023)Compagnoni, Biggio, Orvieto, Proske, Kersting, and Lucchi]{compagnoni2023sde}
Enea~Monzio Compagnoni, Luca Biggio, Antonio Orvieto, Frank~Norbert Proske, Hans Kersting, and Aurelien Lucchi.
\newblock An sde for modeling sam: Theory and insights.
\newblock In \emph{International Conference on Machine Learning}, pages 25209--25253. PMLR, 2023.

\bibitem[Compagnoni et~al.(2024)Compagnoni, Orvieto, Kersting, Proske, and Lucchi]{compagnoni2024sde}
Enea~Monzio Compagnoni, Antonio Orvieto, Hans Kersting, Frank Proske, and Aurelien Lucchi.
\newblock Sdes for minimax optimization.
\newblock In \emph{International Conference on Artificial Intelligence and Statistics}, pages 4834--4842. PMLR, 2024.

\bibitem[Compagnoni et~al.(2025)Compagnoni, Liu, Islamov, Proske, Orvieto, and Lucchi]{compagnoni2025adaptive}
Enea~Monzio Compagnoni, Tianlin Liu, Rustem Islamov, Frank~Norbert Proske, Antonio Orvieto, and Aurelien Lucchi.
\newblock Adaptive methods through the lens of {SDE}s: Theoretical insights on the role of noise.
\newblock In \emph{The Thirteenth International Conference on Learning Representations}, 2025.
\newblock URL \url{https://openreview.net/forum?id=ww3CLRhF1v}.

\bibitem[Condat et~al.(2022)Condat, Yi, and Richt{\'a}rik]{condat2022ef}
Laurent Condat, Kai Yi, and Peter Richt{\'a}rik.
\newblock {EF-BV}: A unified theory of error feedback and variance reduction mechanisms for biased and unbiased compression in distributed optimization.
\newblock \emph{Advances in Neural Information Processing Systems}, 2022.

\bibitem[Condat et~al.(2024)Condat, Maranjyan, and Richt{\'a}rik]{condat2024locodl}
Laurent Condat, Artavazd Maranjyan, and Peter Richt{\'a}rik.
\newblock {LoCoDL}: Communication-efficient distributed learning with local training and compression.
\newblock \emph{arXiv preprint arXiv:2403.04348}, 2024.

\bibitem[Cui et~al.(2020)Cui, Fan, and Jia]{cui2020momentum}
Zhuo-Xu Cui, Qibin Fan, and Cui Jia.
\newblock Momentum methods for stochastic optimization over time-varying directed networks.
\newblock \emph{Signal Processing}, 2020.

\bibitem[Dambrine et~al.(2024)Dambrine, Dossal, Puig, and Rondepierre]{dambrine2024stochastic}
Marc Dambrine, Ch~Dossal, B{\'e}n{\'e}dicte Puig, and Aude Rondepierre.
\newblock Stochastic differential equations for modeling first order optimization methods.
\newblock \emph{SIAM Journal on Optimization}, 2024.

\bibitem[Dean et~al.(2012)Dean, Corrado, Monga, Chen, Devin, Mao, Ranzato, Senior, Tucker, Yang, et~al.]{dean2012large}
Jeffrey Dean, Greg Corrado, Rajat Monga, Kai Chen, Matthieu Devin, Mark Mao, Marc'aurelio Ranzato, Andrew Senior, Paul Tucker, Ke~Yang, et~al.
\newblock Large scale distributed deep networks.
\newblock \emph{Advances in neural information processing systems}, 2012.

\bibitem[Del~Moral and Niclas(2018)]{MN}
Pierre Del~Moral and Angele Niclas.
\newblock A taylor expansion of the square root matrix function.
\newblock \emph{Journal of Mathematical Analysis and Applications}, 465\penalty0 (1):\penalty0 259--266, 2018.

\bibitem[Deng(2012)]{deng2012mnist}
Li~Deng.
\newblock The {MNIST} database of handwritten digit images for machine learning research.
\newblock \emph{IEEE Signal Processing Magazine}, 2012.

\bibitem[Deng et~al.(2024)Deng, Kamani, Mahdavinia, and Mahdavi]{deng2024distributed}
Yuyang Deng, Mohammad~Mahdi Kamani, Pouria Mahdavinia, and Mehrdad Mahdavi.
\newblock Distributed personalized empirical risk minimization.
\newblock \emph{Advances in Neural Information Processing Systems}, 2024.

\bibitem[Devlin(2018)]{devlin2018bert}
Jacob Devlin.
\newblock Bert: Pre-training of deep bidirectional transformers for language understanding.
\newblock \emph{arXiv preprint arXiv:1810.04805}, 2018.

\bibitem[Dosovitskiy et~al.(2021)Dosovitskiy, Beyer, Kolesnikov, Weissenborn, Zhai, Unterthiner, Dehghani, Minderer, Heigold, Gelly, Uszkoreit, and Houlsby]{dosovitskiy2020image}
Alexey Dosovitskiy, Lucas Beyer, Alexander Kolesnikov, Dirk Weissenborn, Xiaohua Zhai, Thomas Unterthiner, Mostafa Dehghani, Matthias Minderer, Georg Heigold, Sylvain Gelly, Jakob Uszkoreit, and Neil Houlsby.
\newblock An image is worth 16x16 words: {T}ransformers for image recognition at {Scale}.
\newblock In \emph{International Conference on Learning Representations}, 2021.

\bibitem[Dua and Graff(2017)]{Dua:2019}
Dheeru Dua and Casey Graff.
\newblock {UCI} machine learning repository, 2017.
\newblock URL \url{http://archive.ics.uci.edu/ml}.

\bibitem[Fatkhullin et~al.(2024)Fatkhullin, Tyurin, and Richt{\'a}rik]{fatkhullin2024momentum}
Ilyas Fatkhullin, Alexander Tyurin, and Peter Richt{\'a}rik.
\newblock Momentum provably improves error feedback!
\newblock \emph{Advances in Neural Information Processing Systems}, 2024.

\bibitem[Fontaine et~al.(2021)Fontaine, De~Bortoli, and Durmus]{fontaine2021convergence}
Xavier Fontaine, Valentin De~Bortoli, and Alain Durmus.
\newblock Convergence rates and approximation results for {SGD} and its continuous-time counterpart.
\newblock In \emph{Conference on Learning Theory}, 2021.

\bibitem[Gao et~al.(2023)Gao, Islamov, and Stich]{gao2023econtrol}
Yuan Gao, Rustem Islamov, and Sebastian Stich.
\newblock {EControl}: Fast distributed optimization with compression and error control.
\newblock \emph{arXiv preprint arXiv:2311.05645}, 2023.

\bibitem[Garrigos and Gower(2023)]{garrigos2023handbook}
Guillaume Garrigos and Robert~M Gower.
\newblock Handbook of convergence theorems for (stochastic) gradient methods.
\newblock \emph{arXiv preprint arXiv:2301.11235}, 2023.

\bibitem[Ge et~al.(2015)Ge, Huang, Jin, and Yuan]{ge2015escaping}
Rong Ge, Furong Huang, Chi Jin, and Yang Yuan.
\newblock Escaping from saddle points—online stochastic gradient for tensor decomposition.
\newblock In \emph{Conference on Learning Theory}, 2015.

\bibitem[Gess et~al.(2024)Gess, Kassing, and Konarovskyi]{gess2024stochastic}
Benjamin Gess, Sebastian Kassing, and Vitalii Konarovskyi.
\newblock Stochastic modified flows, mean-field limits and dynamics of stochastic gradient descent.
\newblock \emph{Journal of Machine Learning Research}, 2024.

\bibitem[Gorbunov et~al.(2020)Gorbunov, Hanzely, and Richt{\'a}rik]{gorbunov2020unified}
Eduard Gorbunov, Filip Hanzely, and Peter Richt{\'a}rik.
\newblock A unified theory of {SGD}: Variance reduction, sampling, quantization and coordinate descent.
\newblock In \emph{International Conference on Artificial Intelligence and Statistics}, 2020.

\bibitem[Gorbunov et~al.(2021)Gorbunov, Hanzely, and Richt{\'a}rik]{gorbunov2021local}
Eduard Gorbunov, Filip Hanzely, and Peter Richt{\'a}rik.
\newblock {Local SGD}: Unified theory and new efficient methods.
\newblock In \emph{International Conference on Artificial Intelligence and Statistics}, 2021.

\bibitem[Gorbunov et~al.(2023)Gorbunov, Sadiev, Danilova, Horv{\'a}th, Gidel, Dvurechensky, Gasnikov, and Richt{\'a}rik]{gorbunov2023high}
Eduard Gorbunov, Abdurakhmon Sadiev, Marina Danilova, Samuel Horv{\'a}th, Gauthier Gidel, Pavel Dvurechensky, Alexander Gasnikov, and Peter Richt{\'a}rik.
\newblock High-probability convergence for composite and distributed stochastic minimization and variational inequalities with heavy-tailed noise.
\newblock \emph{arXiv preprint arXiv:2310.01860}, 2023.

\bibitem[Gu et~al.(2021)Gu, Guo, and Li]{gu2021adversarial}
Haotian Gu, Xin Guo, and Xinyu Li.
\newblock Adversarial training for gradient descent: Analysis through its continuous-time approximation.
\newblock \emph{arXiv preprint arXiv:2105.08037}, 2021.

\bibitem[Gurbuzbalaban et~al.(2021)Gurbuzbalaban, Simsekli, and Zhu]{gurbuzbalaban2021heavy}
Mert Gurbuzbalaban, Umut Simsekli, and Lingjiong Zhu.
\newblock The heavy-tail phenomenon in {SGD}.
\newblock In \emph{International Conference on Machine Learning}, 2021.

\bibitem[Hardt et~al.(2018)Hardt, Ma, and Recht]{hardt2018gradient}
Moritz Hardt, Tengyu Ma, and Benjamin Recht.
\newblock Gradient descent learns linear dynamical systems.
\newblock \emph{Journal of Machine Learning Research}, 19\penalty0 (29):\penalty0 1--44, 2018.

\bibitem[Harris et~al.(2020)Harris, Millman, van~der Walt, Gommers, Virtanen, Cournapeau, Wieser, Taylor, Berg, Smith, Kern, Picus, Hoyer, van Kerkwijk, Brett, Haldane, del R{\'{i}}o, Wiebe, Peterson, G{\'{e}}rard-Marchant, Sheppard, Reddy, Weckesser, Abbasi, Gohlke, and Oliphant]{harris2020array}
Charles~R. Harris, K.~Jarrod Millman, St{\'{e}}fan~J. van~der Walt, Ralf Gommers, Pauli Virtanen, David Cournapeau, Eric Wieser, Julian Taylor, Sebastian Berg, Nathaniel~J. Smith, Robert Kern, Matti Picus, Stephan Hoyer, Marten~H. van Kerkwijk, Matthew Brett, Allan Haldane, Jaime~Fern{\'{a}}ndez del R{\'{i}}o, Mark Wiebe, Pearu Peterson, Pierre G{\'{e}}rard-Marchant, Kevin Sheppard, Tyler Reddy, Warren Weckesser, Hameer Abbasi, Christoph Gohlke, and Travis~E. Oliphant.
\newblock Array programming with {NumPy}.
\newblock \emph{Nature}, 2020.

\bibitem[Horv{\'a}th et~al.(2022)Horv{\'a}th, Ho, Horvath, Sahu, Canini, and Richt{\'a}rik]{horvoth2022natural}
Samuel Horv{\'a}th, Chen-Yu Ho, Ludovit Horvath, Atal~Narayan Sahu, Marco Canini, and Peter Richt{\'a}rik.
\newblock Natural compression for distributed deep learning.
\newblock In \emph{Mathematical and Scientific Machine Learning}, 2022.

\bibitem[Horv{\'a}th et~al.(2023)Horv{\'a}th, Kovalev, Mishchenko, Richt{\'a}rik, and Stich]{horvath2023stochastic}
Samuel Horv{\'a}th, Dmitry Kovalev, Konstantin Mishchenko, Peter Richt{\'a}rik, and Sebastian Stich.
\newblock Stochastic distributed learning with gradient quantization and double-variance reduction.
\newblock \emph{Optimization Methods and Software}, 2023.

\bibitem[Hu et~al.(2019)Hu, Li, and Zhou]{hu2019global}
Wenqing Hu, Chris~Junchi Li, and Xiang Zhou.
\newblock On the global convergence of continuous--time stochastic heavy--ball method for nonconvex optimization.
\newblock In \emph{2019 IEEE International Conference on Big Data (Big Data)}, 2019.

\bibitem[Ikeda and Watanabe(2014)]{IW}
Nobuyuki Ikeda and Shinzo Watanabe.
\newblock \emph{Stochastic differential equations and diffusion processes}.
\newblock Elsevier, 2014.

\bibitem[Islamov et~al.(2021)Islamov, Qian, and Richt{\'a}rik]{islamov2021distributed}
Rustem Islamov, Xun Qian, and Peter Richt{\'a}rik.
\newblock Distributed second order methods with fast rates and compressed communication.
\newblock In \emph{International conference on machine learning}, pages 4617--4628. PMLR, 2021.

\bibitem[Islamov et~al.(2023)Islamov, Qian, Hanzely, Safaryan, and Richt{\'a}rik]{islamov2023distributed}
Rustem Islamov, Xun Qian, Slavom{\'\i}r Hanzely, Mher Safaryan, and Peter Richt{\'a}rik.
\newblock Distributed newton-type methods with communication compression and bernoulli aggregation.
\newblock \emph{Transactions on Machine Learning Research}, 2023.

\bibitem[Islamov et~al.(2024{\natexlab{a}})Islamov, Ajroldi, Orvieto, and Lucchi]{islamov2024loss}
Rustem Islamov, Niccol{\`o} Ajroldi, Antonio Orvieto, and Aurelien Lucchi.
\newblock Loss landscape characterization of neural networks without over-parametrization.
\newblock \emph{arXiv preprint arXiv:2410.12455}, 2024{\natexlab{a}}.

\bibitem[Islamov et~al.(2024{\natexlab{b}})Islamov, Safaryan, and Alistarh]{islamov2024asgrad}
Rustem Islamov, Mher Safaryan, and Dan Alistarh.
\newblock {AsGrad}: A sharp unified analysis of asynchronous-{SGD} algorithms.
\newblock In \emph{International Conference on Artificial Intelligence and Statistics}, 2024{\natexlab{b}}.

\bibitem[Jastrzebski et~al.(2018)Jastrzebski, Kenton, Arpit, Ballas, Fischer, Bengio, and Storkey]{jastrzkebski2017three}
Stanis{\l}aw Jastrzebski, Zachary Kenton, Devansh Arpit, Nicolas Ballas, Asja Fischer, Yoshua Bengio, and Amos Storkey.
\newblock Three factors influencing minima in sgd.
\newblock \emph{ICANN 2018}, 2018.

\bibitem[Jin et~al.(2017)Jin, Ge, Netrapalli, Kakade, and Jordan]{jin2017escape}
Chi Jin, Rong Ge, Praneeth Netrapalli, Sham~M Kakade, and Michael~I Jordan.
\newblock How to escape saddle points efficiently.
\newblock In \emph{International Conference on Machine Learning}, 2017.

\bibitem[Jouppi et~al.(2017)Jouppi, Young, Patil, Patterson, Agrawal, Bajwa, Bates, Bhatia, Boden, Borchers, et~al.]{jouppi2017datacenter}
Norman~P Jouppi, Cliff Young, Nishant Patil, David Patterson, Gaurav Agrawal, Raminder Bajwa, Sarah Bates, Suresh Bhatia, Nan Boden, Al~Borchers, et~al.
\newblock In-datacenter performance analysis of a tensor processing unit.
\newblock In \emph{Proceedings of the 44th annual international symposium on computer architecture}, 2017.

\bibitem[Karimi et~al.(2016)Karimi, Nutini, and Schmidt]{karimi2016linear}
Hamed Karimi, Julie Nutini, and Mark Schmidt.
\newblock Linear convergence of gradient and proximal-gradient methods under the {P}olyak-{\l}ojasiewicz condition.
\newblock In \emph{Machine Learning and Knowledge Discovery in Databases: European Conference, ECML PKDD 2016, Riva del Garda, Italy, September 19-23, 2016, Proceedings, Part I 16}, pages 795--811. Springer, 2016.

\bibitem[Khirirat et~al.(2018)Khirirat, Feyzmahdavian, and Johansson]{khirirat2018distributed}
Sarit Khirirat, Hamid~Reza Feyzmahdavian, and Mikael Johansson.
\newblock Distributed learning with compressed gradients.
\newblock \emph{arXiv preprint arXiv:1806.06573}, 2018.

\bibitem[Khirirat et~al.(2023)Khirirat, Gorbunov, Horváth, Islamov, Karray, and Richtárik]{khirirat2023clip21}
Sarit Khirirat, Eduard Gorbunov, Samuel Horváth, Rustem Islamov, Fakhri Karray, and Peter Richtárik.
\newblock {Clip21}: Error feedback for gradient clipping.
\newblock \emph{arXiv preprint: arXiv 2305.18929}, 2023.

\bibitem[Koloskova et~al.(2020)Koloskova, Loizou, Boreiri, Jaggi, and Stich]{koloskova2020unified}
Anastasia Koloskova, Nicolas Loizou, Sadra Boreiri, Martin Jaggi, and Sebastian Stich.
\newblock A unified theory of decentralized {SGD} with changing topology and local updates.
\newblock In \emph{International Conference on Machine Learning}, 2020.

\bibitem[Krizhevsky et~al.(2009)Krizhevsky, Hinton, et~al.]{krizhevsky2009learning}
Alex Krizhevsky, Geoffrey Hinton, et~al.
\newblock Learning multiple layers of features from tiny images.
\newblock \emph{Toronto, ON, Canada}, 2009.

\bibitem[Kunin et~al.(2023)Kunin, Sagastuy-Brena, Gillespie, Margalit, Tanaka, Ganguli, and Yamins]{kunin2023limiting}
Daniel Kunin, Javier Sagastuy-Brena, Lauren Gillespie, Eshed Margalit, Hidenori Tanaka, Surya Ganguli, and Daniel~LK Yamins.
\newblock The limiting dynamics of {SGD}: Modified loss, phase-space oscillations, and anomalous diffusion.
\newblock \emph{Neural Computation}, 2023.

\bibitem[Kunstner et~al.(2024)Kunstner, Yadav, Milligan, Schmidt, and Bietti]{kunstner2024heavy}
Frederik Kunstner, Robin Yadav, Alan Milligan, Mark Schmidt, and Alberto Bietti.
\newblock Heavy-tailed class imbalance and why adam outperforms gradient descent on language models.
\newblock \emph{arXiv preprint arXiv:2402.19449}, 2024.

\bibitem[Lanconelli and Lauria(2022)]{lanconelli2022note}
Alberto Lanconelli and Christopher~SA Lauria.
\newblock A note on diffusion limits for stochastic gradient descent.
\newblock \emph{arXiv preprint arXiv:2210.11257}, 2022.

\bibitem[Levy(2016)]{levy2016power}
Kfir~Y Levy.
\newblock The power of normalization: Faster evasion of saddle points.
\newblock \emph{arXiv preprint arXiv:1611.04831}, 2016.

\bibitem[Li and Chi(2023)]{li2023convergenceandprivacy}
Boyue Li and Yuejie Chi.
\newblock Convergence and privacy of decentralized nonconvex optimization with gradient clipping and communication compression.
\newblock \emph{arXiv preprint arXiv:2305.09896}, 2023.

\bibitem[Li and Wang(2022)]{li2022uniform}
Lei Li and Yuliang Wang.
\newblock On uniform-in-time diffusion approximation for stochastic gradient descent.
\newblock \emph{arXiv preprint arXiv:2207.04922}, 2022.

\bibitem[Li et~al.(2017)Li, Tai, and Weinan]{li2017stochastic}
Qianxiao Li, Cheng Tai, and E~Weinan.
\newblock Stochastic modified equations and adaptive stochastic gradient algorithms.
\newblock In \emph{International Conference on Machine Learning}, pages 2101--2110. PMLR, 2017.

\bibitem[Li et~al.(2019)Li, Tai, and Weinan]{li2019stochastic}
Qianxiao Li, Cheng Tai, and E~Weinan.
\newblock Stochastic modified equations and dynamics of stochastic gradient algorithms i: Mathematical foundations.
\newblock \emph{The Journal of Machine Learning Research}, 2019.

\bibitem[Li et~al.(2021)Li, Malladi, and Arora]{Li2021validity}
Zhiyuan Li, Sadhika Malladi, and Sanjeev Arora.
\newblock On the validity of modeling {SGD} with stochastic differential equations ({SDE}s).
\newblock In A.~Beygelzimer, Y.~Dauphin, P.~Liang, and J.~Wortman Vaughan, editors, \emph{Advances in Neural Information Processing Systems}, 2021.

\bibitem[Li et~al.(2023)Li, Wang, and Wang]{li2023fast}
Zhiyuan Li, Yi~Wang, and Zhiren Wang.
\newblock Fast equilibrium of {SGD} in generic situations.
\newblock In \emph{The Twelfth International Conference on Learning Representations}, 2023.

\bibitem[Li and Richt{\'a}rik(2020)]{li2020unified}
Zhize Li and Peter Richt{\'a}rik.
\newblock A unified analysis of stochastic gradient methods for nonconvex federated optimization.
\newblock \emph{arXiv preprint arXiv:2006.07013}, 2020.

\bibitem[Liu et~al.(2024)Liu, Drusvyatskiy, Belkin, Davis, and Ma]{liu2024aiming}
Chaoyue Liu, Dmitriy Drusvyatskiy, Misha Belkin, Damek Davis, and Yian Ma.
\newblock Aiming towards the minimizers: fast convergence of sgd for overparametrized problems.
\newblock \emph{Advances in neural information processing systems}, 36, 2024.

\bibitem[Liu et~al.(2021)Liu, Chen, Zhou, and Zhao]{liu2018diffusion}
Tianyi Liu, Zhehui Chen, Enlu Zhou, and Tuo Zhao.
\newblock A diffusion approximation theory of momentum stochastic gradient descent in nonconvex optimization.
\newblock \emph{Stochastic Systems}, 2021.

\bibitem[Malladi et~al.(2022)Malladi, Lyu, Panigrahi, and Arora]{Malladi2022AdamSDE}
Sadhika Malladi, Kaifeng Lyu, Abhishek Panigrahi, and Sanjeev Arora.
\newblock On the {SDEs} and scaling rules for adaptive gradient algorithms.
\newblock In \emph{Advances in Neural Information Processing Systems}, 2022.

\bibitem[Mandt et~al.(2016)Mandt, Hoffman, and Blei]{mandt2016variational}
Stephan Mandt, Matthew Hoffman, and David Blei.
\newblock A variational analysis of stochastic gradient algorithms.
\newblock In \emph{International conference on machine learning}, 2016.

\bibitem[Mandt et~al.(2017)Mandt, Hoffman, and Blei]{mandt2017stochastic}
Stephan Mandt, Matthew~D Hoffman, and David~M Blei.
\newblock Stochastic gradient descent as approximate bayesian inference.
\newblock \emph{JMLR}, 2017.

\bibitem[Mao(2007)]{mao2007stochastic}
Xuerong Mao.
\newblock \emph{Stochastic differential equations and applications}.
\newblock Elsevier, 2007.

\bibitem[Marfoq et~al.(2023)Marfoq, Neglia, Kameni, and Vidal]{marfoq2023federated}
Othmane Marfoq, Giovanni Neglia, Laetitia Kameni, and Richard Vidal.
\newblock Federated learning for data streams.
\newblock In \emph{International Conference on Artificial Intelligence and Statistics}, 2023.

\bibitem[Maulen-Soto et~al.(2024)Maulen-Soto, Fadili, Attouch, and Ochs]{maulen2024stochastic}
Rodrigo Maulen-Soto, Jalal Fadili, Hedy Attouch, and Peter Ochs.
\newblock Stochastic inertial dynamics via time scaling and averaging.
\newblock \emph{arXiv preprint arXiv:2403.16775}, 2024.

\bibitem[Maul{\'e}n~Soto(2021)]{maulen2021continuous}
Rodrigo~Ignacio Maul{\'e}n~Soto.
\newblock A continuous-time model of stochastic gradient descent: convergence rates and complexities under lojasiewicz inequality.
\newblock \emph{Universidad de Chile}, 2021.

\bibitem[Mil’shtein(1986)]{mil1986weak}
GN~Mil’shtein.
\newblock Weak approximation of solutions of systems of stochastic differential equations.
\newblock \emph{Theory of Probability \& Its Applications}, 1986.

\bibitem[Mishchenko et~al.(2022)Mishchenko, Bach, Even, and Woodworth]{mishchenko2022asynchronous}
Konstantin Mishchenko, Francis Bach, Mathieu Even, and Blake~E Woodworth.
\newblock Asynchronous {SGD} beats minibatch {SGD} under arbitrary delays.
\newblock \emph{Advances in Neural Information Processing Systems}, 2022.

\bibitem[Mishchenko et~al.(2024)Mishchenko, Gorbunov, Tak{\'a}{\v{c}}, and Richt{\'a}rik]{mishchenko2024distributed}
Konstantin Mishchenko, Eduard Gorbunov, Martin Tak{\'a}{\v{c}}, and Peter Richt{\'a}rik.
\newblock Distributed learning with compressed gradient differences.
\newblock \emph{Optimization Methods and Software}, 2024.

\bibitem[Orvieto and Lucchi(2019)]{orvieto2019continuous}
Antonio Orvieto and Aurelien Lucchi.
\newblock Continuous-time models for stochastic optimization algorithms.
\newblock \emph{Advances in Neural Information Processing Systems}, 2019.

\bibitem[Paquette et~al.(2021)Paquette, Lee, Pedregosa, and Paquette]{paquette2021sgd}
Courtney Paquette, Kiwon Lee, Fabian Pedregosa, and Elliot Paquette.
\newblock {SGD} in the large: Average-case analysis, asymptotics, and stepsize criticality.
\newblock In \emph{Conference on Learning Theory}. PMLR, 2021.

\bibitem[Pedregosa et~al.(2011)Pedregosa, Varoquaux, Gramfort, Michel, Thirion, Grisel, Blondel, Prettenhofer, Weiss, Dubourg, Vanderplas, Passos, Cournapeau, Brucher, Perrot, and Duchesnay]{scikit-learn}
F.~Pedregosa, G.~Varoquaux, A.~Gramfort, V.~Michel, B.~Thirion, O.~Grisel, M.~Blondel, P.~Prettenhofer, R.~Weiss, V.~Dubourg, J.~Vanderplas, A.~Passos, D.~Cournapeau, M.~Brucher, M.~Perrot, and E.~Duchesnay.
\newblock Scikit-learn: Machine learning in {P}ython.
\newblock \emph{Journal of Machine Learning Research}, 2011.

\bibitem[Philippenko and Dieuleveut(2024)]{PD2025}
Constantin Philippenko and Aymeric Dieuleveut.
\newblock Compressed and distributed least-squares regression: convergence rates with applications to federated learning.
\newblock \emph{Journal of Machine Learning Research}, 25\penalty0 (288):\penalty0 1--80, 2024.
\newblock URL \url{http://jmlr.org/papers/v25/23-1040.html}.

\bibitem[Poggio et~al.(2017)Poggio, Kawaguchi, Liao, Miranda, Rosasco, Boix, Hidary, and Mhaskar]{poggio2017theory}
Tomaso Poggio, Kenji Kawaguchi, Qianli Liao, Brando Miranda, Lorenzo Rosasco, Xavier Boix, Jack Hidary, and Hrushikesh Mhaskar.
\newblock Theory of deep learning {III}: explaining the non-overfitting puzzle.
\newblock \emph{arXiv preprint arXiv:1801.00173}, 2017.

\bibitem[Qian et~al.(2021)Qian, Islamov, Safaryan, and Richt{\'a}rik]{qian2021basis}
Xun Qian, Rustem Islamov, Mher Safaryan, and Peter Richt{\'a}rik.
\newblock Basis matters: better communication-efficient second order methods for federated learning.
\newblock \emph{arXiv preprint arXiv:2111.01847}, 2021.

\bibitem[Richt{\'a}rik et~al.(2021)Richt{\'a}rik, Sokolov, and Fatkhullin]{richtarik2021ef21}
Peter Richt{\'a}rik, Igor Sokolov, and Ilyas Fatkhullin.
\newblock {EF21}: A new, simpler, theoretically better, and practically faster error feedback.
\newblock \emph{Advances in Neural Information Processing Systems}, 2021.

\bibitem[Safaryan and Richtarik(2021)]{safaryan2021signsgd}
Mher Safaryan and Peter Richtarik.
\newblock Stochastic sign descent methods: New algorithms and better theory.
\newblock In \emph{Proceedings of the 38th International Conference on Machine Learning}, 2021.

\bibitem[Safaryan et~al.(2021)Safaryan, Islamov, Qian, and Richt{\'a}rik]{safaryan2021fednl}
Mher Safaryan, Rustem Islamov, Xun Qian, and Peter Richt{\'a}rik.
\newblock {FedNL}: Making newton-type methods applicable to federated learning.
\newblock \emph{arXiv preprint arXiv:2106.02969}, 2021.

\bibitem[Sapio et~al.(2021)Sapio, Canini, Ho, Nelson, Kalnis, Kim, Krishnamurthy, Moshref, Ports, and Richt{\'a}rik]{sapio2021scaling}
Amedeo Sapio, Marco Canini, Chen-Yu Ho, Jacob Nelson, Panos Kalnis, Changhoon Kim, Arvind Krishnamurthy, Masoud Moshref, Dan Ports, and Peter Richt{\'a}rik.
\newblock Scaling distributed machine learning with $\{$In-Network$\}$ aggregation.
\newblock In \emph{18th USENIX Symposium on Networked Systems Design and Implementation (NSDI 21)}, 2021.

\bibitem[Seide et~al.(2014)Seide, Fu, Droppo, Li, and Yu]{seide20141}
Frank Seide, Hao Fu, Jasha Droppo, Gang Li, and Dong Yu.
\newblock 1-bit stochastic gradient descent and its application to data-parallel distributed training of speech {DNNs}.
\newblock In \emph{Interspeech}, 2014.

\bibitem[Shamir and Srebro(2014)]{shamir2014distributed}
Ohad Shamir and Nathan Srebro.
\newblock Distributed stochastic optimization and learning.
\newblock In \emph{2014 52nd Annual Allerton Conference on Communication, Control, and Computing (Allerton)}, 2014.

\bibitem[Simsekli et~al.(2019)Simsekli, Sagun, and Gurbuzbalaban]{simsekli2019tail}
Umut Simsekli, Levent Sagun, and Mert Gurbuzbalaban.
\newblock A tail-index analysis of stochastic gradient noise in deep neural networks.
\newblock In \emph{International Conference on Machine Learning}, 2019.

\bibitem[Smith et~al.(2021)Smith, Dherin, Barrett, and De]{smith2021origin}
Samuel~L. Smith, Benoit Dherin, David G.~T. Barrett, and Soham De.
\newblock On the origin of implicit regularization in stochastic gradient descent.
\newblock \emph{arXiv preprint arXiv: 2101.12176}, 2021.

\bibitem[Soto et~al.(2022)Soto, Fadili, and Attouch]{soto2022sde}
Rodrigo~Maulen Soto, Jalal Fadili, and Hedy Attouch.
\newblock An {SDE} perspective on stochastic convex optimization.
\newblock \emph{arXiv preprint arXiv:2207.02750}, 2022.

\bibitem[Stephan et~al.(2017)Stephan, Hoffman, Blei, et~al.]{stephan2017stochastic}
Mandt Stephan, Matthew~D Hoffman, David~M Blei, et~al.
\newblock Stochastic gradient descent as approximate bayesian inference.
\newblock \emph{Journal of Machine Learning Research}, 2017.

\bibitem[Su and Lau(2023)]{su2023accelerated}
Liqun Su and Vincent~KN Lau.
\newblock Accelerated federated learning over wireless fading channels with adaptive stochastic momentum.
\newblock \emph{IEEE Internet of Things Journal}, 2023.

\bibitem[Sun(2023)]{sun2023distributed}
Chao Sun.
\newblock Distributed stochastic optimization under heavy-tailed noises.
\newblock \emph{arXiv preprint arXiv:2312.15847}, 2023.

\bibitem[Sun et~al.(2023)Sun, Yang, Xun, and Zhang]{sun2023scheduling}
Jianhui Sun, Ying Yang, Guangxu Xun, and Aidong Zhang.
\newblock Scheduling hyperparameters to improve generalization: From centralized {SGD} to asynchronous {SGD}.
\newblock \emph{ACM Transactions on Knowledge Discovery from Data}, 2023.

\bibitem[Van~Rossum and Drake(2009)]{10.5555/1593511}
Guido Van~Rossum and Fred~L. Drake.
\newblock \emph{Python 3 Reference Manual}.
\newblock CreateSpace, Scotts Valley, CA, 2009.
\newblock ISBN 1441412697.

\bibitem[Vogels et~al.(2019)Vogels, Karimireddy, and Jaggi]{vogels2019powersgd}
Thijs Vogels, Sai~Praneeth Karimireddy, and Martin Jaggi.
\newblock {PowerSGD}: Practical low-rank gradient compression for distributed optimization.
\newblock \emph{Advances in Neural Information Processing Systems}, 2019.

\bibitem[Wang et~al.(2017)Wang, Wang, and Srebro]{wang2017memory}
Jialei Wang, Weiran Wang, and Nathan Srebro.
\newblock Memory and communication efficient distributed stochastic optimization with minibatch prox.
\newblock In \emph{Conference on Learning Theory}, 2017.

\bibitem[Wang and Wu(2020)]{wang2020asymptotic}
Yazhen Wang and Shang Wu.
\newblock Asymptotic analysis via stochastic differential equations of gradient descent algorithms in statistical and computational paradigms.
\newblock \emph{Journal of machine learning research}, 2020.

\bibitem[Wang and Mao(2022)]{wang2022two}
Ziqiao Wang and Yongyi Mao.
\newblock Two facets of {SDE} under an information-theoretic lens: Generalization of {SGD} via training trajectories and via terminal states.
\newblock \emph{arXiv preprint arXiv:2211.10691}, 2022.

\bibitem[Wolkowicz and Styan(1980)]{WS}
Henry Wolkowicz and George~PH Styan.
\newblock Bounds for eigenvalues using traces.
\newblock \emph{Linear algebra and its applications}, 29:\penalty0 471--506, 1980.

\bibitem[Wu et~al.(2020)Wu, Hu, Xiong, Huan, Braverman, and Zhu]{wu2020noisy}
Jingfeng Wu, Wenqing Hu, Haoyi Xiong, Jun Huan, Vladimir Braverman, and Zhanxing Zhu.
\newblock On the noisy gradient descent that generalizes as {SGD}.
\newblock In \emph{International Conference on Machine Learning}, 2020.

\bibitem[Xiao et~al.(2024)Xiao, Marshall, Agarwala, and Paquette]{xiao2024exact}
Ke~Liang Xiao, Noah Marshall, Atish Agarwala, and Elliot Paquette.
\newblock Exact risk curves of signsgd in high-dimensions: Quantifying preconditioning and noise-compression effects.
\newblock \emph{arXiv preprint arXiv:2411.12135}, 2024.

\bibitem[Xie et~al.(2021)Xie, Yuan, Zhu, and Sugiyama]{Xie2021}
Zeke Xie, Li~Yuan, Zhanxing Zhu, and Masashi Sugiyama.
\newblock Positive-negative momentum: Manipulating stochastic gradient noise to improve generalization.
\newblock In \emph{Proceedings of the 38th International Conference on Machine Learning}, 2021.

\bibitem[Yang et~al.(2022)Yang, Qiu, and Liu]{yang2022taming}
Haibo Yang, Peiwen Qiu, and Jia Liu.
\newblock Taming fat-tailed (“heavier-tailed” with potentially infinite variance) noise in federated learning.
\newblock \emph{Advances in Neural Information Processing Systems}, 2022.

\bibitem[Yu et~al.(2023)Yu, Jakovetic, and Kar]{yu2023smoothed}
Shuhua Yu, Dusan Jakovetic, and Soummya Kar.
\newblock Smoothed gradient clipping and error feedback for distributed optimization under heavy-tailed noise.
\newblock \emph{arXiv preprint arXiv:2310.16920}, 2023.

\bibitem[Yu et~al.(2019)Yu, Wu, and Huang]{yu2019double}
Yue Yu, Jiaxiang Wu, and Longbo Huang.
\newblock Double quantization for communication-efficient distributed optimization.
\newblock \emph{Advances in neural information processing systems}, 2019.

\bibitem[Zhang et~al.(2020)Zhang, Karimireddy, Veit, Kim, Reddi, Kumar, and Sra]{zhang2020adaptive}
Jingzhao Zhang, Sai~Praneeth Karimireddy, Andreas Veit, Seungyeon Kim, Sashank Reddi, Sanjiv Kumar, and Suvrit Sra.
\newblock Why are adaptive methods good for attention models?
\newblock \emph{Advances in Neural Information Processing Systems}, 2020.

\bibitem[Zhang et~al.(2023)Zhang, Li, Luo, and Xu]{zhang2023stochastic}
Zhongwang Zhang, Yuqing Li, Tao Luo, and Zhi-Qin~John Xu.
\newblock Stochastic modified equations and dynamics of dropout algorithm.
\newblock \emph{arXiv preprint arXiv:2305.15850}, 2023.

\bibitem[Zhao et~al.(2022)Zhao, Lucchi, Proske, Orvieto, and Kersting]{zhao2022batch}
Jim Zhao, Aurelien Lucchi, Frank~Norbert Proske, Antonio Orvieto, and Hans Kersting.
\newblock Batch size selection by stochastic optimal control.
\newblock In \emph{Has it Trained Yet? NeurIPS 2022 Workshop}, 2022.

\bibitem[Zhao et~al.(2018)Zhao, Zhang, Li, and Li]{zhao2018proximal}
Shenyi Zhao, Gong-Duo Zhang, Ming-Wei Li, and Wu-Jun Li.
\newblock Proximal {SCOPE} for distributed sparse learning.
\newblock \emph{Advances in Neural Information Processing Systems}, 2018.

\bibitem[Zhou et~al.(2020)Zhou, Yuan, Li, and Sun]{zhou2020stochastic}
Xiang Zhou, Huizhuo Yuan, Chris~Junchi Li, and Qingyun Sun.
\newblock Stochastic modified equations for continuous limit of stochastic admm.
\newblock \emph{arXiv preprint arXiv:2003.03532}, 2020.

\bibitem[Zhu and Ying(2021)]{zhu2020sharp}
Yuhua Zhu and Lexing Ying.
\newblock A sharp convergence rate for a model equation of the asynchronous stochastic gradient descent.
\newblock \emph{Communications in Mathematical Sciences}, 2021.

\bibitem[Zhu et~al.(2019)Zhu, Wu, Yu, Wu, and Ma]{zhu2019anisotropic}
Zhanxing Zhu, Jingfeng Wu, Bing Yu, Lei Wu, and Jinwen Ma.
\newblock The anisotropic noise in stochastic gradient descent: Its behavior of escaping from sharp minima and regularization effects.
\newblock \emph{ICML}, 2019.

\end{thebibliography}
\bibliographystyle{plainnat}

\newpage

\appendix

%%%%%%%%%%%%%%%%%%%%%%%%%%%%%%%%%%%%%%%%%%%%%%%%%%%%%%%%%%%%%%%%%%%%%%%%%%%%
\section{Theoretical framework - Weak Approximation}\label{sec:theor}
%%%%%%%%%%%%%%%%%%%%%%%%%%%%%%%%%%%%%%%%%%%%%%%%%%%%%%%%%%%%%%%%%%%%%%%%%%%%

In this section, we introduce the theoretical framework used in the paper, together with its assumptions and notations.

First of all, many proofs will use Taylor expansions in powers of $ \eta $. For ease of notation,  we introduce the shorthand that whenever we write $ \mathcal{O}\left(\eta^\alpha\right) $, we mean that there exists a function $ K(x) \in G $ such that the error terms are bounded by $ K(x) \eta^\alpha $. For example, we write
$$
b(x+\eta)=b_0(x)+\eta b_1(x)+\mathcal{O}\left(\eta^2\right)
$$
to mean: there exists $ K \in G $ such that
$$
\left|b(x+\eta)-b_0(x)-\eta b_1(x)\right| \leq K(x) \eta^2 .
$$
Additionally, we introduce the following shorthand:

\begin{itemize}
\item A multi-index is $\alpha=\left(\alpha_1, \alpha_2, \ldots, \alpha_n\right)$ such that $\alpha_j \in\{0,1,2, \ldots\}$;
\item $|\alpha|:=\alpha_1+\alpha_2+\cdots+\alpha_n$;
\item $\alpha !:=\alpha_{1} ! \alpha_{2} ! \cdots \alpha_{n} !$;
\item For $x=\left(x_1, x_2, \ldots, x_n\right) \in \mathbb{R}^n$, we define $x^\alpha:=x_1^{\alpha_1} x_2^{\alpha_2} \cdots x_n^{\alpha_n}$;
\item For a multi-index $\beta$, $\partial_{\beta}^{|\beta|}f(x) := \frac{\partial^{|\beta|}}{\partial^{\beta_1}_{x_1}\partial^{\beta_2}_{x_2} \cdots \partial^{\beta_n}_{x_n} }f(x)$;
\item We also denote the partial derivative with respect to $ x_{i} $ by $ \partial_{e_i} $. \\
\end{itemize}

\begin{definition}[G Set]
    Let $G$ denote the set of continuous functions $\mathbb{R}^{d} \rightarrow \mathbb{R}$ of at most polynomial growth, i.e.~$g \in G$ if there exists positive integers $\nu_1, \nu_2>0$ such that $|g(x)| \leq \nu_1\left(1+|x|^{2 \nu_2}\right)$, for all $x \in \mathbb{R}^{d}$.
\end{definition}

\begin{definition}[$\mathcal{C}_b^k\left(\mathbb{R}^n, \mathbb{R}\right)$]
    $\mathcal{C}_b^k\left(\mathbb{R}^n, \mathbb{R}\right)$ denotes the space of functions whose $k$-th derivatives are bounded.
\end{definition}

\subsection{Assumptions.}\label{sec:AllAss}
In general, we assume some regularity in the loss function.
\begin{mybox}{gray}
\begin{assumption}
Assume that the following conditions on $f, f_i \in \mathcal{C}_b^8\left(\mathbb{R}^n, \mathbb{R}\right)$, and their gradients are satisfied:
\begin{itemize}
\item $\nabla f, \nabla f_i $ satisfy a Lipschitz condition: there exists $ L>0 $ such that
$$
|\nabla f(u)-\nabla f(v)|+\sum_{i=1}^n\left|\nabla f_i(u)-\nabla f_i(v)\right| \leq L|u-v|;
$$
\item $ f, f_i $ and its partial derivatives up to order 7 belong to $ G $;
\item $ \nabla f, \nabla f_i $ satisfy a growth condition: there exists $ M>0 $ such that
$$
|\nabla f(x)|+\sum_{i=1}^n\left|\nabla f_i(x)\right| \leq M(1+|x|).
$$
\end{itemize}
\label{ass:regularity_f}
\end{assumption}
\end{mybox}
Regarding the gradient noise, each optimizer has its mild assumptions which are weaker or in line with the literature.
\paragraph{DSGD}
\begin{enumerate}
    \item The covariance matrices $\Sigma_i(x)$ are Definite Positive;
    \item Their square roots $\sqrt{\Sigma_i}(x)$ are: In $G$ together with their derivatives, Lipschitz, bounded, and satisfy Affine Growth \cite{Malladi2022AdamSDE}.
\end{enumerate}

\paragraph{DCSGD} Additionally w.r.t. DSGD, DCSGD requires:
\begin{enumerate}
    \item The gradient noise $Z(x)$ admits a strictly positive density function $g_x$ for all $x$ and require that $g:\mathbb{R}^{n}\times \mathbb{R}^{n}\rightarrow \left[ 0,\infty \right)$ s.t. $ (x,y)\mapsto g_{x}(y)$ is in $C^{8}(\mathbb{R}^{n}\times \mathbb{R}^{n})$ such that all partial derivatives of $g$ up to order $8$ are integrable with respect to $y$ and s.t. their $L^{1}$-norms are uniformly bounded in $x$. This assumption covers Gaussian and Student's t, thus being \textit{more general than the literature}. Indeed, the Gaussianity of the noise is commonly assumed: Among others, see \cite{ahn2012bayesian,chen2014stochastic,mandt2016variational,stephan2017stochastic,zhu2019anisotropic,wu2020noisy,Xie2021}, while \cite{jastrzkebski2017three} offers an intuitive justification as well;
    \item Bounded and closed domain \cite{shamir2014distributed,wang2017memory,zhao2018proximal,yu2019double,aviv2021learning,ayache2023walk,marfoq2023federated,deng2024distributed}: This assumption is not restrictive in our case. Indeed, our contribution regarding DCSGD is not to prove their convergence, which has been proven before \citep{khirirat2018distributed, li2020unified}, but rather the scaling rules in Prop. \ref{prop:DCSGD_RecoverLaws_Main}. Since convergence has already been guaranteed, we can assume the domain to be closed and bounded without loss of generality while still providing insightful and actionable results. Additionally, this is also assumed in the seminal paper for this theoretical framework \cite{li2019stochastic};
    \item For all compact sets $K$ \[
\sup_{x \in K}\left\vert g(x,\cdot) \right\vert \in L^1(\mathbb{R}^n),
\]
which of course covers the Gaussian case, \textit{thus being more general than the literature.}
\end{enumerate}

\paragraph{DSignSGD}
On top of the assumptions 1. and 3. of DCSGD, we need the functions in Eq. \ref{A3} to be in $G$, which, as we show below, covers Gaussian and Student's t, \textit{thus being more general than the literature}.

\paragraph{Remark}
All the assumptions above are \textit{in line with or more general than those commonly found in the literature}. In line with \textit{Remark 11} of the seminal paper \cite{li2019stochastic}, we observe that while some of these assumptions might seem strong, loss functions in applications have inward pointing gradients for sufficiently large $x$. Therefore, we could simply modify the loss to satisfy the assumptions above.

Regarding the drift and diffusion coefficients, we highlight that many papers in the literature following this framework do not check for their regularity before applying the approximation theorems \cite{hu2019global,an2020stochastic,zhu2020sharp,cui2020momentum,maulen2021continuous,wang2022two,compagnoni2023sde,compagnoni2024sde,li2017stochastic}. At first sight, it would seem that not even the seminal paper \cite{li2019stochastic} checks these conditions carefully. However, a deeper investigation shows that they are restricting their analysis to compact sets to leverage the regularity and convergence properties of mollifiers: The assumption regarding the compactness of the domain is not highlighted nor assumed in any part of the paper. Therefore, we conclude that, willingly or not, most papers are implicitly making these assumptions.

\subsection{Technical Results}

In this subsection, we provide some results that will be instrumental in the derivation of the SDEs.

\begin{lemma}
\label{NoiseCondition}Assume the existence of a probability density $g_{x}$ of the gradient noise $Z (x)$ for all $x$ and require that $g:\mathbb{R}^{n}\times \mathbb{R}^{n}\rightarrow \left[ 0,\infty \right) ;$ $ (x,y)\mapsto g_{x}(y)$ is in $C^{8}(\mathbb{R}^{n}\times \mathbb{R}^{n})$ such that all partial derivatives of $g$ up to order $8$ are integrable with respect to $y$ and such that their $L^{1}-$norms are uniformly bounded in $x$. Further, let $f\in C^{8}(\mathbb{R}^{n})$ and $h:\mathbb{R}^{n}\rightarrow $ $\mathbb{R}$ be a bounded Borel measurable function. Define the function $%
k $ by%
\[
k(x)=\mathbb{E}\left[ h(\nabla f_{\gamma }(x))\right] \text{.} 
\]%
Then there exists a version $\widehat{k}$ of $k$ with $\widehat{k}\in
C_{b}^{7}(\mathbb{R}^{n})$.
\end{lemma}

\begin{proof}
Let $\varphi $ be smooth and compactly supported. Then for all multiindices $ \beta $ with $\left\vert \beta \right\vert \leq 8$, substitution, Fubini`s theorem, and integration by parts imply that
\begin{eqnarray*}
\int_{\mathbb{R}^{n}}k(x)\partial _{\beta }^{\left\vert \beta \right\vert}\varphi (x)dx &=&\int_{\mathbb{R}^{n}}\mathbb{E}\left[ h(\nabla f_{\gamma}(x))\right] \partial _{\beta }^{\left\vert \beta \right\vert }\varphi (x)dx\\
&=&\int_{\mathbb{R}^{n}}\int_{\mathbb{R}^{n}}h(y)g_{x}(y-\nabla f(x))dy\partial _{\beta }^{\left\vert \beta \right\vert }\varphi (x)dx \\
&=&(-1)^{\left\vert \beta \right\vert }\int_{\mathbb{R}^{n}}\int_{\mathbb{R}^{n}}h(y)\partial _{\beta }^{\left\vert \beta \right\vert }(g_{x}(y-\nabla f(x)))dy\varphi (x)dx.
\end{eqnarray*}%
So 
\[
\int_{\mathbb{R}^{n}}h(y)\partial _{\beta }^{\left\vert \beta \right\vert
}(g_{x}(y-\nabla f(x)))dy 
\]%
is a weak derivative $\partial _{\beta }^{\left\vert \beta \right\vert }k$
of $k$ on any bounded open set. For compact sets $K$ we obtain that%
\begin{eqnarray*}
&&\int_{K}\left\vert \int_{\mathbb{R}^{n}}h(y)\partial _{\beta }^{\left\vert
\beta \right\vert }(g_{x}(y-\nabla f(x)))dy\right\vert ^{p}dx \\
&\leq &\left\Vert h\right\Vert _{\infty }^{p}\lambda ^{n}(K)\left(
\sup_{x\in \mathbb{R}^{n}}\int_{\mathbb{R}^{n}}\left\vert \partial _{\beta
}^{\left\vert \beta \right\vert }(g_{x}(y-\nabla f(x)))\right\vert dy\right)
^{p}<\infty
\end{eqnarray*}%
for all $p\geq 2$ because of our assumptions on $g$ and $f$ and substitution
($\lambda ^{n}$ Lebesgue measure). So it follows from Sobolev embeddings
with respect to H\"{o}lder spaces that for all bounded and open sets $\Omega 
$ there exists a version $\widehat{k}$ of $k$ such that $\widehat{k}\in
C^{7}(\Omega )$. The latter version can be extended to $\Omega =\mathbb{R}%
^{n}$, which we also denote by $\widehat{k}$. Since $\partial _{\beta
}^{\left\vert \beta \right\vert }k$ is bounded for $\left\vert \beta
\right\vert \leq 8$, we conclude that $\widehat{k}\in C_{b}^{7}(\mathbb{R}%
^{n})$.
\end{proof}

\begin{lemma}\label{lemma:Derivatives}
    Assuming that for all compact sets $K$ \[
\sup_{x \in K}\left\vert g(x,\cdot) \right\vert \in L^1(\mathbb{R}^n),
\] and the positivity of the density functions, we have that
    for $m=1, \dots, 7$ that
    \begin{equation}
    \left\Vert \partial_{j_1} \dots \partial_{j_m} A^{1/2}(x) \right\Vert \leq C l_m(x), \label{M}
    \end{equation}
where the function $l_m(x)$ is defined as
    \begin{eqnarray}
    l_m(x) &:=& \sum_{r=0}^{m-1} \left( \frac{1}{m(x) + s(x)(n-1)^{1/2}} \left( 1 + \frac{2 s(x)(n-1)^{1/2}}{m(x) - s(x)(n-1)^{-1/2}} \right) \right)^{-(r+1/2)} \nonumber \\
    && \times \max_{\vert \beta \vert \leq m} \left\Vert \partial_{\beta}^{\vert \beta \vert} A(x) \right\Vert^{r+1}. \label{A3}
\end{eqnarray}
\end{lemma}
\begin{proof}
To prove this, we need the fact that the Fréchet derivatives of the square root function $\varphi$ can be represented as follows (see Theorem 1.1 in \cite{MN}):
\[
\nabla \varphi(A)[H] = \int_0^{\infty} e^{-t\varphi(A)} H e^{-t\varphi(A)} dt,
\]
and higher derivatives of order $m \geq 2$ are given by
\begin{eqnarray}
\nabla^m \varphi(A)[H, \dots, H] &=& - \nabla \varphi(A) \left[ \sum_{p+q=m-2} \frac{m!}{(p+1)!(q+1)!} (\nabla^{p+1} \varphi(A)[H, \dots, H]) \right. \nonumber \\
&& \left. \times (\nabla^{q+1} \varphi(A)[H, \dots, H]) \right] \label{TaylorFormula}
\end{eqnarray}
for all $A \in \mathbb{S}$ and symmetric $n \times n$ matrices $H$. Moreover, we have the following estimate for $m \geq 0$:
\begin{equation}
\left\Vert \nabla^{m+1} \varphi(A) \right\Vert \leq (\sqrt{n})^m (m+1)! C_m 2^{-2(m+1)} \lambda_{\min}(A)^{-(m+1/2)}, \label{EstimateDerivative}
\end{equation}
where $\lambda_{\min}(A) > 0$ is the smallest eigenvalue of $A$ and $C_m := \frac{1}{m+1} \binom{2m}{m}$.

We find that $\partial_l A^{1/2}(x) = \nabla \varphi(A(x))[\partial_l A(x)]$ and 
\[
\partial_j \partial_l A^{1/2}(x) = \nabla^2 \varphi(A(x))[\partial_j A(x), \partial_l A(x)] + \nabla \varphi(A(x))[\partial_j \partial_l A(x)].
\]
Thus, it follows from Eq. (\ref{EstimateDerivative}) that
\[
\left\Vert \partial_l A^{1/2}(x) \right\Vert \leq C \lambda_{\min}(A(x))^{-1/2} \left\Vert \partial_l A(x) \right\Vert,
\]
and
\begin{eqnarray*}
\left\Vert \partial_j \partial_l A^{1/2}(x) \right\Vert &\leq& C_1 \lambda_{\min}(A(x))^{-(1+1/2)} \left\Vert \partial_j A(x) \right\Vert \left\Vert \partial_l A(x) \right\Vert \\
&& + C_2 \lambda_{\min}(A(x))^{-1/2} \left\Vert \partial_j \partial_l A(x) \right\Vert.
\end{eqnarray*}

More generally, for $m=1, \dots, 7$,
\begin{eqnarray}
\left\Vert \partial_{j_1} \dots \partial_{j_m} A^{1/2}(x) \right\Vert &\leq& C_m \left\{ \sum_{r=0}^{m-1} \lambda_{\min}(A(x))^{-(r+1/2)} \right. \nonumber \\
&& \left. \times \max_{\vert \beta \vert \leq m} \left\Vert \partial_{\beta}^{\vert \beta \vert} A(x) \right\Vert^{r+1} \right\}. \label{EstimatePartial}
\end{eqnarray}

Let us now provide a lower bound for $\lambda_{\min}(A(x))$ in terms of $tr(A(x))$ and $tr((A(x))^2)$. Define
\[
s^2(x) = n^{-1} \left( tr((A(x))^2) - \frac{(tr(A(x)))^2}{n} \right), \quad m(x) = \frac{tr(A(x))}{n}.
\]
Then, from Corollary 2.1, Corollary 2.2, and Theorem 2.1 in \cite{WS}, we obtain
\begin{eqnarray*}
\frac{1}{\lambda_{\min}(A(x))} &\leq& \frac{1}{\lambda_{\max}(A(x))} \left( 1 + \frac{2 s(x) (n-1)^{1/2}}{m(x) - s(x) (n-1)^{-1/2}} \right) \\
&\leq& \frac{1}{m(x) + s(x) (n-1)^{1/2}} \left( 1 + \frac{2 s(x) (n-1)^{1/2}}{m(x) - s(x) (n-1)^{-1/2}} \right).
\end{eqnarray*}

Therefore, from Eq. (\ref{EstimatePartial}), we have for $m=1, \dots, 7$ that
\begin{equation}
\left\Vert \partial_{j_1} \dots \partial_{j_m} A^{1/2}(x) \right\Vert \leq C l_m(x),
\end{equation}
where the function $l_m(x)$ is defined as
\begin{eqnarray}
l_m(x) &:=& \sum_{r=0}^{m-1} \left( \frac{1}{m(x) + s(x)(n-1)^{1/2}} \left( 1 + \frac{2 s(x)(n-1)^{1/2}}{m(x) - s(x)(n-1)^{-1/2}} \right) \right)^{-(r+1/2)} \nonumber \\
&& \times \max_{\vert \beta \vert \leq m} \left\Vert \partial_{\beta}^{\vert \beta \vert} A(x) \right\Vert^{r+1}.
\end{eqnarray}
\end{proof}

The following results are key to guarantee that an SDE is a weak approximation of an optimizer.

 \begin{mybox}{gray}
\begin{proposition}[Proposition 1 \cite{li2017stochastic}] \label{prop:li1}
Let $ 0<\eta<1 $. Consider a stochastic process $ X_t, t \geq 0 $ satisfying the SDE
$$
d X_t=b\left(X_t\right)dt+\sqrt{\eta} \sigma\left(X_t\right) d W_t
$$
with $ X_0=x \in \mathbb{R}^{d}$ and $ b, \sigma $ together with their derivatives belong to $ G $. Define the one-step difference $ \Delta=X_\eta-x $, and indicate the $i$-th component of $\Delta$ with $\Delta_i$. Then we have

\begin{enumerate}
\item $ \E \Delta_{i}=b_{i} \eta+\frac{1}{2}\left[\sum_{j=1}^d b_{j} \partial_{e_j} b_{i}\right] \eta^2+\mathcal{O}\left(\eta^3\right) \quad \forall i = 1, \ldots, d$;
\item $ \E \Delta_{i} \Delta_{j}=\left[b_{i} b_{j}+\sigma \sigma_{(i j)}^T\right] \eta^2+\mathcal{O}\left(\eta^3\right) \quad \forall i,j = 1, \ldots, d$;
\item $ \E \prod_{j=1}^s \Delta_{\left(i_j\right)}=\mathcal{O}\left(\eta^3\right) $ for all $ s \geq 3, i_j=1, \ldots, d $.
\end{enumerate}
All functions above are evaluated at $ x $.
\end{proposition}
\end{mybox}

\begin{mybox}{gray}
\begin{theorem}[Theorem 2 and Proposition 5, \cite{mil1986weak}]\label{thm:mils}
Let Assumption \ref{ass:regularity_f} hold and let us define $\bar{\Delta}=x_1-x$ to be the increment in the discrete-time algorithm, and indicate the $i$-th component of $\bar{\Delta}$ with $\bar{\Delta}_i$. If in addition there exists $K_1, K_2, K_3, K_4 \in G$ so that

\begin{enumerate}
\item $\left|\E \Delta_{i}-\E \bar{\Delta}_{i}\right| \leq K_1(x) \eta^{2}, \quad \forall i = 1, \ldots, d$;

\item $\left|\E \Delta_{i} \Delta_{j} - \E \bar{\Delta}_{i} \bar{\Delta}_{j}\right| \leq K_2(x) \eta^{2}, \quad \forall i,j = 1, \ldots, d$;
\item $\left|\E \prod_{j=1}^s \Delta_{i_j}-\E \prod_{j=1}^s \bar{\Delta}_{i_j}\right| \leq K_3(x) \eta^{2}, \quad \forall s \geq 3, \quad \forall i_j \in \{1, \ldots, d \}$;
\item $\E \prod_{j=1}^{ 3}\left|\bar{\Delta}_{i_j}\right| \leq K_4(x) \eta^{2}, \quad \forall i_j \in \{1, \ldots, d \}$.
\end{enumerate}
Then, there exists a constant $C$ so that for all $k=0,1, \ldots, N$ we have
$$
\left|\E g\left(X_{k \eta}\right)-\E g\left(x_k\right)\right| \leq C \eta.
$$
\end{theorem}
\end{mybox}
\subsection{Limitations} \label{sec:limitations} Modeling of discrete-time algorithms using SDEs relies on Assumption \ref{ass:regularity_f}. As noted by \cite{Li2021validity}, the approximation can fail when the stepsize $\eta$ is large or if certain conditions on $\nabla f$ and the noise covariance matrix are not met. Although these issues can be addressed by increasing the order of the weak approximation, we believe that the primary purpose of SDEs is to serve as simplification tools that enhance our intuition: We would not benefit significantly from added complexity. Regarding the assumptions on the noise, ours are in line with or more general than those commonly used in the literature.

Another aspect concerns the discretization of SDEs. While our approach has been to experimentally verify that the SDE tracks the evolution of the corresponding discrete algorithms and supports our theoretical insights, alternative theoretical frameworks exist. Notably, backward error analysis offers a promising direction, as it can clarify the role of finite learning rates and help identify different optimizers' implicit biases. This approach has been successfully used to derive higher-order modified equations for SGD \cite{smith2021origin} and Adam \cite{OTIBA2024}. While our work does not include such an analysis, many influential papers \cite{koloskova2020unified,mil1986weak,zhou2020stochastic} similarly omit it. Given that most papers modeling optimizers with SDEs either lack experimental validation or restrict it to artificial landscapes, we take an extra step by validating our insights across various deep neural networks and datasets. To our knowledge, only \cite{paquette2021sgd,compagnoni2023sde} have conducted experiments involving neural networks, and even then, with relatively small models. In contrast, our extensive experiments demonstrate that our insights apply to realistic scenarios, as confirmed by our numerical results.

Finally, while SDEs benefit from Itô Calculus which allows us to study general non-convex loss functions, we had to focus on simple noise structures. Differently, Stochastic Approximation enables a more fine-grained and insightful analysis for very general noise structures (e.g. multiplicative noise), but often forces the analysis to focus on quadratic losses \cite{PD2025}.

\subsection{Distributed SGD}

This subsection provides the first formal derivation of an SDE model for DSGD. Let us consider the stochastic process $ X_t \in \mathbb{R}^{d} $ defined as the solution of

\begin{equation}\label{eq:DUSGD_SDE}
    d X_t = - \nabla f(X_t) dt + \sqrt{\frac{\eta}{N}} \sqrt{\hat{\Sigma}(X_t)} dW_t,
\end{equation}
where $\hat{\Sigma}(x)\eqdef  \frac{1}{N} \sum_{i=1}^{N} \Sigma_i(x)$ is the average of the covariance matrices of the $N$ agents.
\begin{mybox}{gray}
\begin{theorem}[Stochastic modified equations] \label{thm:DSGD_SDE}
Let $0<\eta<1, T>0$ and set $N=\lfloor T / \eta\rfloor$. Let $ x_k \in \mathbb{R}^{d}, 0 \leq k \leq N$ denote a sequence of DSGD iterations defined by Eq.~\eqref{eq:DUSGD_Discr_Update}. Consider the stochastic process $X_t$ defined in Eq.~\eqref{eq:DUSGD_SDE} and fix some test function $g \in G$ and suppose that $g$ and its partial derivatives up to order 6 belong to $G$.

Then, under the assumptions of Section \ref{sec:AllAss}, there exists a constant $ C>0 $ independent of $ \eta $ such that for all $ k=0,1, \ldots, N $, we have

$$
\left|\E g\left(X_{k \eta}\right)-\E g\left(x_k\right)\right| \leq C \eta .
$$

That is, the SDE \eqref{eq:DUSGD_SDE} is an order $ 1 $ weak approximation of the DSGD iterations \eqref{eq:DUSGD_Discr_Update}.
\end{theorem}
\end{mybox}

\begin{proof}
First, we calculate the expected value of the increments of DSGD:
\begin{equation}
    \E \left[ x_{k+1} - x_{k}  \right] = \E\left[ - \frac{\eta}{N} \sum_{i=1}^{N} \nabla f_{\gamma_i} (x_k)  \right] = -\eta \nabla f(x_k);
\end{equation}
Then, we calculate the covariance matrix of the gradient noise of DSGD:
\begin{align}
    \tilde{\Sigma}(x_k) & = \eta^2 \E\left[ \left( \nabla f(x_k) - \frac{1}{N}\sum_{i=1}^{N} \nabla f_{\gamma_i} (x_k)  \right) \left( \nabla f(x_k) - \frac{1}{N}\sum_{j=1}^{N} \nabla f_{\gamma_j} (x_k)  \right)^{\top} \right] \\
    & = \frac{\eta^2}{N} \frac{1}{N} \sum_{i,j=1}^{N} \E \left[ \left( \nabla f(x_k) - \nabla f_{\gamma_i} (x_k)  \right) \left( \nabla f(x_k) - \nabla f_{\gamma_j} (x_k)  \right)^{\top}\right] \\
    & = \frac{\eta^2}{N} \frac{1}{N} \sum_{i=1}^{N} \E \left[ \left( \nabla f(x_k) - \nabla f_{\gamma_i} (x_k)  \right) \left( \nabla f(x_k) - \nabla f_{\gamma_i} (x_k)  \right)^{\top}\right] \\
    & =  \frac{\eta^2}{N} \frac{1}{N} \sum_{i=1}^{N} \Sigma_i(x_k),
\end{align}
where we use independence of $(\nabla f(x_k) - \nabla f_{\gamma_i}(x_k)$ for $i\in[N]$. The thesis follows from Proposition \ref{prop:li1} and Theorem \ref{thm:mils} as drift and diffusion terms are regular by assumption.
\end{proof}

\begin{theorem}\label{thm:DSGD_Convergence_App}
    If $f$ is $\mu$-PL, $L$-smooth, and $ \tr(\Sigma_i(x)) < \mathcal{L}_{\sigma_i}$
    \begin{equation}
        \E \left[ f(X_t) - f(X_*) \right] \leq  \left( f(X_0) - f(X_*) \right) e^{- 2\mu t}  + \frac{\eta L \overline{\mathcal{L}}_{\sigma}} {4 \mu N} (1 - e^{- 2\mu t}),
    \end{equation}
where $\overline{\mathcal{L}}_{\sigma} \eqdef \frac{1}{N} \sum_{i=1}^{N} \mathcal{L}_{\sigma_i}$.
\end{theorem}
\begin{proof}
Using Itô's Lemma
\begin{align}
    d (f(X_t) - f(X_*)) & = - \nabla f(X_t)^{\top} \nabla f(X_t) dt + \mathcal{O}(\text{Noise}) + \frac{\eta}{2 N} \tr(\nabla^2 f(X_t) \tilde{\Sigma}(X_t)) dt \\
    & \leq -2 \mu (f(X_t) - f(X_*)) dt + \frac{\eta L \overline{\mathcal{L}_\sigma}}{2 N} dt + \mathcal{O}(\text{Noise}),
\end{align}
which implies the thesis.
\end{proof}

\begin{corollary}
    Let the batch size be $ \delta B$, learning rate $\kappa \eta$, and $\alpha N$ agents. The scaling rule to preserve the performance independently of $\delta$, $\kappa$, and $\alpha$ is $\frac{\kappa}{\alpha \delta} = 1$.
\end{corollary}
\begin{proof}
    It follows the same steps as Theorem \ref{thm:DSGD_SDE} to derive the SDE and Theorem \ref{thm:DSGD_Convergence_App} to derive the bound. Then, one needs to find the functional relationship between $\kappa$, $\alpha$, and $\delta$ such that the bound does not depend on them.
\end{proof}

\begin{theorem}
    If $f$ is $L$-smooth, we use a learning rate scheduler $\eta_t$ such that $\phi^i_t = \int_0^t (\eta_s)^i ds$, $\phi^1_t \overset{t \rightarrow \infty}{\rightarrow} \infty$, $\frac{\phi^2_t}{\phi^1_t} \overset{t \rightarrow \infty}{\rightarrow} 0$, and $\overline{\mathcal{L}}_{\sigma}\eqdef \frac{1}{N} \sum_{i=1}^{N} \mathcal{L}_{\sigma_i}$
    \begin{equation}
        \E \left[ \lVert \nabla f(X_{\tilde{t}}) \rVert_2^2\right] \leq \frac{f(X_0) - f(X_*)}{\phi^1_t}  + \frac{\eta L \overline{\mathcal{L}}_{\sigma} }{2   N} \frac{\phi^2_t}{\phi^1_t} \overset{t \rightarrow \infty}{\rightarrow} 0,
    \end{equation}
where $\tilde{t}$ has distribution $\frac{\eta_{\tilde{t}}}{\phi^1_t}$.
\end{theorem}
\begin{proof}
Using Itô's Lemma and using a learning rate scheduler $\eta_t$ during the derivation of the SDE of Theorem \ref{thm:DSGD_SDE}, we have
\begin{align}
    d (f(X_t) - f(X_*)) & = - \eta_t \lVert \nabla f(X_t) \rVert_2^2 dt + \mathcal{O}(\text{Noise}) + (\eta_t)^2\frac{\eta}{2 N} \tr(\nabla^2 f(X_t) \tilde{\Sigma}(X_t)) dt \\
    & \leq - \eta_t \lVert \nabla f(X_t) \rVert_2^2 dt + \mathcal{O}(\text{Noise}) + (\eta_t)^2\frac{\eta L \overline{\mathcal{L}}_{\sigma}}{2 N} dt.
\end{align}
Let us now observe that since $\int_{0}^{t} \frac{\eta_s}{\phi^1_t} ds = 1$, the function $s \mapsto \frac{\eta_s}{\phi^1_t}$ defines a probability distribution and let $\tilde{t}$ have that distribution. Then by integrating over time and by the Law of the Unconscious Statistician, we have that
\begin{equation}
    \E \left[ \lVert \nabla f(X_{\tilde{t}}) \rVert_2^2\right] = \frac{1}{\phi^1_t} \int_0^{t}  \lVert \nabla f(X_s)\rVert_2^2 \eta_s ds,
\end{equation}
meaning that
\begin{equation}
    \E \left[ \lVert \nabla f(X_{\tilde{t}}) \rVert_2^2\right] \leq \frac{f(X_0) - f(X_*)}{\phi^1_t}  + \frac{\eta L \overline{\mathcal{L}}_{\sigma} }{2   N} \frac{\phi^2_t}{\phi^1_t} \overset{t \rightarrow \infty}{\rightarrow} 0.
\end{equation}
\end{proof}

\section{Distributed Compressed SGD with Unbiased Compression}

This subsection provides the first formal derivation of an SDE model for DCSGD. Let us consider the stochastic process $ X_t \in \mathbb{R}^{d} $ defined as the solution of

\begin{equation}\label{eq:DCSGD_SDE}
    d X_t = - \nabla f(X_t) dt + \sqrt{\frac{\eta}{N}} \sqrt{\Tilde{\Sigma}(X_t)} dW_t,
\end{equation}
where for $\Phi_{\xi_i,\gamma_i}(x) := \mathcal{C}_{\xi_i} \left( \nabla f_{\gamma_i} (x) \right) - \nabla f_{\gamma_i}(x)$
\small 
\begin{equation}
    \Tilde{\Sigma}(x) = \frac{1}{ N} \sum_{i=1}^{ N} \left( \E_{\xi_i \gamma_i} \left[ \Phi_{\xi_i,\gamma_i}(x)\Phi_{\xi_i,\gamma_i}(x)^{\top} \right] + \Sigma_i(x) \right).
\end{equation}

Before proceeding, we ensure that the SDE admits a unique solution and that its coefficients are sufficiently regular.

\begin{lemma}
    The drift term $\nabla f$ is Lipschitz, satisfies Affine Growth, and is in $G$ together with all its derivatives.
\end{lemma}
\begin{proof}
    This is obvious as we assume all of these conditions.
\end{proof}
Regarding the diffusion term, we have that

\begin{lemma}
    The diffusion term $\Tilde{\Sigma}(x)$ satisfies Affine Growth.
\end{lemma}
\begin{proof}
    Since $\lVert \sqrt{\tilde{\Sigma}_i}(x) \rVert_2 \leq \tr(\tilde{\Sigma}_i(x))^{\frac{1}{2}} \leq (\omega \lVert \nabla f(x) \rVert_2^2 + \lVert \Sigma_i (x)\rVert_{\infty}(\omega + 1))^{\frac{1}{2}} $, the linear growth of the gradient, the boundedness of $\Sigma_i$, and that $\lVert A \rVert_{\infty} \leq \sqrt{d}\lVert A \rVert_2$ for each matrix $A$.
\end{proof}

\begin{lemma}
\label{CorollaryNoise_DCSGD}
Let us assume the same assumptions as Lemma \ref{NoiseCondition} and that the domain is closed and sufficiently large.\footnote{This is a common assumption in the literature \cite{shamir2014distributed,wang2017memory,zhao2018proximal,yu2019double,aviv2021learning,ayache2023walk,marfoq2023federated,deng2024distributed}.} Additionally, assume that
\[
\sup_{x \in K}\left\vert g(x,\cdot) \right\vert \in L^1(\mathbb{R}^n)
\]
for all compact sets $K$. Then the entries of $\tilde{\Sigma}$ in Eq. \ref{eq:DCSGD_SDE} are in $C_b^7(\mathbb{R}^n)$.
\end{lemma}
\begin{proof}
    Since we are on a closed and sufficiently large domain, by the definition of $\tilde{\Sigma}$, dominated convergence, and from the additional assumption on $g$, it follows that $\tilde{\Sigma}$ is continuous. So Lemma \ref{NoiseCondition} entails that the entries of $\tilde{\Sigma}$ are in $C_b^7(\mathbb{R}^n)$.
\end{proof}

\begin{lemma}
    The diffusion term $\Tilde{\Sigma}(x)$ is Definite Positive.
\end{lemma}
\begin{proof}
    By the definition of $\Tilde{\Sigma}(x)$ and the fact that $\Sigma_i(x)$ are DP by assumption, the thesis follows.
\end{proof}

\begin{corollary}
    Since $\tilde{\Sigma}$ is positive definite and its entries are in $C_b^7(\mathbb{R}^n)$, $\sqrt{\tilde{\Sigma}}$ is Lipschitz.
\end{corollary}

\begin{proof}
    The function
    \[
    \varphi: \mathbb{S} \to \mathbb{S}, \quad A \mapsto \sqrt{A}
    \]
    has Fréchet derivatives of any order on $\mathbb{S}$ (see e.g. \cite{MN}). Therefore, $\tilde{\Sigma}^{1/2} \in C^7(\mathbb{R}^n)$, and since $\tilde{\Sigma} \in C_b^7(\mathbb{R}^n)$, $\tilde{\Sigma}^{1/2}$ is Lipschitz continuous (see Proposition 6.2 in \cite{IW}).
\end{proof}

\begin{lemma}
    Under the same assumptions as Lemma \ref{lemma:Derivatives}, $\tilde{\Sigma}^{1/2} \in G$ together with its derivatives. 
\end{lemma}
\begin{proof}
    The thesis follows from the regularity of the entries and the closeness and boundedness of the domain.
\end{proof}
\begin{remark}\label{remark:DCSGD}
    Based on the above results, we have that under mild assumptions on the noise structures (see Sec. \ref{sec:AllAss}) that cover and generalize the well-accepted Gaussianity, and under the well-accepted closeness and boundedness of the domain, the SDE of DCSGD admits a unique solution and its coefficients are regular enough to apply Prop. \ref{prop:li1} and Thm. \ref{thm:mils}.
\end{remark}

\begin{mybox}{gray}
\begin{theorem}[Stochastic modified equations] \label{thm:DCSGD_SDE}
Let $0<\eta<1, T>0$ and set $N=\lfloor T / \eta\rfloor$. Let $ x_k \in \mathbb{R}^{d}, 0 \leq k \leq N$ denote a sequence of DCSGD iterations defined by Eq.~\eqref{eq:DCSGD_Discr_Update}. Consider the stochastic process $X_t$ defined in Eq.~\eqref{eq:DCSGD_SDE} and fix some test function $g \in G$ and suppose that $g$ and its partial derivatives up to order 6 belong to $G$.

Then, under the assumptions of Section \ref{sec:AllAss}, there exists a constant $ C>0 $ independent of $ \eta $ such that for all $ k=0,1, \ldots, N $, we have

$$
\left|\E g\left(X_{k \eta}\right)-\E g\left(x_k\right)\right| \leq C \eta .
$$

That is, the SDE \eqref{eq:DCSGD_SDE} is an order $ 1 $ weak approximation of the DCSGD iterations \eqref{eq:DCSGD_Discr_Update}.
\end{theorem}
\end{mybox}

\begin{proof}
First, we calculate the expected value of the increments of DCSGD:
\begin{align}
    \E \left[ x_{k+1} - x_{k}  \right] &= \E\left[-\frac{\eta}{N}  \sum_{i=1}^{ N} \mathcal{C}_{\xi_i} \left( \nabla f_{\gamma_i} (x_k) \right)  \right]
    = \E\left[ -\frac{\eta}{N}  \sum_{i=1}^{ N} \nabla f_{\gamma_i} (x_k) \right]\\
    & = -\frac{\eta}{N}  \sum_{i=1}^{ N} \nabla f (x_k) =  - \eta \nabla f(x_k);
\end{align}
Then, we calculate the covariance matrix of the gradient noise of DCSGD:
\begin{align}
    \tilde{\Sigma}(x_k) & = \eta^2 \E_{\xi \gamma}\left[ \left( \nabla f(x_k) - \frac{1}{N} \sum_{i=1}^{ N} \mathcal{C}_{\xi_i} \left( \nabla f_{\gamma_i} (x_k) \right) \right)\left( \nabla f(x_k) - \frac{1}{N} \sum_{j=1}^{ N} \mathcal{C}_{\xi_j} \left( \nabla f_{\gamma_j} (x_k) \right) \right)^{\top} \right] \\
    & = \frac{\eta^2}{N} \frac{1}{N} \sum_{i,j=1}^{N} \E _{\xi_i \xi_j \gamma_i\gamma_j} \left[ \left( \nabla f(x_k) - \mathcal{C}_{\xi_i} \left( \nabla f_{\gamma_i} (x_k) \right)  \right) \left( \nabla f(x_k) - \mathcal{C}_{\xi_j} \left( \nabla f_{\gamma_j} (x_k) \right)  \right)^{\top}\right] \\
    & = \frac{\eta^2}{N} \frac{1}{N} \sum_{i=1}^{N} \E_{\xi_i \gamma_i} \left[ \left( \nabla f(x_k) - \mathcal{C}_{\xi_i} \left( \nabla f_{\gamma_i} (x_k) \right)  \right) \left( \nabla f(x_k) - \mathcal{C}_{\xi_i} \left( \nabla f_{\gamma_i} (x_k) \right)  \right)^{\top}\right] \\
    & =  \frac{\eta^2}{N} \frac{1}{ N} \sum_{i=1}^{ N} \left( \E_{\xi_i \gamma_i} \left[ \Phi_{\xi_i,\gamma_i}(x_k)\Phi_{\xi_i,\gamma_i}(x_k)^{\top} \right] + \Sigma_i(x_k) \right),
\end{align}
where $\Phi_{\xi_i,\gamma_i}(x) := \mathcal{C}_{\xi_i} \left( \nabla f_{\gamma_i} (x) \right) - \nabla f_{\gamma_i}(x)$ and we use independence of $\mathcal{C}_{\xi_i}$ and $\nabla f(x_k) - \nabla f_{\gamma_i}(x_k)$ for all $i\in[N]$. Remembering Remark \ref{remark:DCSGD}, the thesis follows from Prop.~\ref{prop:li1} and Thm.~\ref{thm:mils}.

\end{proof}
\begin{remark}
    The expression for $\tilde{\Sigma}(x)$ is easily derived for different compressors by leveraging Proposition 21 in \cite{PD2025}.
\end{remark}
In all the following results, the reader will notice that all the drifts, diffusion terms, and noise assumptions are selected to guarantee that the SDE we derived for DCSGD is indeed a $1$ weak approximation for DCSGD.

\begin{theorem}
    If $f$ is $\mu$-PL, $L$-smooth, $\overline{\omega} = \frac{1}{N} \sum_{i=1}^{N} \omega_i$,  $ \tr(\Sigma_i(x)) < \mathcal{L}_{\sigma_i}$, $\overline{\mathcal{L}}_{\sigma} := \frac{1}{N} \sum_{i=1}^{N} \mathcal{L}_{\sigma_i}$, and $\overline{\omega \mathcal{L}_\sigma}:=\frac{1}{N} \sum_{i=1}^{N} \omega_i \mathcal{L}_{\sigma_i}$
\begin{equation}
    \E \left[ f(X_t) - f(X_*) \right] \leq (f(X_0) - f(X_*))e^{- \left( 2 \mu -  \frac{\eta L^2 \overline{\omega}}{ N}\right) t} + \left(1 -  e^{-\left( 2 \mu -  \frac{\eta L^2 \overline{\omega}}{ N}\right) t}\right) \frac{\frac{\eta L \left( \overline{\mathcal{L}}_{\sigma} + \overline{\omega \mathcal{L}_\sigma} \right)}{2 N}}{\left( 2 \mu -  \frac{\eta L^2 \overline{\omega}}{ N}\right)}.
\end{equation}
\end{theorem}
\begin{proof}
Using Itô's Lemma
\begin{align}
    d (f(X_t) - f(X_*)) & = - \nabla f(X_t)^{\top} \nabla f(X_t) dt + \mathcal{O}(\text{Noise}) + \frac{\eta}{2 N} \tr(\nabla^2 f(X_t) \tilde{\Sigma}(X_t)) dt \\
    & \leq -2 \mu (f(X_t) - f(X_*)) dt \\
    &+ \frac{\eta L}{2 N} \left(\frac{1}{N} \sum_{i=1}^N  \E_{\xi_i,\gamma_i} \lVert \left(  \mathcal{C}_{\xi_i} \left( \nabla f_{\gamma_i} (x) \right) - \nabla f(x) \right) \rVert_2^2\right) dt
    + \mathcal{O}(\text{Noise}).
\end{align}
Let us focus on a single element of the summation:
\begin{align}
    &\E_{\xi_i,\gamma_i} \lVert \left(  \mathcal{C}_{\xi_i} \left( \nabla f_{\gamma_i} (x) \right) - \nabla f(x) \right) \rVert_2^2
    \quad = \E_{\gamma_i}\left[\E_{\xi_i}\left[\|\mathcal{C}_{\xi_i}(\nabla f_{\gamma_i}(x)) - \nabla f_{\gamma_i}(x)\|^2 
    + \|\nabla f_{\gamma_i}(x) - \nabla f(x)\|^2\right] \mid \gamma_i\right]\\
    & \leq  \omega_i\E_{\gamma_i} \lVert \nabla f_{\gamma_i}(x) \rVert_2^2 
    +  \E_{\gamma_i}\left[\|\nabla f_{\gamma_i}(x) - \nabla f(x)\|^2\right]  = \omega_i \lVert \nabla f(x) \rVert_2^2
    +  (\omega_i+1)\E_{\gamma_i}\left[\|\nabla f_{\gamma_i}(x) - \nabla f(x)\|^2\right]\\
    & \leq  2\omega_i L (f(x) - f(x_*)) 
    + \mathcal{L}_{\sigma_i}(\omega_i+1).
\end{align}

Therefore, we have that
    \begin{align}
        d (f(X_t) - f(X_*)) & \leq - 2 \mu (f(X_t) - f(X_*)) dt + \mathcal{O}(\text{Noise}) +  \frac{\eta L^2 \overline{\omega}}{N}(f(X_t) - f(X_*)) dt\\
        & +  \frac{\eta L \left( \overline{\mathcal{L}}_{\sigma} + \overline{\omega \mathcal{L}_\sigma} \right)}{2 N} dt,
    \end{align}
which implies the thesis.
\end{proof}
\textbf{Remark:} We observe that $\overline{\omega\mathcal{L}_{\sigma}}$ gives a tighter bound than $\overline{\omega}\mathcal{L}_{\sigma,\max}$ or $\omega_{\max}\overline{\mathcal{L}}_{\sigma}$.

\begin{theorem}
    If $f$ is $L$-smooth, we use a learning rate scheduler $\eta_t$ such that $\phi^i_t = \int_0^t (\eta_s)^i ds$, $\phi^1_t \overset{t \rightarrow \infty}{\rightarrow} \infty$, $\frac{\phi^2_t}{\phi^1_t} \overset{t \rightarrow \infty}{\rightarrow} 0$, and $\eta_t > \frac{\eta L \overline{\omega}}{2 N} (\eta_t)^2$, then,
\begin{equation}
   \E \left[\lVert \nabla f(X_{\Tilde{t}}) \rVert_2^2 \right] \leq \frac{1}{ 1  - \frac{ \eta L \overline{\omega}}{2 N} \frac{\phi^2_t}{\phi^1_t} } \left( \frac{f(X_*) - f(X_0)}{\phi^1_t}  + \frac{\phi^2_t}{\phi^1_t} \frac{\eta L}{2 N} \left( \overline{\mathcal{L}}_{\sigma} +  \overline{\omega \mathcal{L}_{\sigma} } \right) \right) \overset{t \rightarrow \infty}{\rightarrow} 0,
\end{equation}
where $\Tilde{t}$, is a random time with distribution $\frac{\eta_{\Tilde{t}} - \frac{\eta L \overline{\omega}}{2 N} (\eta_{\Tilde{t}})^2}{\phi^1_t - \frac{\eta L \overline{\omega}}{2 N} \phi^2_t}$.
\end{theorem}
\begin{proof}
Leveraging what we have shown above, we have that

\begin{align}
    d (f(X_t) - f(X_*)) & = - \eta_t \lVert \nabla f(X_t) \rVert_2^2 dt + \mathcal{O}(\text{Noise})\\
    &+ (\eta_t)^2\frac{\eta L}{2 N} \left(\frac{1}{N} \sum_{i=1}^N  \E_{\xi_i,\gamma_i} \lVert \left(  \mathcal{C}_{\xi_i} \left( \nabla f_{\gamma_i} (x) \right) - \nabla f(x) \right) \rVert_2^2\right) dt.
\end{align}
As before, $\E_{\xi_i,\gamma_i} \lVert \left(  \mathcal{C}_{\xi_i} \left( \nabla f_{\gamma_i} (x) \right) - \nabla f(x) \right) \rVert_2^2 \leq  \omega_i \lVert \nabla f(x) \rVert_2^2
    + \mathcal{L}_{\sigma_i}(\omega_i+1)$.
Therefore, we have that
\begin{equation}
    \E \left[ \lVert \nabla f(X_t) \rVert_2^2\right] \left( \eta_t - \frac{\eta L \overline{\omega}}{2 N} (\eta_t)^2 \right)dt \leq - d (f(X_t) - f(X_*)) + \frac{\eta L (\eta_t)^2}{2 N} \left( \overline{\mathcal{L}}_{\sigma} +  \overline{\omega \mathcal{L}_{\sigma} } \right)dt.
\end{equation}
Let us now observe that since $\int_{0}^{t} \frac{\eta_s - \frac{ \eta L \overline{\omega}}{2 N} \eta_s^2}{\phi^1_t - \frac{ \eta L \overline{\omega}}{2 N} \phi^2_t} ds = 1$, the function $s \mapsto \frac{\eta_s - \frac{ \eta L \overline{\omega}}{2 N} \eta_s^2}{\phi^1_t - \frac{ \eta L \overline{\omega}}{2 N} \phi^2_t}$ defines a probability distribution and let $\tilde{t}$ have that distribution. Then by integrating over time and by the Law of the Unconscious Statistician, we have that
\begin{equation}
    \E \left[ \lVert \nabla f(X_{\tilde{t}}) \rVert_2^2\right] = \frac{1}{\phi^1_t - \frac{ \eta L \overline{\omega}}{2 N} \phi^2_t} \int_0^{t}  \lVert \nabla f(X_s)\rVert_2^2 \left(\eta_s - \frac{ \eta L \overline{\omega}}{2 N} \eta_s^2 \right) ds,
\end{equation}
meaning that
\begin{equation}
   \E \left[\lVert \nabla f(X_{\Tilde{t}}) \rVert_2^2 \right] \leq \frac{1}{ \phi^1_t - \frac{ \eta L \overline{\omega}}{2 N} \phi^2_t } \left( f(X_*) - f(X_0)  + \phi^2_t \frac{\eta L}{2 N} \left( \overline{\mathcal{L}}_{\sigma} +  \overline{\omega \mathcal{L}_{\sigma} } \right) \right) \overset{t \rightarrow \infty}{\rightarrow} 0,
\end{equation}
where $\Tilde{t}$, is a random time with distribution $\frac{\eta_{\Tilde{t}} - \frac{\eta L \overline{\omega}}{2 N} (\eta_{\Tilde{t}})^2}{\phi^1_t - \frac{\eta L \overline{\omega}}{2 N} \phi^2_t}$.
\end{proof}

\subsection{Scaling Rules: Recovering DSGD}

\begin{proposition}\label{prop:DCSGD_RecoverLaws}
    Let the batch size be $ \delta B$, learning rate $\kappa \eta$, the compression rates $\beta \omega_i$, and $\alpha N$ agents. The scaling rules to recover the performance of DSGD are complex and many. For practicality and interpretability purposes, we list here those involving only two hyperparameters at the time:

\begin{enumerate}
    \item If $\kappa=\delta=1$, one needs to ensure that the relation between $\alpha$ and $\beta$ is 
    \begin{equation}
        \alpha = 1 + \beta \left(\frac{\overline{\omega \mathcal{L}_\sigma}}{\overline{\mathcal{L}_\sigma}} + \frac{\overline{\omega}\overline{\mathcal{L}_\sigma}\eta L^2}{2 \mu N}\right).
    \end{equation}
    This gives rise to a trade-off between agents and compression: If there is compression, then one needs to increase the number of agents, and the stronger the compression, the more is needed. In the absence of compression, no additional agents are needed.
    \item If $\beta=\delta=1$, one needs to ensure that the relation between $\alpha$ and $\kappa$ is 
    \begin{equation}
        \frac{\alpha}{\kappa} = 1 + \frac{\overline{\omega \mathcal{L}_\sigma}}{\overline{\mathcal{L}_\sigma}} + \frac{\overline{\omega}\overline{\mathcal{L}_\sigma}\eta L^2}{2 \mu N}.
    \end{equation}
    This gives rise to a trade-off between agents and learning rate: If there is compression, one can increase the learning rate, i.e. $\kappa>1$, and compensate with more agents $\alpha>\kappa>1$. If no compression is in place, the classic trade-off of DSGD $\alpha=\kappa$ is recovered.
    \item If $\beta=\kappa=1$, one needs to ensure that the relation between $\alpha$ and $\gamma$ is 
    \begin{equation}
        \alpha = \frac{1 + \frac{\overline{\omega \mathcal{L}_\sigma}}{\overline{\mathcal{L}_\sigma}}}{\delta}  + \frac{\overline{\omega}\eta L^2}{2 \mu N}.
    \end{equation}
    This gives rise to a trade-off between agents and batch size: If there is compression, one can increase the batch size, i.e. $\delta \geq 1$, and needs fewer agents. If no compression is in place, the classic trade-off of DSGD $\alpha \delta=1$ is recovered.
    \item If $\alpha=\delta=1$, one needs to ensure that the relation between $\beta$ and $\kappa$ is 
    \begin{equation}
        \kappa = \frac{\overline{\mathcal{L}_\sigma}}{\overline{\mathcal{L}_\sigma} + \beta \left( \overline{\omega \mathcal{L}_\sigma} + \frac{\overline{\omega}\eta L^2}{2 \mu N}\right)} .
    \end{equation}
    This gives rise to a trade-off between learning rate and compression: More compression, requires a lower learning rate. No compression implies no change in the learning rate.
    \item If $\alpha=\kappa=1$, one needs to ensure that the relation between $\beta$ and $\delta$ is 
    \begin{equation}
        \delta = \frac{2 \mu \left(\overline{\mathcal{L}_\sigma} + \beta \overline{\omega \mathcal{L}_\sigma} \right)}{\overline{\mathcal{L}_\sigma} \left(2\mu - \beta \frac{\overline{\omega}\eta L^2}{ N} \right)} .
    \end{equation}
    This gives rise to a trade-off between batch size and compression: More compression, requires a larger batch size. No compression implies no change in batch size.
    \item If $\alpha=\beta=1$, one needs to ensure that the relation between $\kappa$ and $\delta$ is 
    \begin{equation}
        \kappa = \frac{\delta \overline{\mathcal{L}_\sigma}}{\overline{\mathcal{L}_\sigma} +  \overline{\omega \mathcal{L}_\sigma} + \delta \frac{\overline{\omega}\eta L^2}{2 \mu N} } .
    \end{equation}
    This gives rise to a trade-off between learning rate and batch size: More batch size requires a larger learning rate. No compression implies the classic $\kappa=\delta$ of DSGD.
\end{enumerate}
We summarize the derived rules in the following table:
\begin{table}[ht]
    \centering
    \begin{tabular}{|c|c|}
        \hline
        \textbf{Scaling Rule} & \textbf{Implication}\\
        \hline
        $\alpha = 1 + \beta \frac{\overline{\omega \mathcal{L}_\sigma}}{\overline{\mathcal{L}_\sigma}} + \beta \frac{\overline{\omega}\overline{\mathcal{L}_\sigma}\eta L^2}{2 \mu N}$ & CR $\uparrow \implies$ Agents $\uparrow$\\
        \hline
        $\frac{\alpha}{\kappa} = 1 + \frac{\overline{\omega \mathcal{L}_\sigma}}{\overline{\mathcal{L}_\sigma}} + \frac{\overline{\omega}\overline{\mathcal{L}_\sigma}\eta L^2}{2 \mu N}$ & LR $\uparrow \implies$ Agents $\uparrow$\\
        \hline
        $\alpha = \frac{1}{\delta} \left(1 + \frac{\overline{\omega \mathcal{L}_\sigma}}{\overline{\mathcal{L}_\sigma}}\right) + \frac{\overline{\omega}\eta L^2}{2 \mu N}$ & BS $\downarrow \implies$ Agents $\uparrow$\\
        \hline
        $\kappa = \frac{\overline{\mathcal{L}_\sigma}}{\overline{\mathcal{L}_\sigma} + \beta \left( \overline{\omega \mathcal{L}_\sigma} + \frac{\overline{\omega}\eta L^2}{2 \mu N}\right)}$ & CR $\uparrow \implies$ LR $\downarrow$\\
        \hline
        $\delta = \frac{2 \mu \left(\overline{\mathcal{L}_\sigma} + \beta \overline{\omega \mathcal{L}_\sigma} \right)}{\overline{\mathcal{L}_\sigma} \left(2\mu - \beta \frac{\overline{\omega}\eta L^2}{ N} \right)}$ & CR $\uparrow \implies$ BS $\uparrow$\\
        \hline
        $\kappa = \frac{\delta \overline{\mathcal{L}_\sigma}}{\overline{\mathcal{L}_\sigma} +  \overline{\omega \mathcal{L}_\sigma} + \delta \frac{\overline{\omega}\eta L^2}{2 \mu N} }$ & BS $\uparrow \implies$ LR $\uparrow$\\
        \hline
    \end{tabular}
    \caption{Summary of Trade-offs Between Parameters (CR = Compression Rate, LR = Learning Rate, and BS = Batch Size).}
    \label{table:DCSGD_Laws}
\end{table}
Of course, in the absence of compression, all scaling rules reduce to the scaling rules of DSGD.
\end{proposition}

\begin{proof}
    Using Itô on $f$, we have that
    \begin{align}
        d (f(X_t) - f(X_*)) & = -\kappa \lVert \nabla f(X_t) \rVert_2^2 dt + \mathcal{O}(\text{Noise}) + \frac{\eta \kappa^2 }{2 \alpha N} \tr(\nabla^2 f(X_t) \Tilde{\Sigma}(X_t)) dt \\
        \leq & - 2 \mu \kappa (f(X_t) - f(X_*)) dt + \mathcal{O}(\text{Noise})\\
        & \;+\;  \frac{\eta \kappa^2 L}{2 \alpha N}   \frac{1}{\alpha N} \sum_{i=1}^{\alpha N} \E_{\xi_i,\gamma_i} \lVert \left(  \mathcal{C}_{\xi_i} \left( \nabla f_{\gamma_i} (x) \right) - \nabla f(x) \right) \rVert_2^2 dt.
    \end{align}

As above, $\E_{\xi_i,\gamma_i} \lVert \left(  \mathcal{C}_{\xi_i} \left( \nabla f_{\gamma_i} (x) \right) - \nabla f(x) \right) \rVert_2^2 \leq  2\beta\omega_i L (f(x) - f(x_*))  + \frac{\mathcal{L}_{\sigma_i}}{\delta B}(\beta \omega_i+1) $.
Therefore, we have that 
    \begin{align}
        d (f(X_t) - f(X_*)) & \\
        \leq & - 2 \mu \kappa (f(X_t) - f(X_*)) dt + \mathcal{O}(\text{Noise}) 
        +  \frac{ \eta \kappa^2 L^2 \overline{\omega} \beta}{ \alpha N}(f(X_t) - f(X_*)) dt\\
        + &  \frac{\eta \kappa^2 L \overline{\mathcal{L}}_{\sigma}}{2 B \delta \alpha N} dt + \frac{\beta \eta \kappa^2 L \overline{\omega \mathcal{L}_\sigma}}{2 B \delta \alpha N} dt,
    \end{align}
which implies that
\begin{align}
    \E \left[ f(X_t) - f(X_*) \right] &\leq (f(X_0) - f(X_*))e^{- \left( 2 \mu -  \frac{ \eta \kappa L^2 \overline{\omega} \beta}{ \alpha N}\right) \kappa t} \\
    & \qquad + \left(1 -  e^{-\kappa \left( 2 \mu -  \frac{\eta \kappa L^2 \overline{\omega}}{ \alpha N}\right) t}\right) \frac{\frac{\eta \kappa L \overline{\mathcal{L}}_{\sigma}}{2 B \delta \alpha N} + \frac{\beta \eta \kappa L \overline{\omega \mathcal{L}_\sigma}}{2 B \delta \alpha N}}{\left( 2 \mu -  \frac{\eta \kappa L^2 \overline{\omega}\beta}{ \alpha N}\right)}.
\end{align}

Now, we need to find functional relationships between $\alpha$, $\delta$, $\kappa$, and $\beta$ such that the asymptotic value of the loss of DCSGD with hyperparameters $(\kappa \eta, \delta B, \beta \omega_i, \alpha N)$ matches the asymptotic loss value of DSGD with hyperparameters $(\eta, B, N)$:
\begin{equation}
     \frac{\frac{\eta \kappa L \overline{\mathcal{L}}_{\sigma}}{2 B \delta \alpha N} + \frac{\beta \eta \kappa L \overline{\omega \mathcal{L}_\sigma}}{2 B \delta \alpha N}}{\left( 2 \mu -  \frac{\eta \kappa L^2 \overline{\omega}\beta}{ \alpha N}\right)} = \frac{\eta L \overline{\mathcal{L}}_{\sigma}} {4 \mu N B}.
\end{equation}
Since a general formula involving all four quantities is difficult to interpret, we derive six rules: For each of them, we keep two scaling parameters constant to $1$ and study the relationship between the remaining two.

Let us prove one to show the mechanism as they are all derived in a few passages. We focus on the first one, for which we set $\kappa=\delta=1$ and study the relationship between $\alpha$ and $\beta$. To do this, we solve 

\begin{align}
     & \frac{\frac{\eta L \overline{\mathcal{L}}_{\sigma}}{2 B \alpha N} + \frac{\beta \eta L \overline{\omega \mathcal{L}_\sigma}}{2 B \alpha N}}{\left( 2 \mu -  \frac{\eta L^2 \overline{\omega}\beta}{ \alpha N}\right)} = \frac{\eta L \overline{\mathcal{L}}_{\sigma}} {4 \mu N B} \implies \frac{\frac{ \overline{\mathcal{L}}_{\sigma} + \beta  \overline{\omega \mathcal{L}_\sigma} }{ \alpha } }{\left( 2 \mu -  \frac{\eta L^2 \overline{\omega}\beta}{ \alpha N}\right)} = \frac{\overline{\mathcal{L}}_{\sigma}} {2\mu } \\
     \implies & \frac{1}{\alpha}  2 \mu \left( \overline{\mathcal{L}}_{\sigma} + \beta  \overline{\omega \mathcal{L}_\sigma} \right) = \overline{\mathcal{L}}_{\sigma} \left( 2 \mu -  \frac{1}{\alpha}\frac{\eta L^2 \overline{\omega}\beta}{ N} \right),
\end{align}
which implies the thesis. All the other rules are derived similarly.
\end{proof}

\subsection{Stationary Distribution}

Let us focus on a quadratic function $f(x) = \frac{x^{\top} H x}{2}$ such that $H = \diag(\lambda_1, \cdots, \lambda_d)$ where each $\lambda_j >0$.
\begin{proposition}\label{prop:DCSGD_StaDistr_App}
     Let us consider the $k$-Sparsification compressor. Then,
     \begin{equation}
    \E[X_t] = e^{-H t} X_0 \rightarrow 0,
\end{equation}
and, for $M:= 2 H\left(I_d-\frac{\eta H}{2 N}\left( \frac{d}{k} - 1\right)\right)$,
\begin{equation}
    Cov  \left[ X_t \right] = e^{-M t} X_0^2 + \frac{\eta}{N}\frac{d}{k} \overline{\sigma^2} M^{-1} (I_d - e^{-M t}) - e^{-2H t} X_0^2 \rightarrow \frac{\eta}{N} \frac{d}{k} \overline{\sigma^2} M^{-1}.
\end{equation}
\end{proposition}

\begin{proof}

It is clear that 
\begin{equation}
    d \E[X_t] = - H \E[X_t] dt,
\end{equation}
which implies that
\begin{equation}
    \E[X_t] = e^{-H t} X_0 \rightarrow 0.
\end{equation}

Let us now focus on the dynamics of the square of the $j$-th coordinate $(X_t)_j$ of $X_t$ which, for ease of notation, we call $Z_t$. Since we need to apply Itô's Lemma on $((X_t)_j)^2$, we need to observe that since this is the square of $j$-th component of $X_t$, it can be rewritten as the square of the projection of $X_t$ on the $j$-th coordinate as $\pi_j(X_t)$. Therefore, we have that by Itô's Lemma:

\begin{align}
        d((\pi_j(X_t))^2) & = \partial_t ((\pi_j(X_t))^2) dt + \langle \nabla ((\pi_j(X_t))^2), \nabla f(X_t) \rangle dt \\
        & +  \frac{1}{2}\tr\left(\frac{\eta}{N} \Tilde{\Sigma}(X_t) \nabla^2 ((\pi_j(X_t))^2) \right)dt 
        + \langle\nabla (\pi_j(X_t))^2,\sigma_t\rangle dW_t \\
        & = - (H X_t)^{\top}  \nabla ((\pi_j(X_t))^2) dt + \mathcal{O}(\text{Noise}) + \frac{\eta}{2 N} \tr \left( \nabla^2 ((\pi_j(X_t))^2) \Tilde{\Sigma}(X_t) \right) dt.
\end{align}
Since $\nabla ((\pi_j(X_t))^2) = \left(0, \cdots, 0, 2 (X_t)_j, 0, \cdots, 0\right)$ and $ \nabla^2 ((\pi_j(X_t))^2) = \diag\left(0, \cdots, 0, 2 , 0, \cdots, 0\right)$,
we have that
\begin{align}
    d\left(\left((X_t)_j\right)^2\right) &= d \left((\pi_j(X_t))^2\right) = - (H X_t)^{\top}  \nabla ((\pi_j(X_t))^2)  dt + \mathcal{O}(\text{Noise})\\
    & + \frac{\eta}{2 N} \tr \left( \nabla^2 ((\pi_j(X_t))^2) \Tilde{\Sigma}(X_t) \right) dt,
\end{align}
meaning that 
\begin{equation}
    d (Z_t^2) = - 2 h_j Z_t^2 dt + \mathcal{O}(\text{Noise}) + \frac{ \eta}{N} \Tilde{\Sigma}_{jj}(X_t) dt.
\end{equation}
Since we have that 
\begin{align}
    \Tilde{\Sigma}(x) & = \frac{1}{N} \sum_{i=1}^{N} \E_{\xi_i \gamma_i} \left[ \left(  \mathcal{C}_{\xi_i} \left( \nabla f_{\gamma_i} (x) \right) - \nabla f(x) \right)\left(  \mathcal{C}_{\xi_i} \left( \nabla f_{\gamma_i} (x) \right) - \nabla f(x) \right)^{\top} \right] \\
    & = \frac{1}{N} \sum_{i=1}^{N} \E_{\xi_i \gamma_i} \left[  \mathcal{C}_{\xi_i} \left( \nabla f_{\gamma_i} (x) \right)\mathcal{C}_{\xi_i} \left( \nabla f_{\gamma_i} (x) \right)^{\top}\right] - \nabla f(x)\nabla f(x)^{\top}  \\
    & = \frac{1}{N} \sum_{i=1}^{N} \left( \frac{d}{k} \nabla f(x)\nabla f(x)^{\top} + \frac{d}{k} (\Sigma_i(X_t)) \right) - \nabla f(x)\nabla f(x)^{\top} \\
    & = \left( \frac{d}{k} - 1  \right)\nabla f(x)\nabla f(x)^{\top} + \frac{d}{k} \overline{\Sigma^2},
\end{align}
where $\overline{\Sigma^2}:= \frac{1}{N} \sum_{i=1}^{N} (\Sigma_i(X_t))$.
Therefore, we have that 
\begin{equation}
    \E\left[ (X_t)^2 \right] = e^{-M t} X_0^2 + \frac{\eta}{N} \frac{d}{k} \overline{\sigma^2} M^{-1} (I_d - e^{-M t})
\end{equation}
where $\overline{\sigma^2}:= \diag(\overline{\Sigma^2})$.
The thesis follows from here.
\end{proof}

\section{Distributed SignSGD}

This subsection provides the first formal derivation of an SDE model for DSignSGD. Note that the single node case was simultaneously tackled by \cite{compagnoni2025adaptive} and \cite{xiao2024exact}: The first derived the SDE for SignSGD under the WA framework, while \cite{xiao2024exact} derived an SDE for SignSGD in the high dimensional setting for a linear regression task --- See Appendix F in \cite{xiao2024exact} for a comparison between the two derivations. Let us consider the stochastic process $ X_t \in \mathbb{R}^{d} $ defined as the solution of

\begin{equation}\label{eq:DSignSGD_SDE}
   d X_t = - \frac{1}{N} \sum_{i=1}^{N} \left( 1 - 2 \mathbb{P}(\nabla f_{\gamma_i} (X_t) <0) \right) dt + \sqrt{\frac{\eta}{N}}\sqrt{\overline{\Sigma}(X_t)} dW_t.
\end{equation}
where
\begin{equation}\label{eq:DSignSGD_CovMatr}
    \overline{\Sigma}(X_t) := \frac{1}{N} \sum_{i=1}^{N} \overline{\Sigma_i}(X_t),
\end{equation}
and  $ \overline{\Sigma_i}(x) = \E[\xi_{\gamma_i}(x)\xi_{\gamma_i}(x)^\top]$ where $\xi_{\gamma_i}(x):= \sign (\nabla f_{\gamma_i}(x)) - 1 + 2 \mathbb{P}(\nabla f_{\gamma_i}(x)<0)$ the noise in the sample $ \sign \left(\nabla f_{\gamma_i}(x) \right)$.

Before proceeding, we ensure that the SDE admits a unique solution and that its coefficients are sufficiently regular.

\begin{lemma}
    The drift term $b(x):= \frac{1}{N} \sum_{i=1}^{N} \left( 1 - 2 \mathbb{P}(\nabla f_{\gamma_i}(x) < 0) \right)$ is Lipschitz, satisfies affine growth, and belongs to the space $G$ together with its derivatives.
\end{lemma}

\begin{proof}
    Since we are assuming that the gradient noise has a smooth and bounded probability density function,\footnote{This is commonly assumed in the literature. Among others, \cite{ahn2012bayesian,chen2014stochastic,mandt2016variational,stephan2017stochastic,zhu2019anisotropic,wu2020noisy,Xie2021} assume that it is Gaussian, while \cite{jastrzkebski2017three} offers an intuitive justification.} the drift can be rewritten in terms of the CDF $F_{Z_i}(x)$ of the noise as the average of $b_i(x):= 1 - 2 F_{Z_i}(-\nabla f(x))$, whose derivative is $2 F^{'}_{Z_i}(-\nabla f(x))\nabla^2 f(x)$. Since the density functions and the Hessian of $f$ are bounded, we conclude that the derivative is bounded. Therefore, the drift is Lipschitz and as regular as $\nabla f$, meaning that each entry is in $G$, together with its derivatives. Finally, since it is bounded, it has affine growth.
\end{proof}

\begin{lemma}
    The diffusion coefficient $\sqrt{\overline{\Sigma}}$ satisfies the affine growth condition.
\end{lemma}

\begin{proof}
    Since it is bounded, the result follows immediately.
\end{proof}

\begin{lemma}
\label{CorollaryNoise}
Let us assume the same assumptions as Lemma \ref{NoiseCondition}. Additionally, assume that
\[
\sup_{x \in K}\left\vert g(x,\cdot) \right\vert \in L^1(\mathbb{R}^n)
\]
for all compact sets $K$. Then the entries of $\overline{\Sigma}$ in Eq. \ref{eq:DSignSGD_CovMatr} are in $C_b^7(\mathbb{R}^n)$.
\end{lemma}

\begin{proof}
    By the definition of $\overline{\Sigma}$ in terms of the $\text{sign}$-function and dominated convergence, from the additional assumption on $g$, it follows that $\overline{\Sigma}$ is continuous. So Lemma \ref{NoiseCondition} entails that the entries of $\overline{\Sigma}$ are in $C_b^7(\mathbb{R}^n)$.
\end{proof}

\begin{lemma}
    Under the assumption that 
    \begin{equation}\label{PositiveDensity}
    g(x,y) > 0,
    \end{equation}
    the covariance matrix $\overline{\Sigma}$ is positive definite.
\end{lemma}

\begin{proof}
    Let us focus on the case $N=1$ as the generalization is straightforward. For $y = (y_1, \dots, y_n)^T$, observe that
    \[
    \left( \overline{\Sigma}(x) y, y \right) = \sum_{i,j=1}^{n} y_i \mathbb{E} \left[ \xi_{\gamma}^i(x) \xi_{\gamma}^j(x) \right] y_j = \mathbb{E} \left[ \left( \sum_{i=1}^{n} \xi_{\gamma}^i(x) y_i \right)^2 \right].
    \]
    Using the definition of $\xi_{\gamma}$ and the positivity of the density $g$, we can argue by contradiction and see that for $y \neq 0$, the right-hand side of the equation must be strictly greater than zero for all $x$. Therefore, $\overline{\Sigma}(x) \in \mathbb{S}$ for all $x$, where $\mathbb{S}$ denotes the open set of positive definite matrices in the space of symmetric $n \times n$ matrices.
\end{proof}

\begin{corollary}
    Since $\overline{\Sigma}$ is positive definite and its entries are in $C_b^7(\mathbb{R}^n)$, $\sqrt{\overline{\Sigma}}$ is Lipschitz.
\end{corollary}

\begin{proof}
    The function
    \[
    \varphi: \mathbb{S} \to \mathbb{S}, \quad A \mapsto \sqrt{A}
    \]
    has Fréchet derivatives of any order on $\mathbb{S}$ (see e.g. \cite{MN}). Therefore, $\overline{\Sigma}^{1/2} \in C^7(\mathbb{R}^n)$, and since $\overline{\Sigma} \in C_b^7(\mathbb{R}^n)$, $\overline{\Sigma}^{1/2}$ is Lipschitz continuous (see Proposition 6.2 in \cite{IW}).
\end{proof}

\begin{proposition}
\label{Prop}
Assume the conditions of Lemma \ref{lemma:Derivatives} and assume that the functions $l_m(x)$ for $m=1, \dots, 7$ in Eq. (\ref{A3}) are of polynomial growth. Then $\overline{\Sigma}^{1/2} \in G$ together with its derivatives.
\end{proposition}

\begin{corollary}
    If the noise $Z(x) \sim \mathcal{N}(0,\Sigma)$ or $Z(x) \sim t_{\nu}(0,\Sigma)$, then $\overline{\Sigma}^{1/2} \in G$ together with its derivatives.
\end{corollary}
\begin{proof}
With the definition of $\Xi_{\nu}(x)$ given in Corollary \ref{thm:DSignSGD_Theorem_Student_App}, the function $K(x) := \sqrt{1- 4\Xi_{\nu}(x)^2}$ is in $G$ together with its derivative: It is easy to verify that all the derivatives of $K(x)$ are bounded even in the case $\nu=1$, which is the most pathological one. Therefore, in the case $N=1$, $\sqrt{\overline{\Sigma}}(x)$ is in $G$ together with its derivatives. Generalizing to $N>1$ follows the same steps.
   
\end{proof}

\begin{remark}\label{remark:DsignSGD}
    Based on the above results, we have that under mild assumptions on the noise structures (see Sec. \ref{sec:AllAss}) that cover and generalize the well-accepted Gaussianity, e.g. covering Student's t as well, the SDE of DSignSGD admits a unique solution and its coefficients are regular enough to apply Prop. \ref{prop:li1} and Thm. \ref{thm:mils}.
\end{remark}

\begin{mybox}{gray}
\begin{theorem}[Stochastic modified equations] \label{thm:DSignSGD_SDE}
Let $0<\eta<1, T>0$ and set $N=\lfloor T / \eta\rfloor$. Let $ x_k \in \mathbb{R}^{d}, 0 \leq k \leq N$ denote a sequence of DSignSGD iterations defined by Eq.~\eqref{eq:DSignSGD_Discr_Update}. Consider the stochastic process $X_t$ defined in Eq.~\eqref{eq:DSignSGD_SDE} and fix some test function $g \in G$ and suppose that $g$ and its partial derivatives up to order 6 belong to $G$.

Then, under the assumptions of Section \ref{sec:AllAss}, there exists a constant $ C>0 $ independent of $ \eta $ such that for all $ k=0,1, \ldots, N $, we have

$$
\left|\E g\left(X_{k \eta}\right)-\E g\left(x_k\right)\right| \leq C \eta .
$$

That is, the SDE \eqref{eq:DSignSGD_SDE} is an order $ 1 $ weak approximation of the DSignSGD iterations \eqref{eq:DSignSGD_Discr_Update}.
\end{theorem}
\end{mybox}

\begin{proof}
First, we calculate the expected value of the increments of DSignSGD:
\begin{equation}
    \E \left[ x_{k+1} - x_{k}  \right] = \E\left[-\frac{\eta}{N}  \sum_{i=1}^{ N} \sign(\nabla f_{\gamma_i}(x_k))  \right] = - \frac{\eta}{N}  \sum_{i=1}^{ N} \left(1  - 2 \mathbb{P}(\nabla f_{\gamma_i}(x_k)<0) \right);
\end{equation}
Then, we calculate the covariance matrix of the gradient noise of DSignSGD:
\begin{align}
    \overline{\Sigma}(x_k) & = \eta^2 \E_{\gamma}\left[ \left( \frac{1}{N}  \sum_{i=1}^{ N} \sign(\nabla f_{\gamma_i}(x_k) -\frac{1}{N}  \sum_{i=1}^{ N} \left(1  - 2 \mathbb{P}(\nabla f_{\gamma_i}(x_k)<0) \right) \right)\right.\\
    &\qquad \qquad \left.\left( \frac{1}{N}  \sum_{i=1}^{ N} \sign(\nabla f_{\gamma_i}(x_k) -\frac{1}{N}  \sum_{i=1}^{ N} \left(1  - 2 \mathbb{P}(\nabla f_{\gamma_i}(x_k)<0) \right) \right)^{\top} \right] \\
    & = \frac{\eta^2}{N} \frac{1}{N} \sum_{i,j=1}^{N} \E _{\gamma_i\gamma_j} \left[ \left( \sign(\nabla f_i(x_k)) - 1 + 2 \mathbb{P}(\nabla f_{\gamma_i}(x_k)<0)  \right) \right.\\
    &\qquad \qquad \left.\left( \sign(\nabla f_j(x_k)) - 1 + 2 \mathbb{P}(\nabla f_{\gamma_j}(x_k)<0)  \right)^{\top}\right] \\
    & = \frac{\eta^2}{N} \frac{1}{N} \sum_{i=1}^{N} \E_{ \gamma_i} \left[ \left( \sign(\nabla f_i(x_k)) - 1 + 2 \mathbb{P}(\nabla f_{\gamma_i}(x_k)<0)  \right)\right.\\
    &\qquad \qquad \left.\left( \sign(\nabla f_i(x_k)) - 1 + 2 \mathbb{P}(\nabla f_{\gamma_i}(x_k)<0)  \right)^{\top}\right] \\
    & =  \frac{\eta^2}{N} \frac{1}{N} \sum_{i=1}^{N} \overline{\Sigma}_i(x_k).
\end{align}
where we use independence of $\gamma_i$ for all $i\in[N]$. Remembering Remark \ref{remark:DsignSGD}, the thesis follows from Prop.~\ref{prop:li1} and Thm.~\ref{thm:mils}.
\end{proof}
In all the following results, the reader will notice that all the drifts, diffusion terms, and noise assumptions are selected to guarantee that the SDE we derived for DSignSGD is indeed a $1$ weak approximation for DSignSGD.
\begin{corollary}\label{thm:DSignSGD_Theorem_Student_App}
Let us take the same assumptions of Theorem~\ref{thm:DSignSGD_SDE}, and that the stochastic gradients are $\nabla f_{\gamma_i}(x) = \nabla f(x) +  \sqrt{\Sigma_i} Z_i$ such that $Z_i \sim t_{\nu}(0, I_d) $ does not depend on $x$, $\nu$ are the degrees of freedom, and scale matrices $\Sigma_i= \diag(\sigma_{1,i}^2, \cdots, \sigma_{d,i}^2)$. Then, the SDE of DSignSGD is
\begin{equation}\label{eq:SDE_HSignSGD_Full_App}
    d X_t = - \frac{2}{N} \sum_{i=1}^{N} \Xi_{\nu} \left( \Sigma_i^{-\frac{1}{2}} \nabla f(X_t) \right) dt + \sqrt{\frac{\eta}{N}}\sqrt{\Tilde{\Sigma}(X_t)} dW_t.
\end{equation}
where $\Xi_{\nu}(x)$ is defined as $\Xi_{\nu}(x) := x \frac{\Gamma\left(\frac{\nu+1}{2}\right)}{\sqrt{\pi \nu} \Gamma\left(\frac{\nu}{2}\right)}{ }_2 F_1\left(\frac{1}{2}, \frac{\nu+1}{2} ; \frac{3}{2} ;-\frac{x^2}{\nu}\right)$, ${ }_2 F_1\left(a, b;c; x\right)$ is the hypergeometric function, and 
\begin{equation}
    \Tilde{\Sigma}(X_t) := I_d -  \frac{4}{N}\sum_{i=1}^{N}  \left(\Xi_{\nu} \left( \Sigma_i^{-\frac{1}{2}} \nabla f(X_t) \right) \right)^2.
\end{equation}
\end{corollary}

\begin{proof}
First of all, we observe that
\begin{align}
    & 1 - 2 \mathbb{P}\left(\nabla f_{\gamma_i}(x)<0\right) = 1 - 2 \mathbb{P}\left(\nabla f(x) + \Sigma_i^{\frac{1}{2}} U_i <0\right) = 1 - 2 F_{\nu} \left(-\Sigma_i^{-\frac{1}{2}} \nabla f(x) \right),
\end{align}
where $F_{\nu} \left( x \right)$ is the cumulative function of a $t$ distribution with $\nu$ degrees of freedom. Remembering that
\begin{equation}
    F_{\nu} \left( x \right) = \frac{1}{2}  + \Xi_{\nu}(x),
\end{equation}
we have that
\begin{equation}
    1 - 2 \mathbb{P}\left(\nabla f_{\gamma_i}(x)<0\right) = 1 - 2 \left(\frac{1}{2} + \Xi_{\nu}(-\Sigma_i^{-\frac{1}{2}} \nabla f(x)) \right) = 2 \Xi_{\nu}(\Sigma_i^{-\frac{1}{2}} \nabla f(x)).
\end{equation}
Similarly, one can prove that $\overline{\Sigma_i}$ becomes
\begin{equation}
    \bar \Sigma_i = I_d - 4\diag \left( \Xi_{\nu} \left( \Sigma_i^{-\frac{1}{2}} \nabla f(X_t) \right) \right)^2.
\end{equation}
\end{proof}

\begin{proposition} \label{prop:HSignSGD_three_phases_App}
Under the assumptions of Corollary~\ref{thm:DSignSGD_Theorem_Student_App} and signal-to-noise ratios $Y^i_t := \Sigma_i^{-\frac{1}{2}} \nabla f(X_t)$, let $\psi_{\nu}\in \mathbb{R}$ such that $\vert x \vert>\psi_{\nu} \implies 2 \vert\Xi_{\nu}(x)\vert \sim 1$. Then, the DSignSGD has three phases:
    \begin{enumerate} 
        \item \textbf{Phase~1:} If $ \left\vert Y^i_t  \right\vert>\psi_{\nu}$, the SDE coincides with the ODE of SignGD:
        \begin{equation} \label{sde:HSignSGD_ph1_App}
            d X_t = - \sign ( \nabla f(X_t)) dt;
        \end{equation}
        \item \textbf{Phase~2:}  If $ 1 < \left\vert Y^i_t \right\vert< \psi_{\nu}$:\footnote{Let $m_{\nu}$ and $q_{\nu,1}$ are the slope and intercept of the line secant to the graph of $2\Xi_{\nu}(x)$ between the points $(1,2\Xi_{\nu}(1))$ and $\left(\psi_{\nu},2\Xi_{\nu}\left(\psi_{\nu}\right)\right)$, while $q_{\nu,2}$ is the intercept of the line tangent to the graph of $2\Xi_{\nu}(x)$ and slope $m_{\nu}$, $ (\mathbf{q}^{+}_{\nu})_i := \begin{dcases}
        q_{\nu,2}  & \text{if $\partial_i f(x)>0$ } \\
        -q_{\nu,1} & \text{if $\partial_i f(x)<0$ }
        \end{dcases}$, $(\mathbf{q}^{-})_i := \begin{dcases}
        q_{\nu,1}  & \text{if $\partial_i f(x)>0$ } \\
        -q_{\nu,2} & \text{if $\partial_i f(x)<0$ }
        \end{dcases}$, and $\hat{q}_{\nu}:= \max(q_{\nu,1},q_{\nu,2})$.}
        \begin{enumerate}
            \item $- m_{\nu} \left(\frac{1}{N} \sum_{i=1}^N \Sigma_i^{-\frac{1}{2}} \right) \nabla f(X_t) -  \mathbf{q}_{\nu}^{+} \leq \frac{d \E \left[ X_t \right]}{dt}  \leq - m_{\nu} \left(\frac{1}{N} \sum_{i=1}^N \Sigma_i^{-\frac{1}{2}} \right) \nabla f(X_t) -  \mathbf{q}_{\nu}^{-}$;
            \item $\mathbb{P}\left[ \lVert X_t - \E \left[ X_t \right] \rVert^2_2 > a \right] \leq \frac{\eta}{a} \left( d - \frac{1}{N} \sum_{i=1}^N \lVert  m_{\nu} Y^i_t + \mathbf{q}_{\nu}^{-} \rVert^2_2\right)$;
        \end{enumerate}
        \item \textbf{Phase~3:} If $ \left\vert Y^i_t \right\vert<1$ and $\ell_{\nu} := 2 \Xi_{\nu}^{'}(0)$, the SDE is
        \begin{equation} \label{sde:HSignSGD_ph3_App}
            d X_t = -  \ell_{\mathbf{{\textcolor{Rhodamine}{\nu}}}} \left(\frac{1}{N} \sum_{i=1}^N \Sigma_i^{-\frac{1}{2}} \right) \nabla f(X_t) dt + \sqrt{\frac{\eta}{N}} \sqrt{I_d - \frac{\ell_{\mathbf{{\textcolor{Rhodamine}{\nu}}}}^2}{N} \sum_{i=1}^N \diag \left(\Sigma_i^{-\frac{1}{2}} \nabla f(X_t) \right)^2} d W_t.
        \end{equation}
    \end{enumerate}
\end{proposition}
\begin{proof}
Exploiting the regularity of the $\Xi_{\nu}(x)$ function, we approximate the SDE in \eqref{eq:SDE_HSignSGD_Full_App} in three different regions:
\begin{enumerate}
    \item \textbf{Phase~1:} If $ \left\vert x \right\vert>\psi_{\nu}$, $2\Xi_{\nu}(x) \sim \sign(x)$. Therefore, if $ \left\vert \Sigma_i^{-\frac{1}{2}} \nabla f(X_t) \right\vert>\psi_{\nu}$,
    \begin{enumerate}
        \item $2\Xi_{\nu} \left( \Sigma_i^{-\frac{1}{2}} \nabla f(X_t) \right) \sim \sign \left( \Sigma_i^{-\frac{1}{2}} \nabla f(X_t) \right) = \sign \left( \nabla f(X_t)\right) $;
        \item $4\Xi_{\nu} \left( \Sigma_i^{-\frac{1}{2}} \nabla f(X_t) \right)^2 \sim \sign \left( \Sigma_i^{-\frac{1}{2}} \nabla f(X_t) \right)^2 = (1, \dots, 1) $.
    \end{enumerate}
    Therefore, 
    \begin{align}
        d X_t \sim - \sign ( \nabla f(X_t)) dt;
    \end{align}
    \item \textbf{Phase~2:} If $ 1 <  x < \psi_{\nu}$, we have that
    \begin{equation}
        m_{\nu} x + q_{\nu,1} < 2\Xi_{\nu}(x) < m_{\nu} x + q_{\nu,2}.
    \end{equation}
    Analogously, if $ -\psi_{\nu} <  x < -1$
    \begin{equation}
        m_{\nu} x - q_{\nu,2} < 2\Xi_{\nu}(x) < m_{\nu} x - q_{\nu,1}.
    \end{equation}
    Therefore, we have that if $ 1 < \left\vert Y_t^i \right\vert < \psi_{\nu}$, then
    \begin{enumerate}
        \item \begin{equation}
       m_{\nu} \Sigma_i^{-\frac{1}{2}} \nabla f(X_t) + \mathbf{q}_{\nu}^{-} < 2 \Xi_{\nu} \left( \Sigma_i^{-\frac{1}{2}} \nabla f(X_t) \right) < m_{\nu} \Sigma_i^{-\frac{1}{2}} \nabla f(X_t) + \mathbf{q}_{\nu}^{+}.
    \end{equation}
    Therefore,
    \begin{equation}
                -m_{\nu} \left(\frac{1}{N} \sum_{i=1}^N \Sigma_i^{-\frac{1}{2}} \right) \nabla f(X_t) - \mathbf{q}_{\nu}^{+} \leq \frac{d \E \left[ X_t \right]}{dt}  \leq -m_{\nu}\left(\frac{1}{N} \sum_{i=1}^N \Sigma_i^{-\frac{1}{2}} \right) \nabla f(X_t) + \mathbf{q}_{\nu}^{-};
            \end{equation}
    \item Similar to the above, 
    \begin{equation*}
        \left(  m_{\nu} \Sigma_i^{-\frac{1}{2}} \nabla f(X_t) + \mathbf{q}_{\nu}^{-} \right)^2 \leq 4 \Xi_{\nu} \left( \Sigma_i^{-\frac{1}{2}} \nabla f(X_t) \right)^2 \leq \left(m_{\nu} \Sigma_i^{-\frac{1}{2}} \nabla f(X_t) + \mathbf{q}_{\nu}^{+} \right)^2.
    \end{equation*}
    Therefore,
    \begin{align}
        \mathbb{P}\left[ \lVert X_t - \E \left[ X_t \right] \rVert^2_2 > a \right] & \leq \mathbb{P}\left[ \exists j \text{ s.t. }\vert X^j_t - \E \left[ X^j_t \right] \vert^2 > a \right] \\
        & \leq \sum_j \mathbb{P}\left[\vert X^j_t - \E \left[ X^j_t \right] \vert > \sqrt{a} \right] \nonumber \\
        &  \leq \frac{\eta}{a} \sum_j \left( 1 - \frac{4}{N}\sum_{i=1}^{N}\Xi_{\nu}\left((\Sigma_i)_{j}^{-\frac{1}{2}} \partial_j f(X_t) \right) ^2\right) \\
        & < \frac{\eta}{a} \left( d - \frac{1}{N}\sum_{i=1}^N \lVert  m_{\nu} \Sigma_i^{-\frac{1}{2}} \nabla f(X_t) + \mathbf{q}_{\nu}^{-} \rVert^2_2\right).
    \end{align}
    
    \end{enumerate}
    \item  \textbf{Phase~3:} If $ \left\vert x \right\vert<1$, $2 \Xi_{\nu} (x) \sim \ell_{\nu} x$ for $\ell_{\nu} := 2 \Xi_{\nu}^{'}(0)$. Therefore, if $ \left\vert \Sigma_i^{-\frac{1}{2}} \nabla f(X_t) \right\vert<1$,
    \begin{enumerate}
        \item $2 \Xi_{\nu} \left(\Sigma_i^{-\frac{1}{2}} \nabla f(X_t) \right) \sim \ell_{\nu} \Sigma_i^{-\frac{1}{2}} \nabla f(X_t) $;
        \item $ 4\left(\Xi_{\nu} \left(\Sigma_i^{-\frac{1}{2}} \nabla f(X_t) \right) \right) ^2 \sim \ell_{\nu}^2 \left(  \Sigma_i^{-\frac{1}{2}} \nabla f(X_t)\right)^2$.
    \end{enumerate}
    Therefore, 
    \begin{align}
        d X_t = -  \ell_{\mathbf{{\textcolor{Rhodamine}{\nu}}}} \left(\frac{1}{N} \sum_{i=1}^N \Sigma_i^{-\frac{1}{2}} \right) \nabla f(X_t) dt + \sqrt{\frac{\eta}{N}} \sqrt{I_d - \frac{\ell_{\mathbf{{\textcolor{Rhodamine}{\nu}}}}^2}{N} \sum_{i=1}^N \diag \left(\Sigma_i^{-\frac{1}{2}} \nabla f(X_t) \right)^2} d W_t.
    \end{align}
\end{enumerate}
\end{proof}

\begin{theorem}
Let $f$ be $\mu$-strongly convex, $\tr(\nabla^2 f(x)) < \mathcal{L}_{\tau}$, $\Sigma_i \leq \sigma_{\text{max}, i}^2$, $S_t:=f(X_t) - f(X_*)$, and $\sigma_{\mathcal{H},j}$ be the harmonic mean of $\{(\sigma_{\text{max}, i})^j \}$. Then, if all agents are in
\begin{enumerate}
    \item Phase~1, the loss will reach $0$ before $t_* = 2 \sqrt{\frac{S_0}{\mu}}$ because $S_t \leq \frac{1}{4} \left( \sqrt{\mu}t - 2 \sqrt{S_0}\right)^2$;
    \item Phase~2, $ \E[S_t] \leq S_0 e^{- 2 \mu \Delta t} +  \frac{\eta (\mathcal{L}_{\tau} - \mu d \hat{q}^2)}{2 N} \frac{1}{ 2 \mu \Delta} \left(1 - e^{- 2 \mu \Delta t}\right)$
    with $\Delta:=  m_{\nu} \sigma_{\mathcal{H},1}^{-1} + \frac{\eta \mu m_{\nu}^2}{2 N} \sigma_{\mathcal{H},2}^{-1}$;
    \item Phase~3, $ \E[S_t] \leq S_0 e^{- 2 \mu \Delta t} +  \frac{\eta \mathcal{L}_{\tau}}{2 N} \frac{1}{ 2 \mu \Delta} \left(1 - e^{- 2 \mu \Delta t}\right)$ with $\Delta:=  \ell_{\nu} \sigma_{\mathcal{H},1}^{-1} + \frac{\eta \mu \ell_{\nu}^2}{2 N} \sigma_{\mathcal{H},2}^{-1}$.
\end{enumerate}
\end{theorem}
\begin{proof}
Let us prove the above phase by phase:

\textbf{For Phase~1}, 
\begin{equation}
        d(f(X_t) - f(X_*)) = - \nabla f(X_t) \sign ( \nabla f(X_t))dt = -\lVert \nabla f(X_t)\rVert_1 dt\leq -\lVert \nabla f(X_t)\rVert_2 dt.
    \end{equation}
    Since $f$ is $\mu$-PL, we have that $-\lVert \nabla f(X_t)\rVert_2^2<-  2 \mu ( f(X_t) - f(X_*))$, which implies that
    \begin{equation}
            f(X_t) - f(X_*) \leq \frac{1}{4} \left( \sqrt{\mu}t - 2 \sqrt{f(X_0) - f(X_*)}\right)^2,
        \end{equation}
        meaning that the dynamics will stop before $t_* = 2 \sqrt{\frac{f(X_0) - f(X_*)}{\mu}}$;

\textbf{For Phase~2}, using Itô on $f$, we have that
\begin{align}
        d (f(X_t) - f(X_*)) & = - \frac{m_{\nu} }{N} \sum_{i=1}^{ N} \nabla f(X_t)^{\top} \Sigma_i^{-\frac{1}{2}} \nabla f(X_t) dt - \nabla f(X_t)^{\top }\mathbf{q}_{\nu}^{-}dt +\frac{\eta \mathcal{L}_{\tau}}{2 N} dt \\
        & +\mathcal{O}(\text{Noise}) - \frac{\eta \mu}{2 N} \frac{1}{N} \sum_{i=1}^{N} \lVert m_{\nu} \Sigma_i^{-\frac{1}{2}} \nabla f(X_t) + \mathbf{q}_{\nu} \rVert_2^2 dt \\
        & \leq - m_{\nu} \sigma_{\mathcal{H},1}^{-1} \lVert \nabla f(X_t) \rVert_2^2  dt - \hat{q}\lVert \nabla f(X_t) \rVert_1 dt+ \frac{\eta \mathcal{L}_{\tau}}{2 N} dt \\
        & + \mathcal{O}(\text{Noise}) - \frac{\eta \mu d \hat{q}^2}{2 N} dt - \frac{\eta m_{\nu}^2 \mu}{2 N} \sigma_{\mathcal{H},2}^{-1} \lVert \nabla f(X_t) \rVert_2^2 dt - \frac{\eta \mu m_{\nu} \sigma_{\mathcal{H},1}^{-1} \hat{q}}{N} \lVert \nabla f(X_t) \rVert_1 dt \\
        &  \leq - \left( m_{\nu} \sigma_{\mathcal{H},1}^{-1} + \frac{\eta \mu m_{\nu}^2}{2 N} \sigma_{\mathcal{H},2}^{-1} \right) \lVert \nabla f(X_t) \rVert_2^2  dt + \frac{\eta (\mathcal{L}_{\tau} - \mu d \hat{q}^2)}{2 N} dt + \mathcal{O}(\text{Noise}) \\
        & \leq - 2 \mu \left( m_{\nu} \sigma_{\mathcal{H},1}^{-1} + \frac{\eta \mu m_{\nu}^2}{2 N} \sigma_{\mathcal{H},2}^{-1} \right) (f(X_t) - f(X_*))  dt + \frac{\eta (\mathcal{L}_{\tau} - \mu d \hat{q}^2)}{2 N} dt + \mathcal{O}(\text{Noise})
    \end{align}
meaning that
\begin{equation}
    \E[f(X_t) - f(X_*)] \leq (f(X_0) - f(X_*)) e^{- 2 \mu \Delta t} +  \frac{\eta (\mathcal{L}_{\tau} - \mu d \hat{q}^2)}{2 N} \frac{1}{ 2 \mu \Delta} \left(1 - e^{- 2 \mu \Delta t}\right)
\end{equation}
 with $\Delta:=  m_{\nu} \sigma_{\mathcal{H},1}^{-1} + \frac{\eta \mu m_{\nu}^2}{2 N} \sigma_{\mathcal{H},2}^{-1}$, $\sigma_{\mathcal{H},j}$ is the harmonic mean of $\{(\sigma_{\text{max}, i})^j \}$.

\textbf{For Phase~3}, using Itô on $f$, we have that
    \begin{align}
        d (f(X_t) - f(X_*)) & = - \frac{\ell_{\nu} }{N} \sum_{i=1}^{ N} \nabla f(X_t)^{\top} \Sigma_i^{-\frac{1}{2}} \nabla f(X_t) dt + \frac{\eta \mathcal{L}_{\tau}}{2 N} dt \\
        & +\mathcal{O}(\text{Noise}) - \frac{\eta}{2 N} \frac{\ell_{\nu}^2}{N} \sum_{i=1}^{N} \lVert \nabla^2 f(X_t) \Sigma_i^{-\frac{1}{2}} \nabla f(X_t) \rVert_2^2 dt \\
        & \leq - \ell_{\nu} \left(\frac{1}{N} \sum_{i=1}^{ N} \frac{1}{\sigma_{\text{max},i}}\right) \lVert \nabla f(X_t) \rVert_2^2  dt + \frac{\eta \mathcal{L}_{\tau}}{2 N} dt \\
        & + \mathcal{O}(\text{Noise}) - \frac{\eta \mu \ell_{\nu}^2}{2 N} \left(\frac{1}{N} \sum_{i=1}^{ N} \frac{1}{\sigma_{\text{max},i}^2}\right) \lVert \nabla f(X_t) \rVert_2^2 dt \\
        & = - \ell_{\nu} \sigma_{\mathcal{H},1}^{-1} \lVert \nabla f(X_t) \rVert_2^2  dt + \frac{\eta \mathcal{L}_{\tau}}{2 N} dt + \mathcal{O}(\text{Noise}) - \frac{\eta \mu \ell_{\nu}^2}{2 N} \sigma_{\mathcal{H},2}^{-1} \lVert \nabla f(X_t) \rVert_2^2 dt \\
        &  \leq - \left( \ell_{\nu} \sigma_{\mathcal{H},1}^{-1} + \frac{\eta \mu \ell_{\nu}^2}{2 N} \sigma_{\mathcal{H},2}^{-1} \right) \lVert \nabla f(X_t) \rVert_2^2  dt + \frac{\eta \mathcal{L}_{\tau}}{2 N} dt + \mathcal{O}(\text{Noise}) \\
        & \leq - 2 \mu \left( \ell_{\nu} \sigma_{\mathcal{H},1}^{-1} + \frac{\eta \mu \ell_{\nu}^2}{2 N} \sigma_{\mathcal{H},2}^{-1} \right) (f(X_t) - f(X_*))  dt + \frac{\eta \mathcal{L}_{\tau}}{2 N} dt + \mathcal{O}(\text{Noise})
    \end{align}
meaning that
\begin{equation}
    \E[f(X_t) - f(X_*)] \leq (f(X_0) - f(X_*)) e^{- 2 \mu \Delta t} +  \frac{\eta \mathcal{L}_{\tau}}{2 N} \frac{1}{ 2 \mu \Delta} \left(1 - e^{- 2 \mu \Delta t}\right)
\end{equation}
 with $\Delta:=  \ell_{\nu} \sigma_{\mathcal{H},1}^{-1} + \frac{\eta \mu \ell_{\nu}^2}{2 N} \sigma_{\mathcal{H},2}^{-1}$, $\sigma_{\mathcal{H},j}$ is the harmonic mean of $\{(\sigma_{\text{max}, i})^j \}$, and $\ell_{\nu} := 2\Xi^{'}_{\nu}(0)$.
\end{proof}

\textbf{Remark:} If not all the agents are in the same Phase, we can upper bound the dynamics of $df_t$ with the case where they are all in the third Phase, which is that of weakest descent.

\begin{theorem}
Let $f$ be $\mu$-PL, $L$-Smooth, $\Sigma_i \leq \sigma_{\text{max}, i}^2$, $S_t:=f(X_t) - f(X_*)$, and $\sigma_{\mathcal{H},j}$ be the harmonic mean of $\{(\sigma_{\text{max}, i})^j \}$. Then, if all agents are in
\begin{enumerate}
    \item Phase~1, the loss will reach $0$ before $t_* = 2 \sqrt{\frac{S_0}{\mu}}$ because $S_t \leq \frac{1}{4} \left( \sqrt{\mu}t - 2 \sqrt{S_0}\right)^2$;
    \item Phase~2, $ \E[S_t] \leq S_0 e^{- 2 \mu \Delta t} +  \frac{\eta L d}{4 \mu \Delta N} \left(1 - e^{- 2 \mu \Delta t}\right)$
    with $\Delta:=  m_{\nu} \sigma_{\mathcal{H},1}^{-1}$;
    \item Phase~3, $ \E[S_t] \leq S_0 e^{- 2 \mu \Delta t} +  \frac{\eta L d}{4 \mu \Delta N} \left(1 - e^{- 2 \mu \Delta t}\right)$
    with $\Delta:=  \ell_{\nu} \sigma_{\mathcal{H},1}^{-1}$;
\end{enumerate}
\end{theorem}
\begin{proof}
Let us prove the above phase by phase:

\textbf{For Phase~1}, 
\begin{equation}
        d(f(X_t) - f(X_*)) = - \nabla f(X_t) \sign ( \nabla f(X_t))dt = -\lVert \nabla f(X_t)\rVert_1 dt\leq -\lVert \nabla f(X_t)\rVert_2 dt.
    \end{equation}
    Since $f$ is $\mu$-PL, we have that $-\lVert \nabla f(X_t)\rVert_2^2<-  2 \mu ( f(X_t) - f(X_*))$, which implies that
    \begin{equation}
            f(X_t) - f(X_*) \leq \frac{1}{4} \left( \sqrt{\mu}t - 2 \sqrt{f(X_0) - f(X_*)}\right)^2,
        \end{equation}
        meaning that the dynamics will stop before $t_* = 2 \sqrt{\frac{f(X_0) - f(X_*)}{\mu}}$;

\textbf{For Phase~2}, using Itô on $f$, we have that
\begin{align}
        d (f(X_t) - f(X_*)) & = - \frac{m_{\nu} }{N} \sum_{i=1}^{ N} \nabla f(X_t)^{\top} \Sigma_i^{-\frac{1}{2}} \nabla f(X_t) dt - \nabla f(X_t)^{\top }\mathbf{q}_{\nu}^{-}dt +\frac{\eta L d}{2 N} dt \\
        & \leq - 2 \mu  m_{\nu} \sigma_{\mathcal{H},1}^{-1} (f(X_t) - f(X_*))  dt + \frac{\eta L d }{2 N} dt + \mathcal{O}(\text{Noise})
    \end{align}
which implies the thesis.

\textbf{For Phase~3}, using Itô on $f$, we have that
\begin{align}
        d (f(X_t) - f(X_*)) & = - \frac{\ell_{\nu} }{N} \sum_{i=1}^{ N} \nabla f(X_t)^{\top} \Sigma_i^{-\frac{1}{2}} \nabla f(X_t) dt - \nabla f(X_t)^{\top }\mathbf{q}_{\nu}^{-}dt +\frac{\eta L d}{2 N} dt \\
        & \leq - 2 \mu  \ell_{\nu} \sigma_{\mathcal{H},1}^{-1} (f(X_t) - f(X_*))  dt + \frac{\eta L d }{2 N} dt + \mathcal{O}(\text{Noise})
    \end{align}
which implies the thesis.
\end{proof}

\begin{theorem}
    If $f$ is $L$-smooth, we use a learning rate scheduler $\eta_t$ such that $\phi^i_t = \int_0^t (\eta_s)^i ds$, $\phi^1_t \overset{t \rightarrow \infty}{\rightarrow} \infty$, $\frac{\phi^2_t}{\phi^1_t} \overset{t \rightarrow \infty}{\rightarrow} 0$, and $\Sigma_i \leq \sigma_{\text{max}, i}^2$.

    \begin{enumerate}
        \item In Phase~1, $\lVert \nabla f\left(X_{\Tilde{t}^1}\right)\rVert_1 \leq \frac{f(X_0) - f(X_*)}{\phi_t^1} \overset{t \rightarrow \infty}{\rightarrow} 0$;
        \item In Phase~2, 
        \begin{equation}
         m_{\nu} \E \lVert \nabla f\left(X_{\Tilde{t}^{(1,2)}}\right)\rVert_2^2 + \hat{q}  \sigma_{\mathcal{H},1} \E \lVert \nabla f\left(X_{\Tilde{t}^{(2,2)}}\right)\rVert_1  \leq \frac{ \sigma_{\mathcal{H},1}}{\phi^1_t} \left( f(X_0) - f(X_*) + \frac{\eta L d \phi_t^2}{2 N} \right) \overset{t \rightarrow \infty}{\rightarrow} 0;
    \end{equation}
        \item In Phase~3, 
        \begin{equation}
        \ell_{\nu} \E \lVert \nabla f\left(X_{\Tilde{t}^{3}}\right)\rVert_2^2 \leq \frac{ \sigma_{\mathcal{H},1}}{\phi^1_t} \left( f(X_0) - f(X_*) + \frac{\eta L d \phi_t^2}{2 N} \right) \overset{t \rightarrow \infty}{\rightarrow} 0.
    \end{equation}
    \end{enumerate}
Above, $\Tilde{t}^1$, $\Tilde{t}^{(1,2)}$, $\Tilde{t}^{(2,2)}$, and $\Tilde{t}^3$ are random times with distribution $\frac{\eta_t}{\phi^1_t}$.

\end{theorem}
\begin{proof}
Let us prove the above phase by phase:

\textbf{For Phase~1}, \begin{align}
        d(f(X_t) - f(X_*)) & = - \eta_t \nabla f(X_t) \sign ( \nabla f(X_t)) dt = - \eta_t \lVert \nabla f(X_t)\rVert_1 dt \\
        & = - \phi_t^1 \frac{\eta_t \lVert \nabla f(X_t)\rVert_1 }{\phi_t^1} dt
    \end{align}
    Therefore, by integrating over time and using the law of the unconscious statistician
    \begin{equation}
        \lVert \nabla f\left(X_{\Tilde{t}^1}\right)\rVert_1 \leq \frac{f(X_0) - f(X_*)}{\phi_t^1} \overset{t \rightarrow \infty}{\rightarrow} 0;
    \end{equation}
where $\Tilde{t}^1$ is a random time with distribution $\frac{\eta_t}{\phi^1_t}$;
\end{proof}

\textbf{For Phase~2}, using Itô on $f$, we have that
\begin{align}
        d (f(X_t) - f(X_*)) &  \leq - \eta_t  m_{\nu} \sigma_{\mathcal{H},1}^{-1}\lVert \nabla f(X_t) \rVert_2^2  dt - \eta_t \hat{q} \lVert \nabla f(X_t) \rVert_1 dt + \eta_t^2 \frac{\eta L d}{2 N} dt + \mathcal{O}(\text{Noise})
    \end{align}
Therefore, by integrating over time and using the law of the unconscious statistician we have
    \begin{equation}
         m_{\nu} \E \lVert \nabla f\left(X_{\Tilde{t}^{(1,2)}}\right)\rVert_2^2 + \hat{q}  \sigma_{\mathcal{H},1} \E \lVert \nabla f\left(X_{\Tilde{t}^{(2,2)}}\right)\rVert_1  \leq \frac{ \sigma_{\mathcal{H},1}}{\phi^1_t} \left( f(X_0) - f(X_*) + \frac{\eta L d \phi_t^2}{2 N} \right) \overset{t \rightarrow \infty}{\rightarrow} 0,
    \end{equation}
where $\Tilde{t}^{(1,2)}$, $\Tilde{t}^{(2,2)}$, and $\Tilde{t}^3$ are random times with distribution $\frac{\eta_t}{\phi^1_t}$;

\textbf{For Phase~3}, using Itô on $f$, we have that
\begin{align}
        d (f(X_t) - f(X_*)) &  \leq - \eta_t  \ell_{\nu} \sigma_{\mathcal{H},1}^{-1}\lVert \nabla f(X_t) \rVert_2^2  dt + \eta_t^2 \frac{\eta L d}{2 N} dt + \mathcal{O}(\text{Noise})
    \end{align}
Therefore, by integrating over time and using the law of the unconscious statistician we have
    \begin{equation}
         \ell_{\nu} \E \lVert \nabla f\left(X_{\Tilde{t}^{3}}\right)\rVert_2^2 \leq \frac{ \sigma_{\mathcal{H},1}}{\phi^1_t} \left( f(X_0) - f(X_*) + \frac{\eta L d \phi_t^2}{2 N} \right) \overset{t \rightarrow \infty}{\rightarrow} 0,
    \end{equation}
where $\Tilde{t}^3$ is a random time with distribution $\frac{\eta_t}{\phi^1_t}$.

\subsection{Scaling Rules}

\begin{proposition}\label{prop:DSignSGD_ScalingLaws}
    Let the batch size be $ \delta B$, learning rate $\kappa \eta$, and $\alpha N$ agents. Let $K_1 := \ell_{\nu} \sqrt{B} \sigma_{\mathcal{H},1}^{-1}$ and $K_2:=\frac{ \eta \ell_{\nu}^2 B \mu \sigma_{\mathcal{H},2}^{-1}}{2 N}$. The scaling rules (involving only two parameters at the time) to preserve the performance independently of $\delta$, $\kappa$, and $\alpha$, are:
    \begin{table}[ht]
    \centering
    \begin{tabular}{|c|c|}
        \hline
        \textbf{Scaling Rule} & \textbf{Implication}\\
        \hline
        $\alpha = \frac{1}{\sqrt{\delta}} + \frac{K_2}{K_1} \left( \frac{1}{\sqrt{\delta}} - \sqrt{\delta} \right) $ & BS $\downarrow \implies$ Agents $\uparrow$\\
        \hline
        $\alpha = \kappa$ & LR $\uparrow \implies$ Agents $\uparrow$\\
        \hline
        $\kappa = \frac{\sqrt{\delta}}{1 + \frac{K_1}{K_2}(1 -\delta)}$ & BS $\uparrow \implies$ Agents $\uparrow$\\
        \hline
    \end{tabular}
    \caption{Summary of Trade-offs Between Parameters (LR = Learning Rate and BS = Batch Size).}
\end{table}
Finally, we observe that if $\frac{K_1}{K_2}\sim 0$, for example when $N \gg 1$, then these rules can be summarized as $\frac{\kappa}{\alpha \sqrt{\delta}}=1$, which recover the Scaling Rules of Adam and RMSprop as well as allow for the enhanced design flexibility of the distributed setting.
\end{proposition}
\begin{proof}
Let us focus on Phase~3, which is when the dynamics reaches stability. Using Itô's Lemma on $f$ we have that
    \begin{align}
        d (f(X_t) - f(X_*)) & = - \frac{ \kappa}{\alpha N} \sum_{i=1}^{\alpha N} \ell_{\nu} \sqrt{\delta B} \nabla f(X_t)^{\top} \Sigma_i^{-\frac{1}{2}} \nabla f(X_t) dt + \frac{\eta \kappa^2  \mathcal{L}_{\tau}}{2 \alpha N} dt \\
        & +\mathcal{O}(\text{Noise}) - \frac{\eta \kappa^2}{2 \alpha N} \frac{\ell_{\nu}^2 B \delta}{\alpha N} \sum_{i=1}^{\alpha N} \lVert \nabla^2 f(X_t) \Sigma_i^{-\frac{1}{2}} \nabla f(X_t) \rVert_2^2 dt \\
        & \leq - (2\mu \kappa \ell_{\nu} \sqrt{\delta B}) \left(\frac{1}{\alpha N}\sum_{i=1}^{\alpha
         N} \frac{1}{\sigma_{\text{max},i}} \right)(f(X_t) - f(X_*)) dt + \frac{\eta \kappa^2  \mathcal{L}_{\tau}}{2 \alpha N} dt \\
        & +\mathcal{O}(\text{Noise}) - \frac{2 \mu^2 \eta \kappa^2}{2 \alpha N} \ell_{\nu}^2 B \delta\left(\frac{1}{\alpha N}\sum_{i=1}^{\alpha
         N} \frac{1}{\sigma_{\text{max},i}^2} \right)(f(X_t) - f(X_*))dt,
    \end{align}
meaning that
\begin{equation}
    \E[f(X_t) - f(X_*)] \leq (f(X_0) - f(X_*)) e^{- 2 \mu \Delta t} +  \frac{\eta \kappa^2}{2 \alpha N} \frac{\mathcal{L}_{\tau}}{ 2 \mu \Delta} \left(1 - e^{- 2 \mu \Delta t}\right)
\end{equation}
 with $\Delta:= \left( \ell_{\nu} \kappa \sqrt{\delta B} \sigma_{\mathcal{H},1}^{-1} + \frac{ \eta \kappa^2}{2\alpha N} \ell_{\nu}^2 \delta B \mu \sigma_{\mathcal{H},2}^{-1}  \right)$.

 The asymptotic limit is thus:

\begin{equation}
    \frac{\eta \mathcal{L}_{\tau}}{4 \mu N} \frac{\kappa}{\alpha \sqrt{\delta}} \frac{1}{ \ell_{\nu} \sqrt{B} \sigma_{\mathcal{H},1}^{-1} +  \frac{ \eta \ell_{\nu}^2 B \mu \sigma_{\mathcal{H},2}^{-1}}{2 N}  \frac{\kappa \sqrt{\delta}}{\alpha}}.
\end{equation}
To maintain the performance of DSignSGD independently of $\alpha$, $\kappa$, and $\delta$, we need to solve the following equation:
\begin{equation}
    \frac{\eta \mathcal{L}_{\tau}}{4 \mu N} \frac{\kappa}{\alpha \sqrt{\delta}} \frac{1}{ \ell_{\nu} \sqrt{B} \sigma_{\mathcal{H},1}^{-1} +  \frac{ \eta \ell_{\nu}^2 B \mu \sigma_{\mathcal{H},2}^{-1}}{2 N}  \frac{\kappa \sqrt{\delta}}{\alpha}} = \frac{\eta \mathcal{L}_{\tau}}{4 \mu N}  \frac{1}{ \ell_{\nu} \sqrt{B} \sigma_{\mathcal{H},1}^{-1} +  \frac{ \eta \ell_{\nu}^2 B \mu \sigma_{\mathcal{H},2}^{-1}}{2 N}}.
\end{equation}
 To provide easily interpretable and actionable scaling rules, we fix one of the three parameters to $1$ and find the relationship between the others. With simple math, the thesis follows.
\end{proof}

\subsection{Stationary Distribution}

\begin{proposition}
\label{prop:HSignSGD_StaDistr_App}
Let $H = \diag (\lambda_1, \dots, \lambda_d)$, $M_t:=e^{-2\left( \ell_{\nu} \Sigma_{\mathcal{H},1} H + \frac{\eta}{2 N}\ell_{\nu}^2 \Sigma_{\mathcal{H},2} H^2 \right) t}$ where $\Sigma_{\mathcal{H},1}=\frac{1}{N}\sum_{i=1}^{N} \Sigma_i^{-\frac{1}{2}}$, and $\Sigma_{\mathcal{H},2}=\frac{1}{N}\sum_{i=1}^{N} \Sigma_i^{-1}$. Then,
\begin{enumerate}
    \item $\E  \left[ X_t \right] = e^{- \ell_{\nu} \Sigma_{\mathcal{H},1} H t} X_0 \overset{t \rightarrow \infty}{\rightarrow} 0 $;
    \item $ Cov\left[ X_t \right] = \left( M_t - e^{-2 \ell_{\nu} \Sigma_{\mathcal{H},1} H t} \right) X_0^2 + \frac{\eta}{2 N} \left( \ell_{\nu}I_d  + \frac{\eta}{2 N}\ell_{\nu}^2 \Sigma_{\mathcal{H},2} \Sigma_{\mathcal{H},1}^{-1} H\right)^{-1} H^{-1} \Sigma_{\mathcal{H},1}^{-1} \left( I_d - M_t \right),$\\ which as $t\to\infty$ converges to $\frac{\eta}{2 N} \left( \ell_{\nu}I_d  + \frac{\eta}{2 N}\ell_{\nu}^2 \Sigma_{\mathcal{H},2} \Sigma_{\mathcal{H},1}^{-1} H\right)^{-1} H^{-1} \Sigma_{\mathcal{H},1}^{-1}$.
\end{enumerate}
\end{proposition}
\begin{proof}
    The proof mimics that of Prop. \ref{prop:DCSGD_StaDistr_App}.
\end{proof}

\section{Additional Related Works} \label{app:RelWorks}

In this section, we list some papers that derived or used SDEs to model optimizers. In particular, we focus on the aspect of empirically verifying the validity of such SDEs in the sense that they indeed track the respective optimizers. We divide these into three categories: Those that did not carry out any type of validation, those that did it on simple landscapes (quadratic functions et similia), and those that did small experiments on neural networks.

None of the following papers carried out any experimental validation of the approximating power of the SDEs they derived. Many of them did not even validate the insights derived from the SDEs: \citep{liu2018diffusion,hu2019global,bercher2020weak,zhu2020sharp,cui2020momentum,maulen2021continuous,wang2020asymptotic,lanconelli2022note,ayadi2021stochastic,soto2022sde,li2022uniform,wang2022two,bardi2022deep,chen2022approximation,kunin2023limiting,zhang2023stochastic,sun2023scheduling,li2023fast, gess2024stochastic,dambrine2024stochastic,maulen2024stochastic}.

The following ones carried out validation experiments on artificial landscapes, e.g. quadratic or quartic function, or easy regression tasks: \citep{li2017stochastic,li2019stochastic,zhou2020stochastic,an2020stochastic,fontaine2021convergence,gu2021adversarial,su2023accelerated,ankirchner2024comparison}.

The following papers carried out some experiments which include neural networks: \citep{paquette2021sgd,compagnoni2023sde}. In particular, they both simulate the SDEs with a numerical integrator and compare them with the respective optimizers: The first validates the SDE on a shallow MLP while the second does so on a shallow and a deep MLP. We also verify our SDEs on simple landscapes as well as on an MPL (on Breast Cancer) and, importantly, we verify our insights on ViTs (on MNIST) and ResNets (on CIFAR-10).

It would be great to extend the theoretical results to a more practical class of structural non-convex problems, e.g., under quasar convexity \citep{hardt2018gradient}, $\alpha$-$\beta$-condition \citep{islamov2024loss}, or Aiming condition \citep{liu2024aiming}.

\section{Numerical Integration of SDEs}
In this section, we only define the Euler-Maruyama Integration Method for SDEs: For a deeper introduction to SDEs and Itô Calculus, we refer the reader to Appendix B in \cite{compagnoni2025adaptive}. Let us consider an SDE of the form
\begin{equation*}
    dX_t = b(X_t,t)dt + \sigma(X_t,t) dW_t.
\end{equation*}

The simplest algorithm to provide a sample path $(\hat x_k)_{k\ge 0}$ for $X_t$, so that $X_{k\Delta t} \approxeq \hat{x}_k$ for some small $\Delta t$ and for all $k\Delta t\le M$ is called Euler-Maruyama (Algorithm~\ref{algo:EulerMaruryama_SDE}). For more details on this integration method and its approximation properties, the reader can check~\cite{mao2007stochastic}.
\begin{algorithm}
\caption{Euler-Maruyama Integration Method for SDEs}
    \label{algo:EulerMaruryama_SDE}
\begin{algorithmic}
    \INPUT{The drift $b$, the volatility $\sigma$, and the initial condition $x_0$}.
    \STATE Fix a stepsize $\Delta t$;
    \STATE Initialize $\hat x_0 = x_0$;
    \STATE $k=0$;
    \WHILE{$k \le \left\lfloor\frac{T}{\Delta t}\right\rfloor$}
        \STATE Sample some $d$-dimensional Gaussian noise $Z_k\sim\mathcal{N}(0,I_d)$;
        \STATE Compute $\hat x_{k+1} = \hat x_k + \Delta t \ b(\hat x_k,k \Delta t)+ \sqrt{\Delta t} \ \sigma(\hat x_k,k \Delta t) Z_k;$
        \STATE $k=k+1$;
    \ENDWHILE
    \OUTPUT The approximated sample path $(\hat x_k)_{0\le k\le\left\lfloor\frac{T}{\Delta t}\right\rfloor }$.
\end{algorithmic}
\end{algorithm}
%%%%%%%%%%%%%%%%%%%%%%%%%%%%%%%%%%%%%%%%%%%%%%%%%%%%%%%%%%%%%%%%%%%%%%%%%%%%

%%%%%%%%%%%%%%%%%%%%%%%%%%%%%%%%%%%%%%%%%%%%%%%%%%%%%%%%%%%%%%%%%%%%%%%%%%%%
\section{Experiments} \label{app:Exper}
%%%%%%%%%%%%%%%%%%%%%%%%%
In this section, we provide the modeling choices and instructions to replicate our experiments.

The code is implemented in Python 3 \citep{10.5555/1593511} mainly using Numpy \citep{harris2020array}, scikit-learn \citep{scikit-learn}, and JAX \citep{jax2018github}.

\subsection{SDE Validation (Figure \ref{fig:SDEs})}
In this subsection, we describe the experiments we run to produce Figure \ref{fig:SDEs}: The trajectories of the SDEs match those of the respective algorithms on average. Additionally, the SDEs and the algorithms move at the same speed.

\paragraph{DSGD - Rosenbrock}
This paragraph refers to the \textit{upper-left} of Figure \ref{fig:SDEs}. The loss function is the Rosenbrock function with parameters $a=1$ and $b=100$. We run DSGD for $10000$ epochs as we calculate the full gradient and inject it with Gaussian noise $Z \sim \mathcal{N}(0, \sigma^2 I_d)$ where $\sigma=100$. The learning rate is $\eta =0.001$ and $N=10$. Similarly, we integrate the DSGD SDE (Thm. \ref{thm:DSGD_Theorem}) with Euler-Maruyama  (Algorithm \ref{algo:EulerMaruryama_SDE}) with $\Delta t = \eta$. Results are averaged over $5000$. We plot the averaged trajectories and observe that they overlap to a great degree of agreement.

\paragraph{DCSGD - Embedded Saddle}
This paragraph refers to the \textit{upper-right} of Figure \ref{fig:SDEs}. We optimize the function $f(x) = \frac{x^{\top} H x}{2}  + \frac{1}{4}\lambda \sum_{i=1}^2 x_i^4 -  \frac{\xi}{3}\sum_{i=1}^2 x_i^3$ where $H = \diag(1,-2)$, $\lambda=1$, and $\xi = 1$. We run DCSGD with Rand-$k$ as $k=1$ for $1000$ epochs as we calculate the full gradient and inject it with Gaussian noise $Z \sim \mathcal{N}(0, \sigma^2 I_d)$ where $\sigma=10$. The learning rate is $\eta =0.1$ and $N=10$. Similarly, we integrate the DCSGD SDE (Thm. \ref{thm:DCSGD_Theorem}) with Euler-Maruyama  (Algorithm \ref{algo:EulerMaruryama_SDE}) with $\Delta t = \eta$. Results are averaged over $5000$. We plot the averaged trajectories and observe that they overlap to a great degree of agreement.

\paragraph{DSignSGD - Convex Quadratic Function}
This paragraph refers to the \textit{bottom-left} of Figure \ref{fig:SDEs}. We optimize the function $f(x) = \frac{x^{\top} H x}{2}$ where $H = \diag(5,5)$. We run DSignSGD for $3000$ epochs as we calculate the full gradient and inject it with Gaussian noise $Z \sim \mathcal{N}(0, \diag(\sigma_i^2))$ where $\sigma_i=0.01*(1+i)$, the learning rate is $\eta =0.001$ and $N=10$. Similarly, we integrate the DSignSGD SDE (Thm. \ref{thm:DSignSGD_Theorem}) with Euler-Maruyama  (Algorithm \ref{algo:EulerMaruryama_SDE}) with $\Delta t = \eta$. Results are averaged over $5000$. We plot the averaged trajectories and observe that they overlap to a great degree of agreement.

\paragraph{DNN on Breast Cancer Dataset \citep{Dua:2019}}
This paragraph refers to the \textit{bottom-right} of Figure \ref{fig:SDEs}. The DNN has $10$ dense layers with $20$ neurons each activated with a ReLu. We minimize the binary cross-entropy loss. We run DCSGD with Rand-$k$ and $k=1000$, $d=1502$, and for $30000$ epochs as we calculate the full gradient and inject it with Gaussian noise $Z \sim \mathcal{N}(0, \sigma^2 I_d)$ where $\sigma=0.0001$. The learning rate is $\eta =0.1$ and $N=3$. Similarly, we integrate the DCSGD SDE (Thm. \ref{thm:DCSGD_Theorem}) with Euler-Maruyama  (Algorithm \ref{algo:EulerMaruryama_SDE}) with $\Delta t = \eta$. Results are averaged over $3$ runs and the shaded areas are the average $\pm$ the standard deviation.

\subsection{DCSGD - Scaling Rules (Figure \ref{fig:DCSGD_ScalLaw})}

\paragraph{Transformer on MNIST \citep{deng2012mnist}}
This paragraph refers to the \textit{left} of Figure \ref{fig:DCSGD_ScalLaw}. The Architecture is a scaled-down version of \citep{dosovitskiy2020image}, where the hyperparameters are \textit{patch size}=$28$, \textit{out features}=$10$, \textit{width}=$48$, \textit{depth}=$3$, \textit{num heads}=$6$, and \textit{dim ffn}=$192$. We minimize the cross-entropy loss.
In this experiment, we run DSGD with some hyperparameters $(\eta, B, N)$ for $1000$ epochs. Then, we need to verify the scaling rules in Prop. \ref{prop:DCSGD_RecoverLaws_Main}, meaning that we run DCSGD with hyperparameters that follow the rules reported there and confirm that they indeed recover the performance of DSGD. Then, we also run DCSGD with combinations of hyperparameters that do not do so and indeed they do not recover the performance of DSGD. In all our experiments, we calculate the full gradient and inject it with Gaussian noise $Z \sim \mathcal{N}(0, \sigma^2 I_d)$ where $\sigma=0.01$ which corresponds to $B=1$. The learning rate is $\eta=0.01$ and the number of agents is $N=3$. DCSGD($\eta,B,\omega, N$) is with $\omega=1$ and indeed does not perform as DSGD($\eta,B,N$). DCSGD($\eta,B,\omega, (1 + \omega)N$) almost recovers the performance of DSGD($\eta,B,N$). The same with  DCSGD($\eta, B, \beta \omega, (1 + \beta \omega )N$),  DCSGD($\eta,(1+\omega) B, \omega, N$), and  DCSGD($\eta,(1+\beta \omega) B, \beta \omega, N$) for $\beta =2$.
On the contrary, neither DCSGD($\kappa \eta, B, \beta \omega, (1 + \beta \omega) N$) for $\kappa=3$ nor DCSGD($\eta, \delta B, \omega,  N$) for $\delta=1/3$ do so because they do not satisfy our scaling rules. See Figure \ref{fig:Boxplot_DCSGD} for a boxplot comparing the errors at the last iterate: Clearly, those hyperparameter combinations that do not follow our prescriptions behave much differently DSGD than those that do follow our rules. Results are averaged over $50$ runs.

\paragraph{ResNet on CIFAR-10 \citep{krizhevsky2009learning}}
This paragraph refers to the \textit{left} of Figure \ref{fig:DCSGD_ScalLaw}. The ResNet 
has a $(3,3,32)$ convolutional layer with stride $1$, followed by a ReLu activation, a second $(3,3,32)$ convolutional layer with stride $1$, followed by a residual connection from the first convolutional layer, then a $(2,2)$ max pool layer with stride $(2,2)$. Then the activations are flattened and passed through a dense layer that compresses them into $128$ dimensions, a final ReLu activation, and a final dense layer into the output dimension $10$. The output finally goes through a softmax as we minimize the cross-entropy loss. 
In this experiment, we run DSGD with some hyperparameters $(\eta, B, N)$. Then, we need to verify the scaling rules in Prop. \ref{prop:DCSGD_RecoverLaws_Main}, meaning that we run DCSGD with hyperparameters that follow the rules reported there and confirm that they indeed recover the performance of DSGD. Then, we also run it with a combination that does not do so and indeed it does not recover the performance of DSGD. In all our experiments, we calculate the full gradient and inject it with Gaussian noise $Z \sim \mathcal{N}(0, \sigma^2 I_d)$ where $\sigma=0.01$ which corresponds to $B=1$. The learning rate is $\eta=0.01$ and the number of agents is $N=3$. DCSGD($\eta,B,\omega, N$) is with $\omega=1$ and indeed does not perform as DSGD($\eta,B,N$). DCSGD($\eta,B,\omega, (1 + \omega)N$) almost recovers the performance of DSGD($\eta,B,N$). The same with  DCSGD($\eta, B, \beta \omega, (1 + \beta \omega )N$),  DCSGD($\eta,(1+\omega) B, \omega, N$), and  DCSGD($\eta,(1+\beta \omega) B, \beta \omega, N$) for $\beta =2$.
On the contrary, neither DCSGD($\kappa \eta, B, \beta \omega, (1 + \beta \omega) N$) for $\kappa=3$ nor DCSGD($\eta, \delta B, \omega,  N$) for $\delta=1/3$ do so because they do not satisfy our scaling rules. See Figure \ref{fig:Boxplot_DCSGD} for a boxplot comparing the errors at the last iterate: Clearly, those hyperparameter combinations that do not follow our prescriptions behave much differently DSGD than those that do follow our rules. Results are averaged over $10$ runs.

\begin{figure}%
\subfloat{{\includegraphics[width=0.50\linewidth]{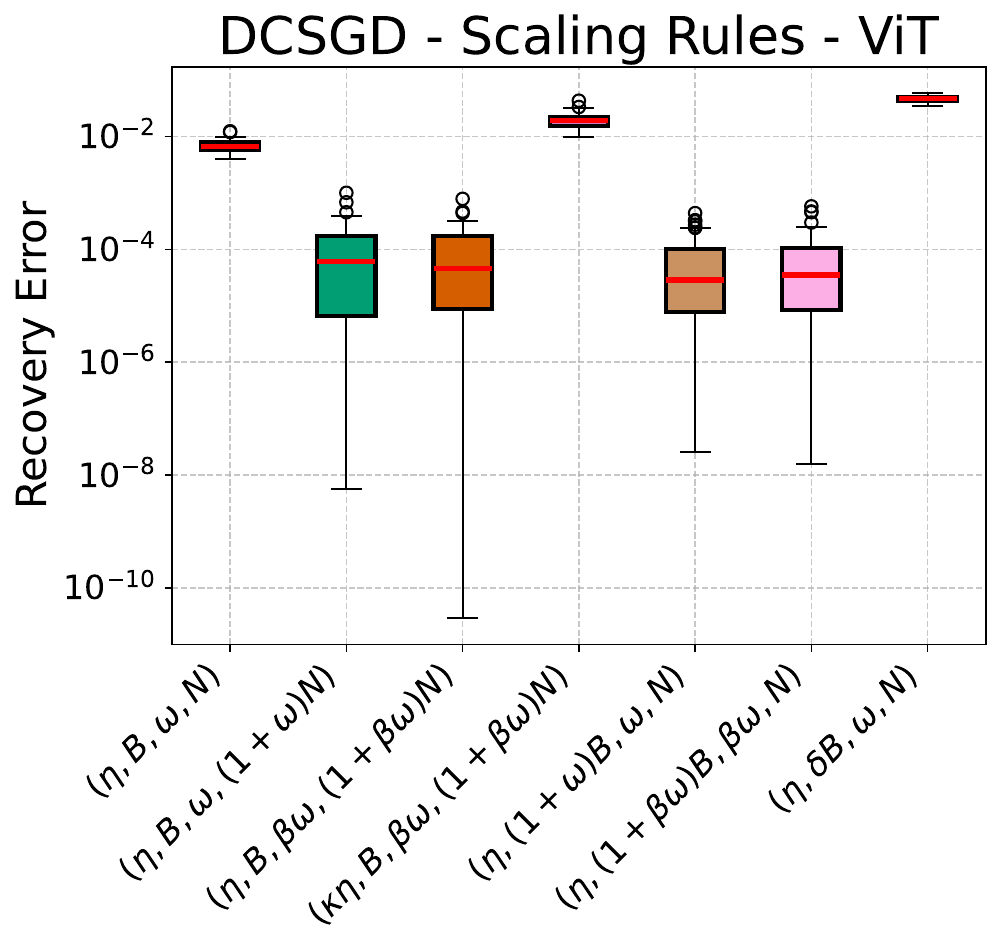} }}
    \subfloat{{\includegraphics[width=0.50\linewidth]{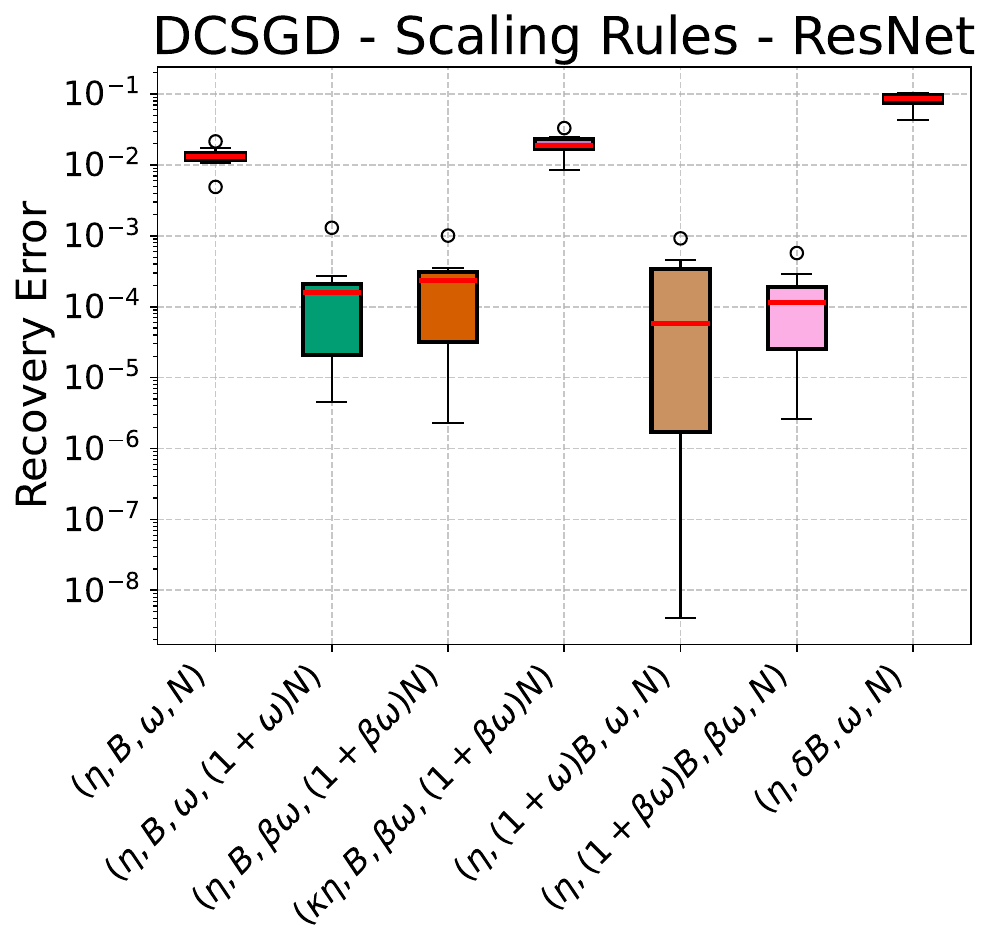} }} \\
    \caption{Box plot of the error between the last iterate of DSGD base run and the runs of DCSGD with the different combinations of hyperparameters: Those runs that follow our Scaling Rules achieve a much smaller error than those that do not.}
    \label{fig:Boxplot_DCSGD}%
\end{figure}

\subsection{DSignSGD - Bound and Linear Speedup (Figure \ref{fig:DSignSGD_Insights})}

\paragraph{Bound - Left}
In this paragraph, we describe how we validated the existence of the phases of DSignSGD as predicted in Thm. \ref{thm:DSignSGD_Convergence}.
We run DSignSGD with $\eta=0.001$ for $800$ epochs, $N=12$ as we optimize function is $f(x) = \frac{x^{\top}H x}{2}$ for $H = \diag(2)$, and inject Gaussian noise with covariance matrix $\Sigma = \sigma^2 I_d$ where $\sigma = 0.1$ on the full gradient. We plot the bounds as per Thm. \ref{thm:DSignSGD_Convergence} and confirm that they indeed match or bound the dynamics of the loss as prescribed. Results are averaged over $100$ runs.

\paragraph{Linear Speedup - Right}
In this paragraph, we describe how we validated the linear speedup of DSignSGD on the ViT described above. We run DSignSGD as we calculate the full gradient and inject it with Gaussian noise $Z \sim \mathcal{N}(0, \sigma^2 I_d)$ where $\sigma=1$, $\eta=0.01$ and $N \in \{1,2,4,8,16 \}$. Results are averaged over $3$ runs.

\subsection{DSignSGD - Scaling Rules (Figure \ref{fig:DSignSGD_ScalLaw})}

\paragraph{Transformer on MNIST \citep{deng2012mnist}}
This paragraph refers to the \textit{left} of Figure \ref{fig:DSignSGD_ScalLaw}. The ViT is the same as described above. In this experiment, we run DSignSGD with some hyperparameters $(\eta, B, N)$. Then, we need to verify the scaling rules in Prop. \ref{prop:DSignSGD_ScalingLaws_Main}, meaning that we run DSignSGD with hyperparameters that follow the rules reported there and confirm that they indeed preserve the performance. Then, we also run it with a combination that does not do so and indeed it does not preserve them. In all our experiments, we calculate the full gradient and inject it with Gaussian noise $Z \sim \mathcal{N}(0, \sigma^2 I_d)$ where $\sigma=0.2$ which corresponds to $B=1$. The learning rate is $\eta=0.01$ for $1000$ epochs and the number of agents is $N=4$. Since they follow our scaling rules, DSignSGD($\kappa \eta, \kappa^2 B, N$), DSignSGD($\kappa \eta, B, \kappa N$), and DSignSGD($ \eta, \kappa^2 B, N/\kappa$) with $\kappa=2$ indeed preserve the performance of DSignSGD($ \eta, B, N$), while DSignSGD($\kappa \eta, \kappa^2 B, N/\kappa$) does not.
See Figure \ref{fig:Boxplot_DSignSGD} for a boxplot comparing the errors at the last iterate: Clearly, those hyperparameter combinations that do not follow our prescriptions behave much differently than the base run than those that do follow our rules. Results are averaged over $5$ runs.

\paragraph{ResNet on CIFAR-10 \citep{krizhevsky2009learning}}
This paragraph refers to the \textit{left} of Figure \ref{fig:DSignSGD_ScalLaw}. The ResNet is the same as we described above. In this experiment, we run DSignSGD with some hyperparameters $(\eta, B, N)$. Then, we need to verify the scaling rules in Prop. \ref{prop:DSignSGD_ScalingLaws_Main}, meaning that we run DSignSGD with hyperparameters that follow the rules reported there and confirm that they indeed preserve the performance. Then, we also run it with a combination that does not do so and indeed it does not preserve them. In all our experiments, we calculate the full gradient and inject it with Gaussian noise $Z \sim \mathcal{N}(0, \sigma^2 I_d)$ where $\sigma=1$ which corresponds to $B=1$. The learning rate is $\eta=0.01$ for $2000$ epochs and the number of agents is $N=4$. Since they follow our scaling rules, DSignSGD($\kappa \eta, \kappa^2 B, N$), DSignSGD($\kappa \eta, B, \kappa N$), and DSignSGD($ \eta, \kappa^2 B, N/\kappa$) with $\kappa=2$ indeed preserve the performance of DSignSGD($ \eta, B, N$), while DSignSGD($\kappa \eta, \kappa^2 B, N/\kappa$) does not.
See Figure \ref{fig:Boxplot_DSignSGD} for a boxplot comparing the errors at the last iterate: Clearly, those hyperparameter combinations that do not follow our prescriptions behave much differently than the base run than those that do follow our rules. Results are averaged over $10$ runs.

\begin{figure}%
\subfloat{{\includegraphics[width=0.50\linewidth]{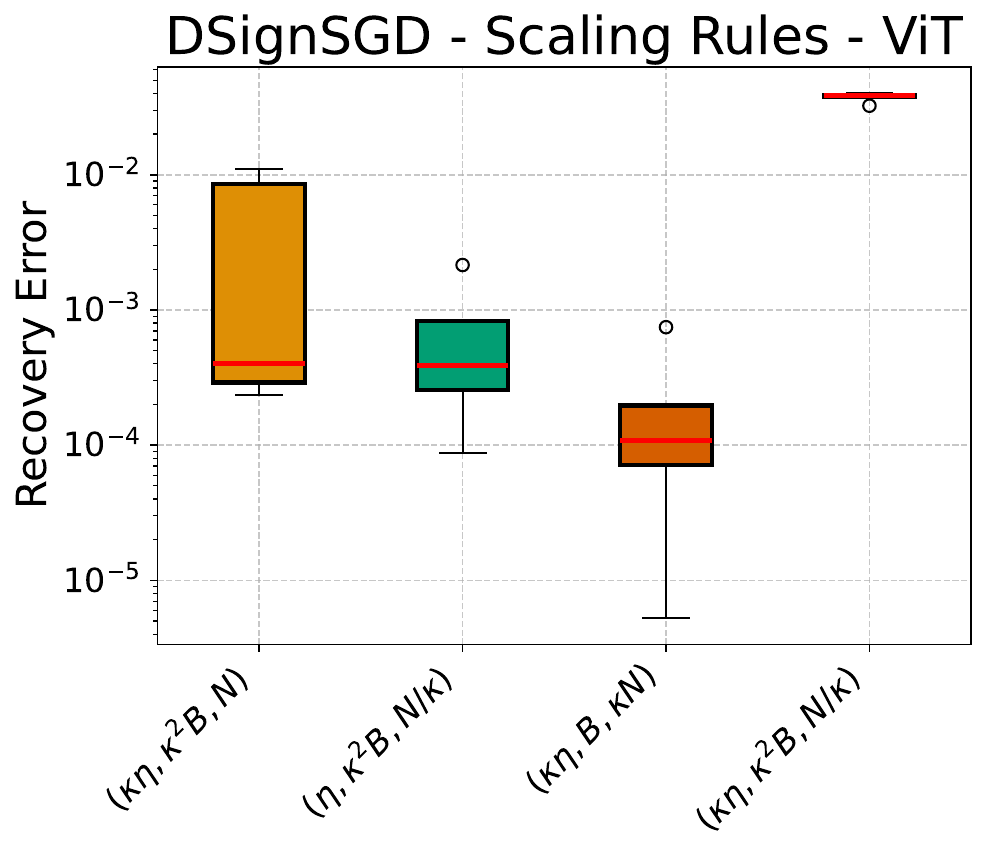} }}
    \subfloat{{\includegraphics[width=0.50\linewidth]{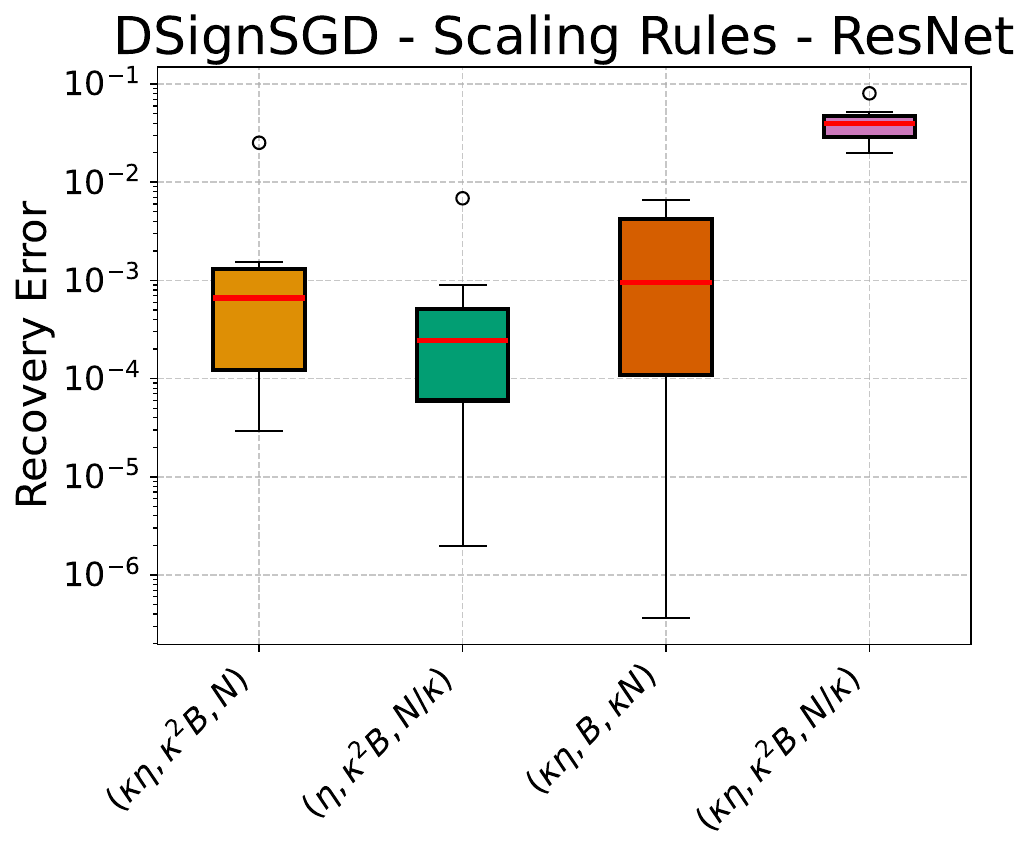} }} \\
    \caption{Box plot of the error between the last iterate of DSignSGD's base run and the runs with the different combinations of hyperparameters: Those runs that follow our Scaling Rules achieve a much smaller error than those that do not.}
    \label{fig:Boxplot_DSignSGD}%
\end{figure}

\subsection{Heavy Tailed and Large Noise (Figure \ref{fig:ComparisonLargeNoise_ViT})}
The ViT is the same as above for each sub-figure.

\paragraph{DCSGD - Heavy-Tailed Noise}
We train the ViT with DCSGD with Rand-$k$ where $k=100000$ out of $d=133930$, $\eta=0.01$ for $1000$ epochs, and as we inject noise distributed as a Student's t with scale $\Sigma = \sigma^2 I_d$ and $\nu \in \{1,2,3,8,64,\infty\}$, and $N=3$. Even if the scale is small ($\sigma^2=10^{-8}$): 1)  When $\nu=1$, the optimizer diverges; 2) When $\nu=2$, the loss is non-stationary; 3) The larger $\nu$, the more stable and optimal the loss. This confirms that indeed DCSGD cannot handle Heavy-Tailed noise. Results are averaged over $3$ runs.

\paragraph{DCSGD - Large Noise}
We train the ViT with DCSGD with Rand-$k$ where $k=100000$ out of $d=133930$, $\eta=0.01$ for $5000$ epochs, and as we inject noise distributed as a Gaussian with covariance matrix $\sigma^2 \in \{10^{-8}, 10^{-6}, 10^{-4}, 10^{-2}, 10^{-1}, 10^{0}\}$, and $N=3$. As the variance increases, the optimizer diverges more and more: This confirms that indeed DCSGD cannot handle large noise as its loss level scales quadratically in the noise level. Results are averaged over $3$ runs.

\paragraph{DSignSGD - Heavy-Tailed Noise}
We train the ViT with DSignSGD as we inject noise distributed as a Student's t with scale $\Sigma = \sigma^2 I_d$ and $\nu \in \{1,2,3,8,64,\infty \}$, and $N=3$, $\eta =0.01$ for $1000$ epochs. Even if the scale is large $(\sigma^2=1)$ and the noise is of unbounded expected value, DSignSGD never diverges. Of course, fatter tails imply larger loss: This confirms that indeed DSignSGD can handle Heavy-Tailed noise. Results are averaged over $3$ runs.

\paragraph{DSignSGD - Large Noise}
We train the ViT with DSignSGD as we inject noise distributed as a Gaussian with covariance matrix $\Sigma = \sigma^2 I_d$ and $\sigma^2 \in \{10^{-8}, 10^{-6}, 10^{-4}, 10^{-2}, 10^{-1}, 10^{0} \}$, and $N=3$, $\eta =0.01$ for $8000$ epochs. As the variance increases, the optimizer never diverges: This confirms that indeed DSignSGD can handle large noise as its loss level scales linearly in the noise level. Results are averaged over $3$ runs.

\subsection{DCSGD - Divergence, Bound, and Linear Speedup (Figure \ref{fig:DCSGD_Insights_App})}

\paragraph{Divergence}
We optimize the function $f(x) = \frac{x^{\top} H x}{2}$ where $H = \diag(100 I_{128})$. We run DCSGD with Rand-$k$ for $25$ epochs as we calculate the full gradient and inject it with Gaussian noise $Z \sim \mathcal{N}(0, \sigma^2 I_d)$ where $\sigma = 0.1$. The learning rate is $\eta =0.01$. As we decrease $k \in \{128,64,32,16,8,4,2,1 \}$, we see that the convergence slows down and reaches a larger and larger asymptotic loss value, up to diverging. Results are averaged over $5000$ runs.

\paragraph{Linear Speedup}
In this paragraph, we describe how we validated the linear speedup of DCSGD on the same ViT as above. We run DCSGD with Rand-$k$ with $k=100000$ as we calculate the full gradient and inject it with Gaussian noise $Z \sim \mathcal{N}(0, \sigma^2 I_d)$ where $\sigma=0.01$, $\eta=0.01$ and $N \in \{1,2,4,8,16 \}$. Averaged over $3$ runs.

\paragraph{Bound}
In this paragraph, we describe how we validated the bound DCSGD as predicted in Thm. \ref{thm:DCSGD_Convergence}.
We run DCSGD with Rand-$k$ for $2000$ epochs as we calculate the full gradient and inject it with Gaussian noise $Z \sim \mathcal{N}(0, \sigma^2 I_d)$ where $\sigma = 0.1$, $\eta=0.01$, $k=2$ and $N=12$ as we optimize function is $f(x) = \frac{x^{\top}H x}{2}$ for $H = I_{100}$. We plot the bounds as per Thm. \ref{thm:DCSGD_Convergence} and confirm that they indeed match loss as prescribed. Additionally, we also verify the Scaling Rules as per Prop. \ref{prop:DCSGD_RecoverLaws_Main}.

\begin{figure*}%
    \hspace{-0.1cm}
    \subfloat{{\includegraphics[width=0.245\linewidth]{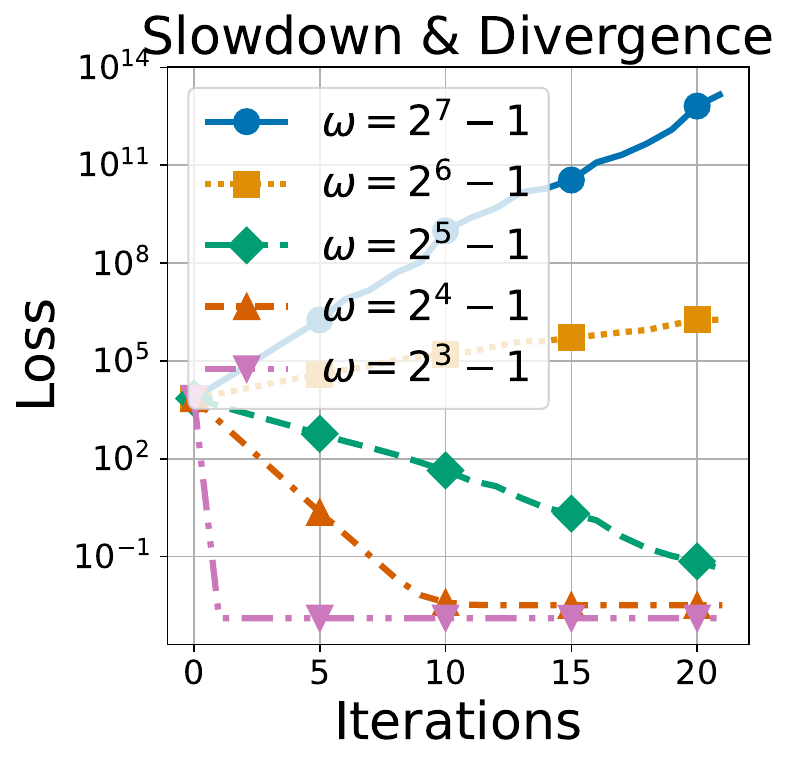} }}%
    \subfloat{{\includegraphics[width=0.27\linewidth]{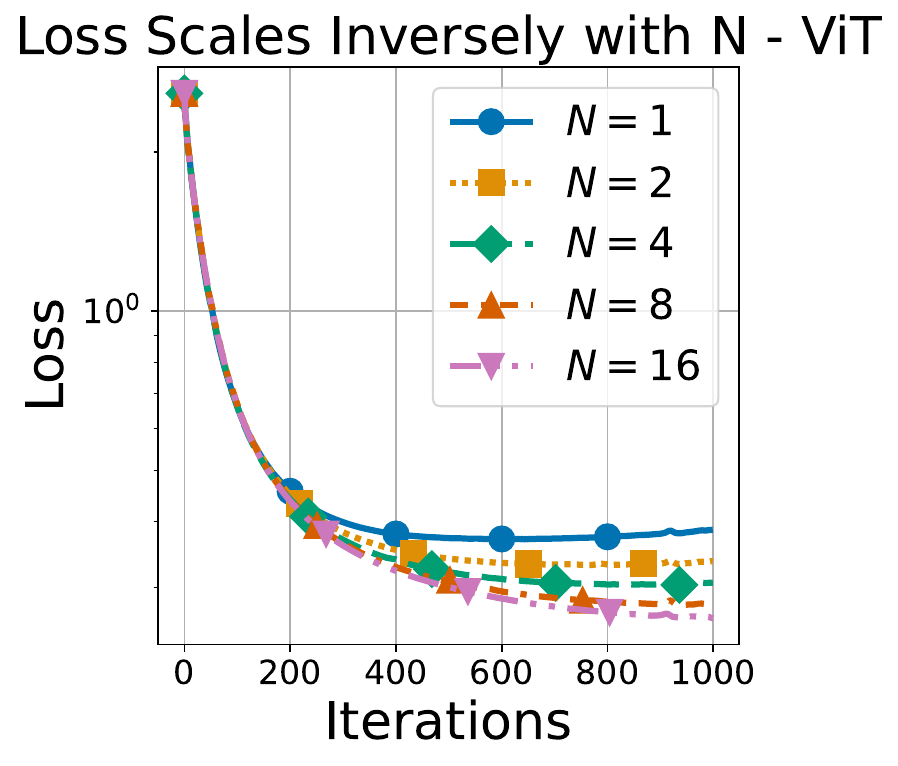} }}
    \subfloat{{\includegraphics[width=0.46\linewidth]{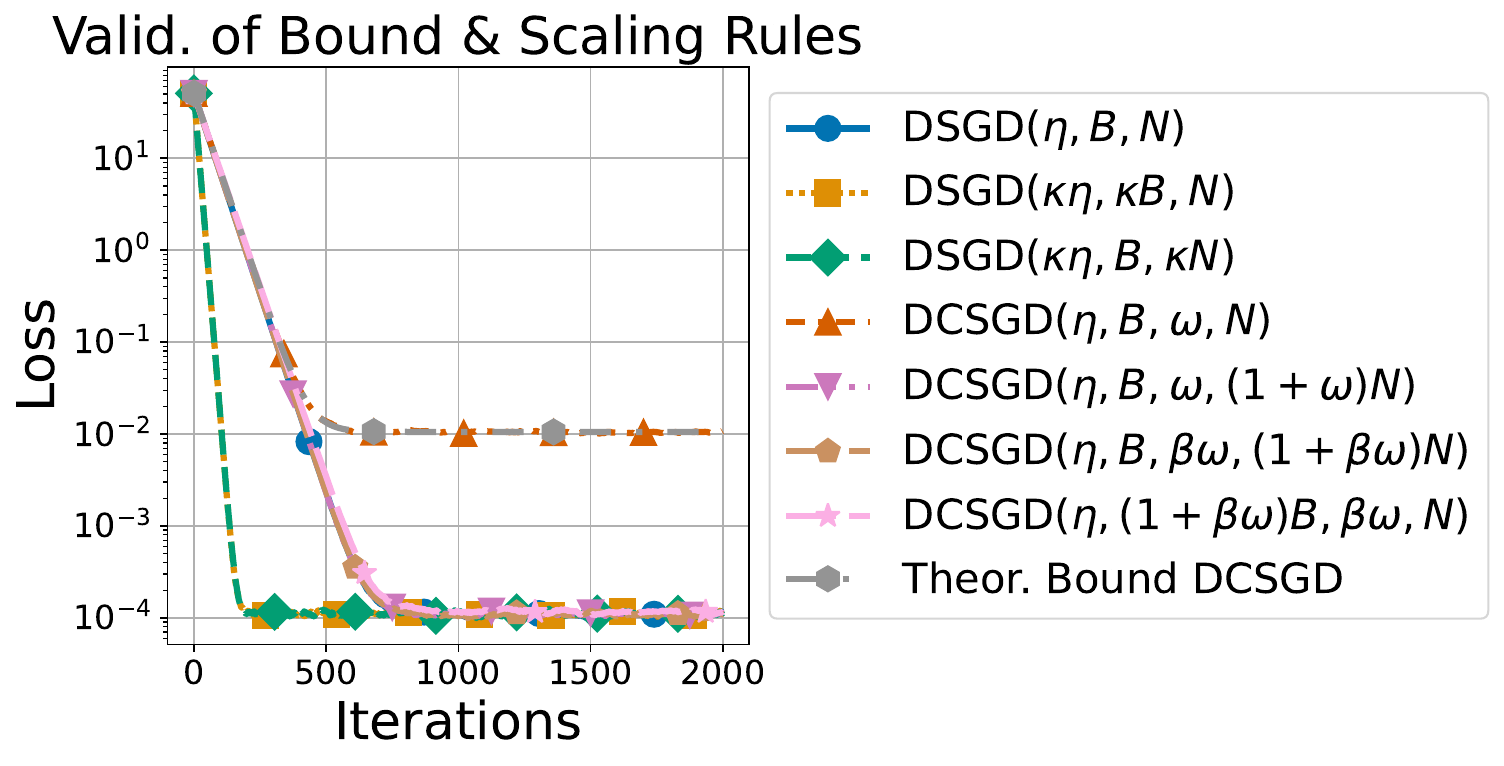} }}
    \caption{As per Thm.~\ref{thm:DCSGD_Convergence}: More compression implies slowdown, up to possible divergence (Left); As per Thm.~\ref{thm:DCSGD_Convergence}, more agents imply more optimality (Center); Validation of Bound and Scaling Rules: i) The bound derived in Thm.~\ref{thm:DCSGD_Convergence} matches the experimental loss of DCSGD$(\eta, B, \omega, N)$; ii) Consistently with Prop.~\ref{prop:DSGD_ScalingLaws}, DSGD$(\kappa\eta, \kappa B, N)$ recovers the asymptotic behavior of DSGD$(\eta, B, N)$; iii) DCSGD run with hyperparameters that follow Prop.~\ref{prop:DCSGD_RecoverLaws_Main} do recover the asymptotic performance of DSGD; iv) DCSGD$(\kappa \eta, \kappa B, \omega, N)$ fails to do so as it does not follow them (Right);
    }%
    \label{fig:DCSGD_Insights_App}%
\end{figure*}

\subsection{DCSGD \& DSignSGD - Stationary Distributions (Figure \ref{fig:StatDistr})}
On the left of Figure \ref{fig:StatDistr}, we validate the Stationary Distribution of DCSGD run with Rand-$k$ while on the right we do the same for DSignSGD.

\paragraph{DCSGD}
We optimize the function $f(x) = \frac{x^{\top} H x}{2}$ where $H = \diag(2,1,1,1,1,1,1,1,1,1)$. We run DCSGD with Rand-$k$ for $1000$ epochs as we calculate the full gradient and inject it with Gaussian noise $Z \sim \mathcal{N}(0, \sigma^2 I_d)$ where $\sigma = 0.1$. The learning rate is $\eta =0.01$,  $k=3$. We plot the evolution of the average variances with the theoretical predictions of Prop. \ref{prop:DCSGD_StaDistr_App}: Results are averaged over $50000$. The experimental moments match the theoretical predictions.

\paragraph{DSignSGD}
We optimize the function $f(x) = \frac{x^{\top} H x}{2}$ where $H = \diag(2,1,1,1,1,1,1,1,1,1)$. We run DSignSGD for $10000$ epochs as we calculate the full gradient and inject it with Gaussian noise $Z \sim \mathcal{N}(0, \sigma^2 I_d)$ where $\sigma = 0.1$. The learning rate is $\eta =0.001$. We plot the evolution of the average variances with the theoretical predictions of Prop. \ref{prop:HSignSGD_StaDistr_App}: Results are averaged over $5000$. The experimental moments match the theoretical predictions.

\begin{figure}%
\subfloat{{\includegraphics[width=0.50\linewidth]{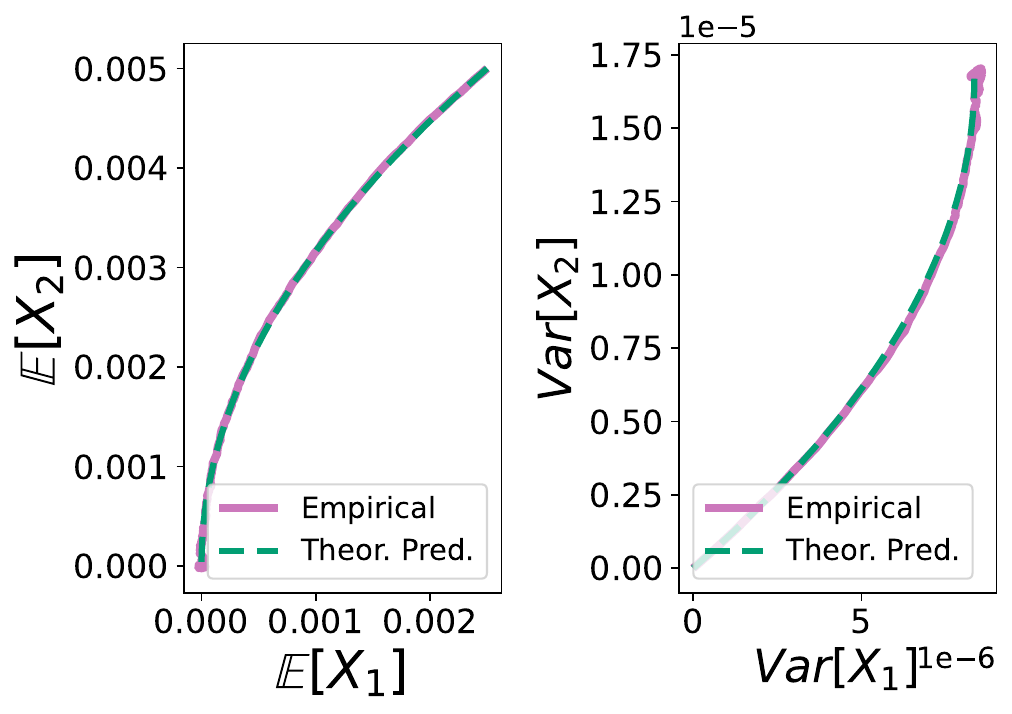} }}
    \subfloat{{\includegraphics[width=0.50\linewidth]{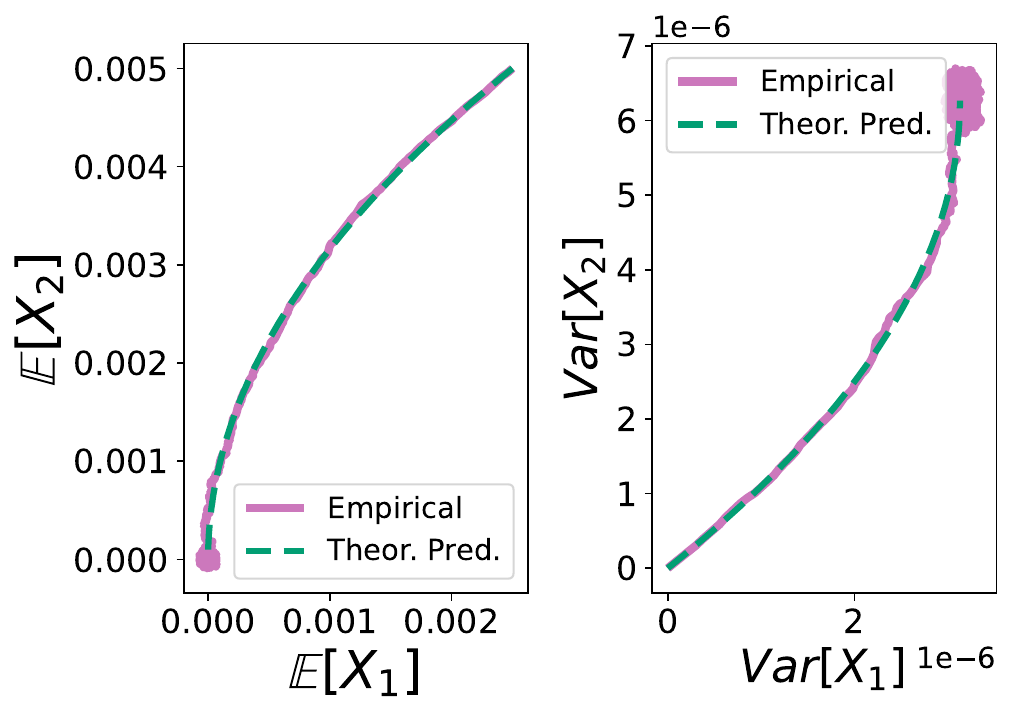} }}
    \caption{Verification of the Stationary Distribution of DCSGD and DSignSGD on a convex quadratic function.}
    \label{fig:StatDistr}%
\end{figure}

\subsection{Top-k and its Modification - Resilience to Heavy-Tailed Noise (Figure \ref{fig:Topk})}

To produce Figure \ref{fig:Topk}, we train the ResNet above on CIFAR-10. As above, we inject heavy-tailed noise onto the full gradients and observe that Top-$k$ cannot handle such noise. However, using Top-$k$ on top of Normalized SGD seems to mitigate this issue. Therefore, we confirm that sign compression is not the only one that can handle heavy-tailed noise and that there is room to develop alternative optimizers.

\begin{figure}%
\subfloat{{\includegraphics[width=0.50\linewidth]{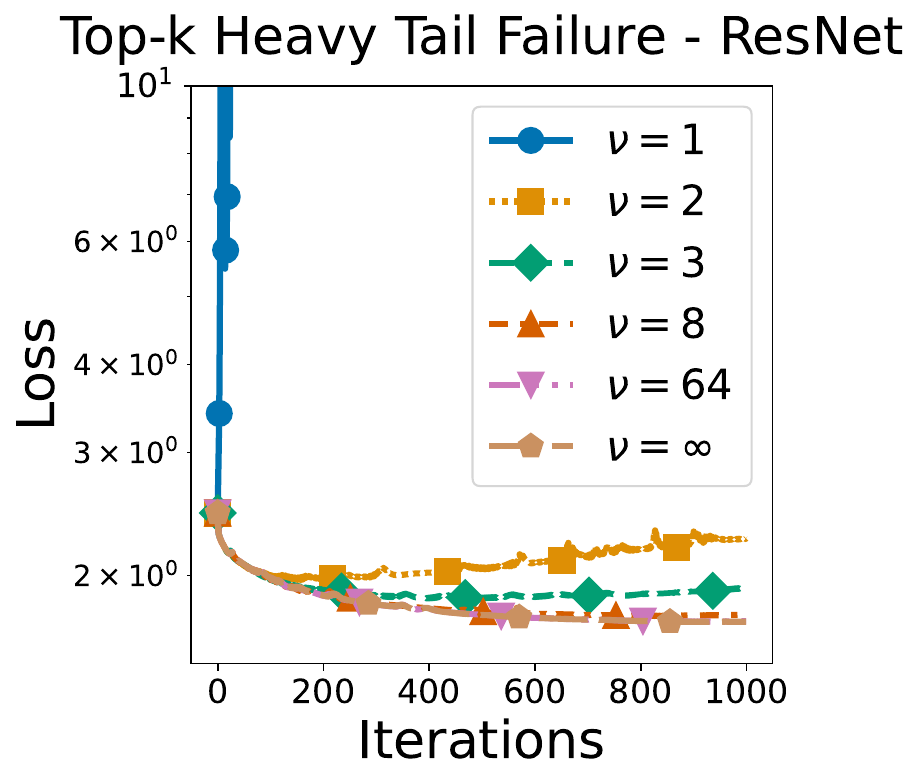} }}
    \subfloat{{\includegraphics[width=0.50\linewidth]{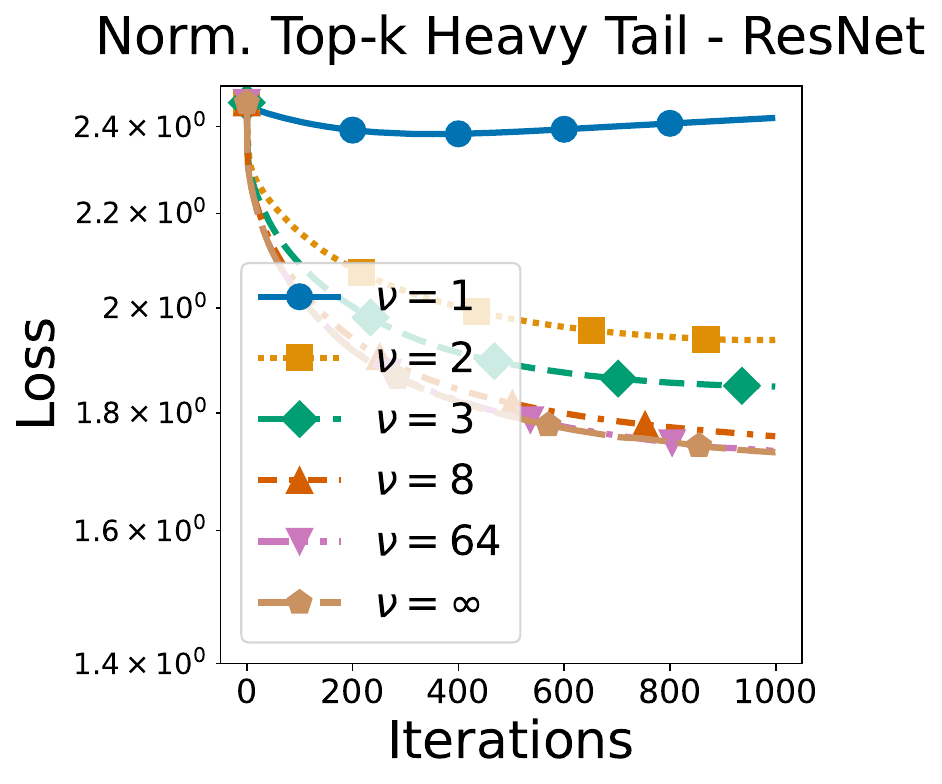} }}
    \caption{On the left, Top-$k$ fails at handling increasingly heavy-tailed noise, while on the right we see that combining Top-$k$ with Normalized SGD is promising.}
    \label{fig:Topk}%
\end{figure}

\subsection{Heavy Tailed and Large Noise (Figure \ref{fig:ComparisonLargeNoise_Quadratics})}

\paragraph{DCSGD - Large Noise - Top Left}
 We optimize the function $f(x) = \frac{x^{\top} H x}{2}$ where $H = I_{10}$ with DCSGD with Rand-$k$ where $k=1$, $\eta=0.01$, and as we inject noise distributed as a Gaussian with covariance matrix $\Sigma = \sigma^2 I_d$ and $\sigma^2 \in \{10^{-4},10^{-2},10^{0},10^{2},10^{4}\}$, and $N=3$. As the variance increases, the optimizer diverges more and more: This confirms that indeed DCSGD cannot handle large noise as its loss level scales quadratically in the noise level. The asymptotic loss level matches that predicted in Thm. \ref{thm:DCSGD_Convergence}. Results are averaged over $100$ runs.

\paragraph{DSignSGD - Large Noise - Top Right}
 We optimize the function $f(x) = \frac{x^{\top} H x}{2}$ where $H = I_{10}$ with DSignSGD as we inject noise distributed as a Gaussian with covariance matrix $\sigma^2 \in \{10^{-4},10^{-2},10^{0},10^{2},10^{4}\}$, $\eta = 0.001$, and $N=3$. As the variance increases, the optimizer never diverges: This confirms that indeed DSignSGD can handle large noise as its loss level scales linearly in the noise level. The asymptotic loss level matches that predicted in Thm. \ref{thm:DSignSGD_Convergence}. Results are averaged over $100$ runs.

\paragraph{DCSGD - Heavy-Tailed Noise - Bottom Left}
 We optimize the function $f(x) = \frac{x^{\top} H x}{2}$ where $H = I_{10}$ with DCSGD with Rand-$k$ where $k=1$, $\eta=0.01$, and as we inject noise distributed as a Gaussian with covariance matrix $\Sigma = \sigma^2 I_d$ and $\sigma=0.1$, $\nu \in \{1,2,3,8,64,\infty \}$, and $N=3$. Even if the scale is small: 1)  When $\nu=1$, the optimizer diverges; 2) When $\nu=2$, the loss is non-stationary; 3) The larger $\nu$, the more stable and optimal the loss. This confirms that indeed DCSGD cannot handle Heavy-Tailed noise. Results are averaged over $100$ runs.

\paragraph{DSignSGD - Heavy-Tailed Noise - Bottom Right}
  We optimize the function $f(x) = \frac{x^{\top} H x}{2}$ where $H = I_{10}$ with DSignSGD as we inject noise distributed as a Student's t with scale $\Sigma = \sigma^2 I_d$ and $\nu \in \{1,2,3,8,64,\infty \}$, and $N=3$. Even if the scale is large ($\sigma=1)$ noise is of unbounded expected value, DSignSGD never diverges. Of course, fatter tails imply larger loss: This confirms that indeed DSignSGD can handle Heavy-Tailed noise. Results are averaged over $100$ runs.

\begin{figure}[H]%
    \centering
    \subfloat{{\includegraphics[width=0.49\linewidth]{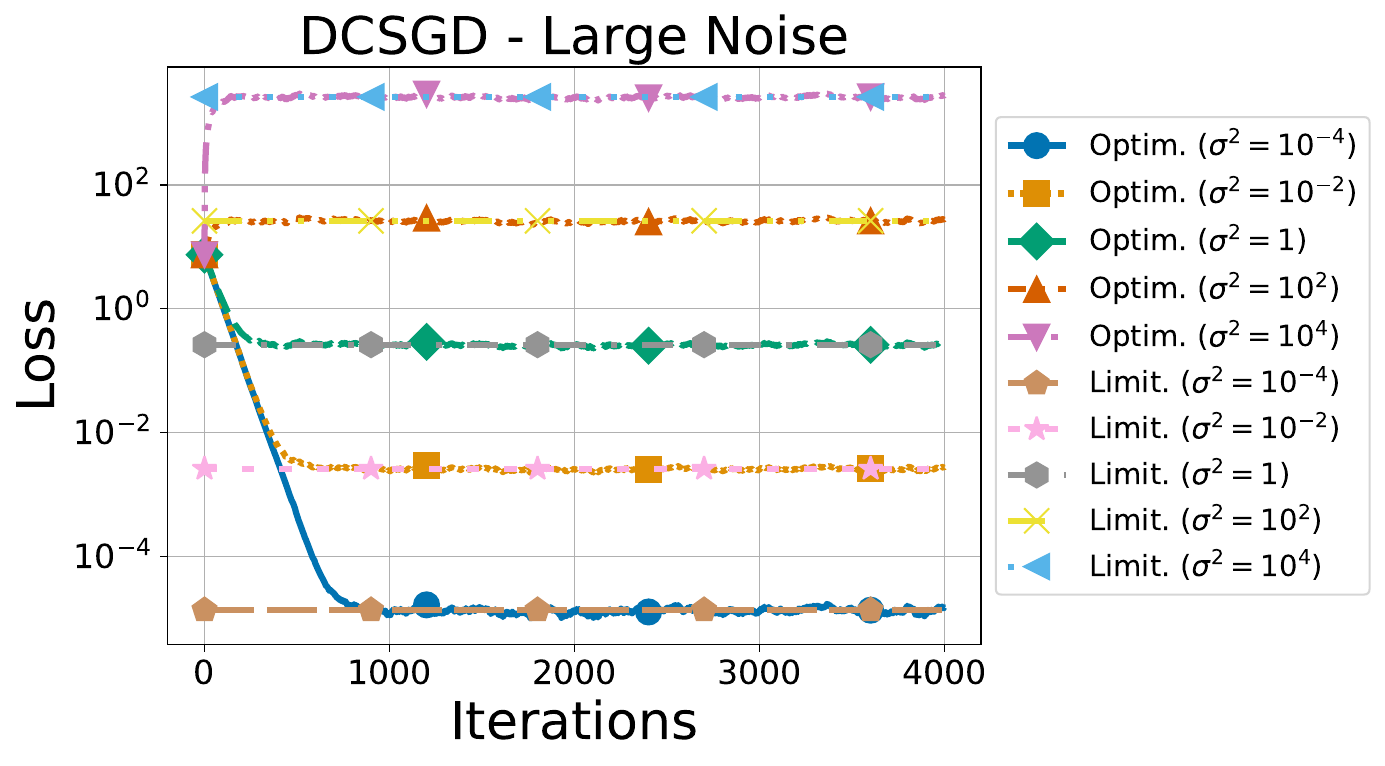} }}%
    \subfloat{{\includegraphics[width=0.49\linewidth]{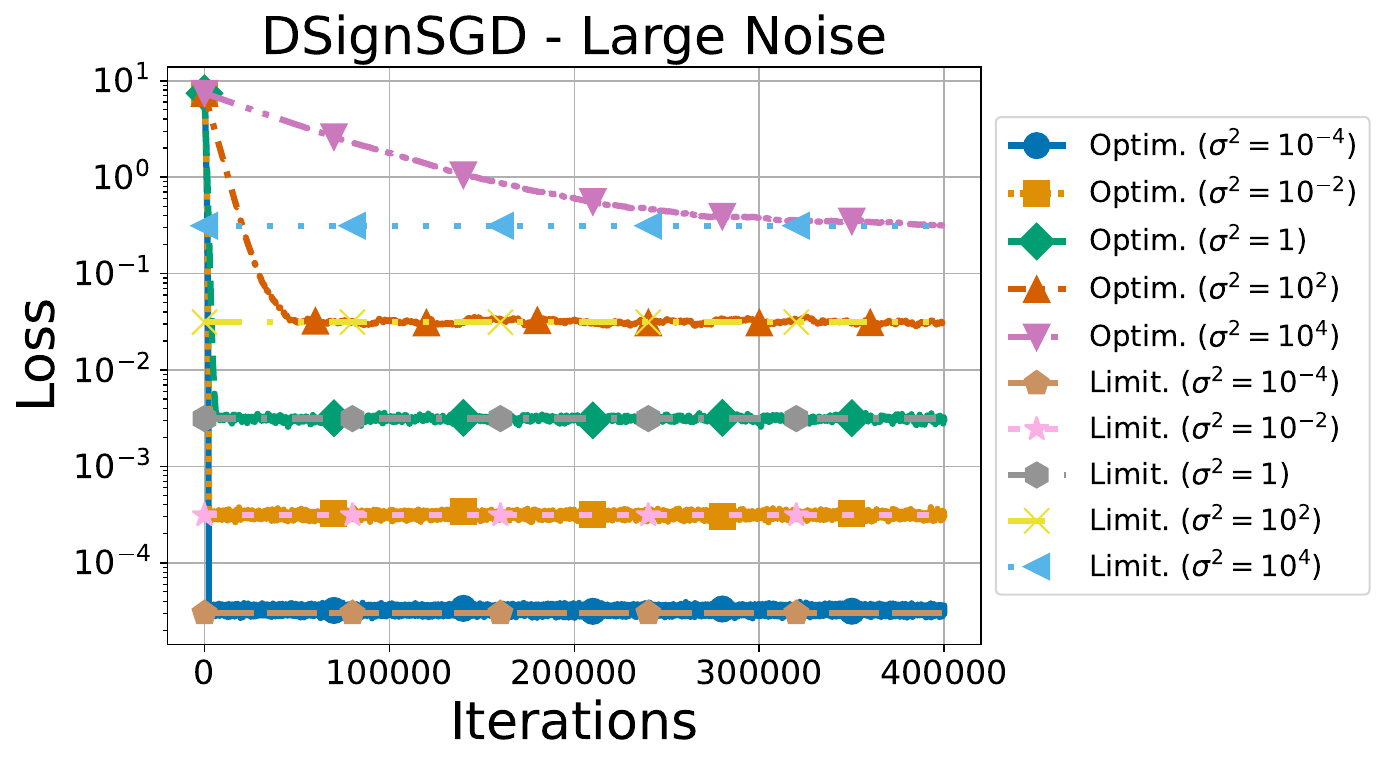} }} \\
    \subfloat{{\includegraphics[width=0.49\linewidth]{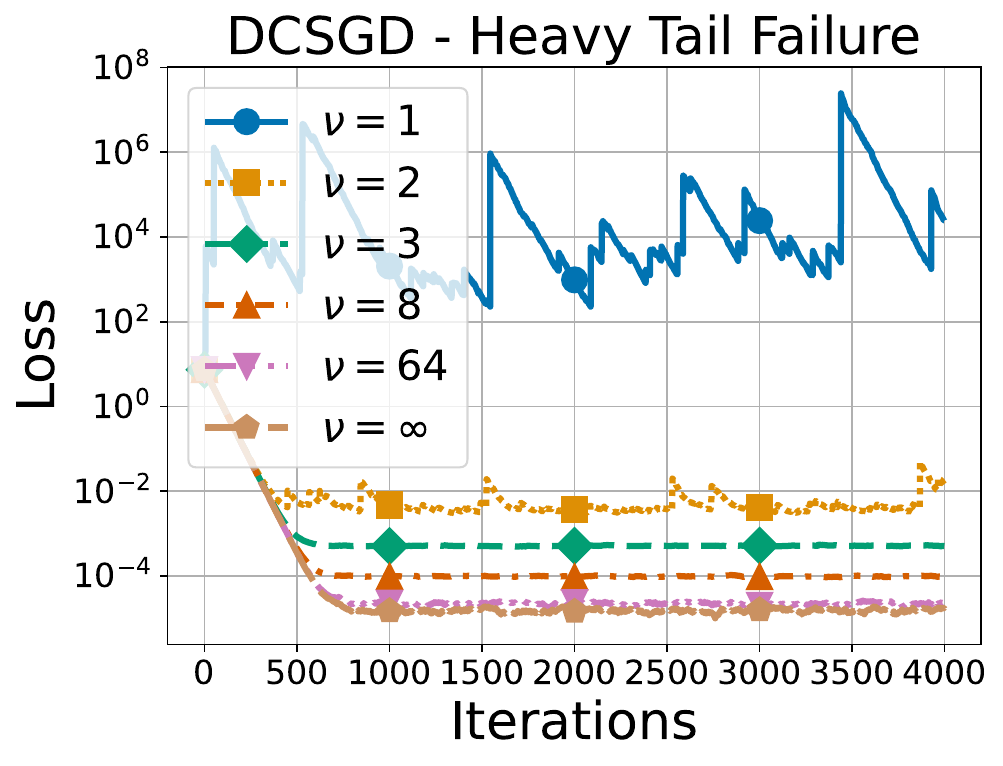} }}
    \subfloat{{\includegraphics[width=0.49\linewidth]{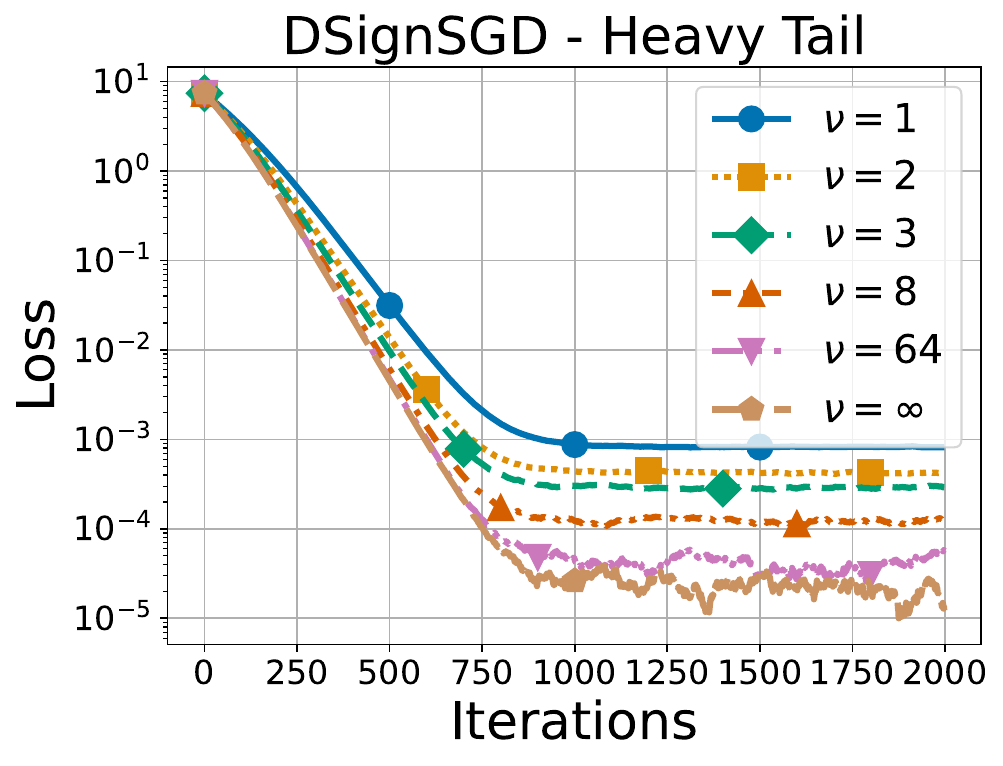} }}
    \caption{DCSGD cannot handle large noise as its asymptotic loss level scales quadratically in the noise level (Top Left); on the contrary,  DSignSGD can as its level scales linearly in the noise level (Top Right); DCSGD cannot handle heavy-tailed noise (Bottom Left) while DSignSGD can (Bottom Right).}%
    \label{fig:ComparisonLargeNoise_Quadratics}%
\end{figure}

\section{Scaling Rule Validation - GPT2}\label{ScalingRulesGPT2}

This experiment aims to validate some of the scaling rules derived for DCSGD (see Table \ref{tab:DCSGD_ScalingLaws} related to Proposition \ref{prop:DCSGD_RecoverLaws_Main}) and DSignSGD (see Proposition \ref{prop:DSignSGD_ScalingLaws_Main}) using a GPT-2-like model. To be precise, in these experiments, we fix a \textit{base run} optimizer: Optimizer$(\mathcal{H})$, where $\mathcal{H}$ is a configuration of hyperparameters (e.g., the learning rate $\eta$, the batch size $B$, or the number of agents $N$, i.e., GPUs). Then, we run the same optimizer with other hyperparameter configurations $(\tilde{\mathcal{H}})$. We verify that hyperparameter configurations $(\tilde{\mathcal{H}})$ that satisfy the functional relationships prescribed by our propositions achieve a performance much closer to the base run $(\mathcal{H})$ than those configurations that violate such prescriptions. This demonstrates that when adjustments to hyperparameters are necessary to accommodate new scenarios, one can follow our scaling rules to preserve the performance of DSignSGD/DCSGD without needing to repeat the fine-tuning process. For example, one might desire larger batch sizes to fully utilize newly available larger GPUs, or face a reduction in available GPUs due to budget cuts. We highlight that, in general, scaling rules are not meant to be exact prescriptions, but rather to give a principled approach to reduce the hyperparameter search space.

\subsection{Model Architecture and Dataset} 

The model architecture is provided in the popular GitHub repository nanoGPT by Andrej Karpathy: All details can, of course, be found on the repository --- We used the smallest configuration with 124M parameters. Regarding the dataset, we train our models on the FineWeb-Edu dataset. To do so, we minimize the \textit{Cross-entropy} loss for $10,000$ iterations. Note that we use no learning rate schedulers, as we do not aim for optimal performance but rather for clear and fair experimental validation of our theoretical insights. We encourage future work to explore the validity of our theory in larger and more realistic settings.

\subsubsection{DSignSGD - (Figure \ref{fig:DSignSGD_GPT})}

In this experiment, we fix the ``\textbf{base run}'' optimizer DSignSGD$(\eta, B, N)$ by selecting $\eta=0.001$, $B=4$, and $N=4$. We selected $6$ different hyperparameter configurations: $3$ that satisfy our scaling rules and $3$ that do not. We ran each configuration $5$ times and computed the average absolute percentage error of the last $100$ iterations with respect to the ``\textbf{base run}''. Figure \ref{fig:DSignSGD_GPT} shows the boxplots of these errors, and one can see that configurations that satisfy our scaling rules (marked in green in the figure) achieve a significantly lower error compared to those that do not follow them (marked in red).

\begin{figure}[H]
    \centering
    \includegraphics[width=0.49\linewidth]{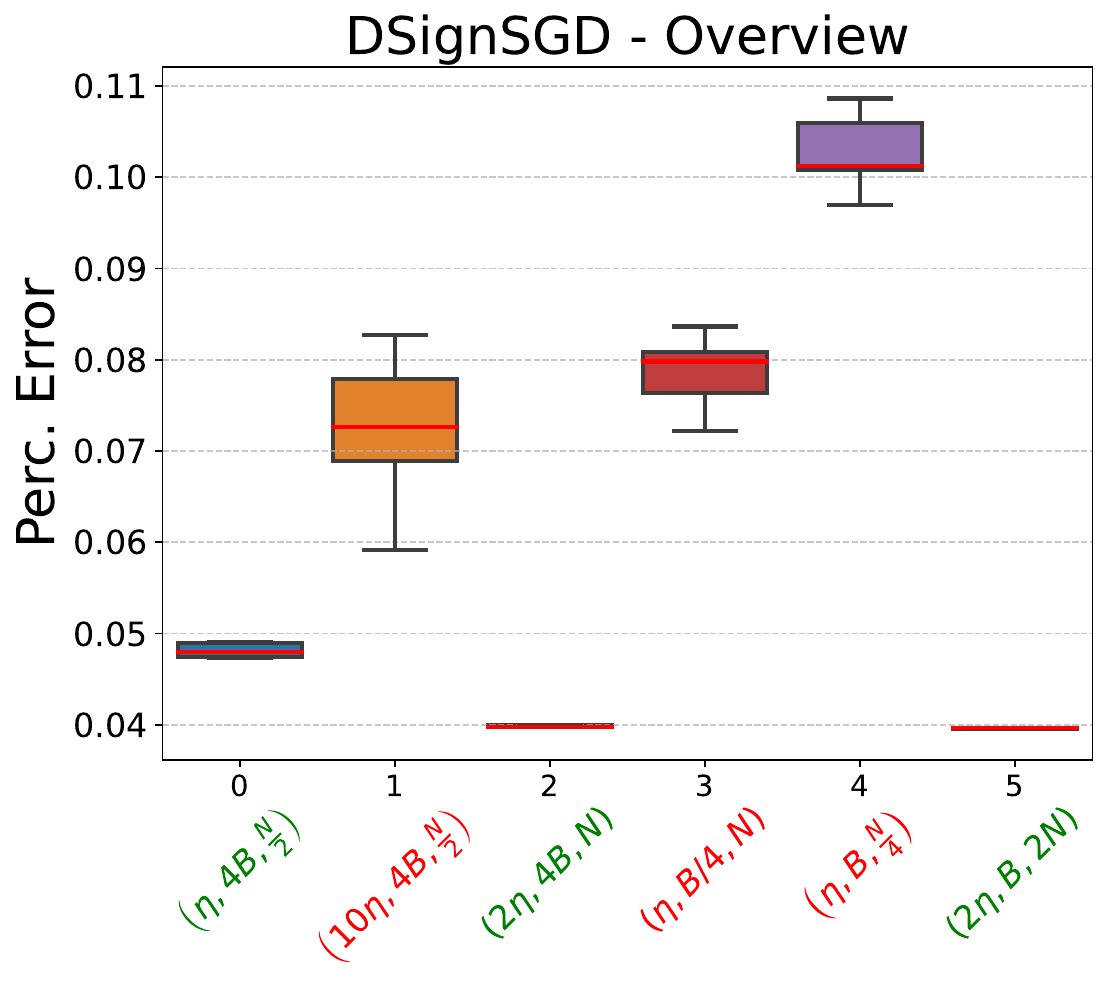}
    \caption{Boxplots of errors: Validation of Scaling Rules for DSignSGD on a 124M GPT2 model.}
    \label{fig:DSignSGD_GPT}
\end{figure}

\subsubsection{DCSGD - (Figure \ref{fig:DCSGD_GPT})}

In this experiment, we replicate the effort above by optimizing the model with DSGD$(\eta, B, N)$ as a ``\textbf{base run}'' (i.e., no compression). We select $\eta=0.1$ and $B=N=1$. Then, we run DCSGD with the Rand-$k$ compressor; in this case, the compression factor $\omega$ is equal to $\frac{d}{k}$, where $d$ is the total number of trainable parameters and $k$ is the number of parameters that at each iteration are randomly selected to be trained --- The remaining $d-k$ are left unchanged for that iteration. We simulated $12$ different configurations and computed the average absolute percentage error of the last $100$ iterations with respect to the ``\textbf{base run}''. Since Table \ref{tab:DCSGD_ScalingLaws} contains many more rules than DSignSGD, we opt for a different style of visualization that showcases the validity of multiple rules simultaneously. Since we had to simulate many more configurations, we only ran each configuration $3$ times.

Left of Figure \ref{fig:DCSGD_GPT}: In this figure, we show that: 1) For fixed ($\eta, B, \omega$), increasing the number of agents $N$ helps mitigate the performance loss due to compression (Rule 1 in green); 2) For fixed ($\eta, N, \omega$), increasing the batch size $B$ helps mitigate the performance loss due to compression (Rule 2 in orange); 3) Combining these two rules is even better (Rule 3).

Right of Figure \ref{fig:DCSGD_GPT}: In this figure, we ran DCSGD while doubling the learning rate ($\eta=0.2)$, thereby combining the effects of compression and increased learning rate on performance. Once again, we observe that increasing compression leads to worse performance, even more so when the learning rate has doubled. However, in accordance with our scaling rules, for any compression level, increasing $N$ and $B$ helps mitigate the performance loss.

\begin{figure}[H]
    \centering
    \subfloat{{\includegraphics[width=0.49\linewidth]{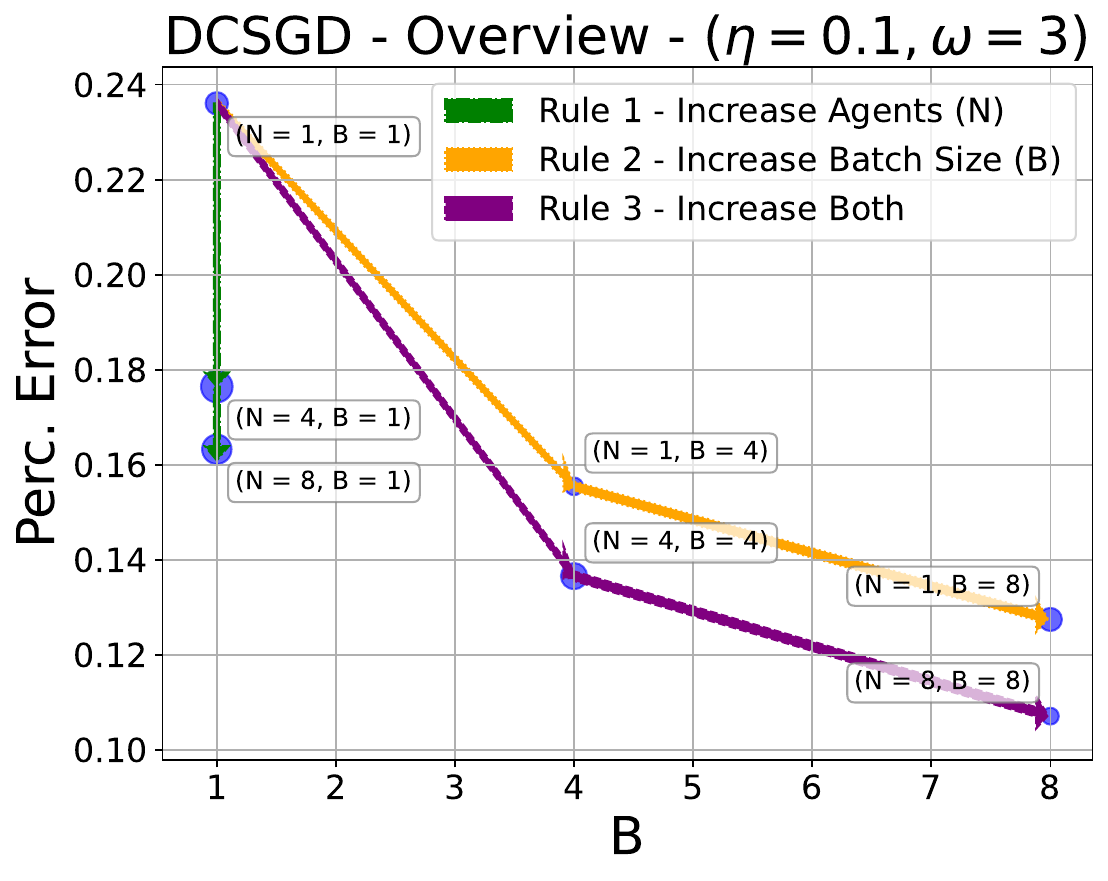} }}
    \subfloat{{\includegraphics[width=0.49\linewidth]{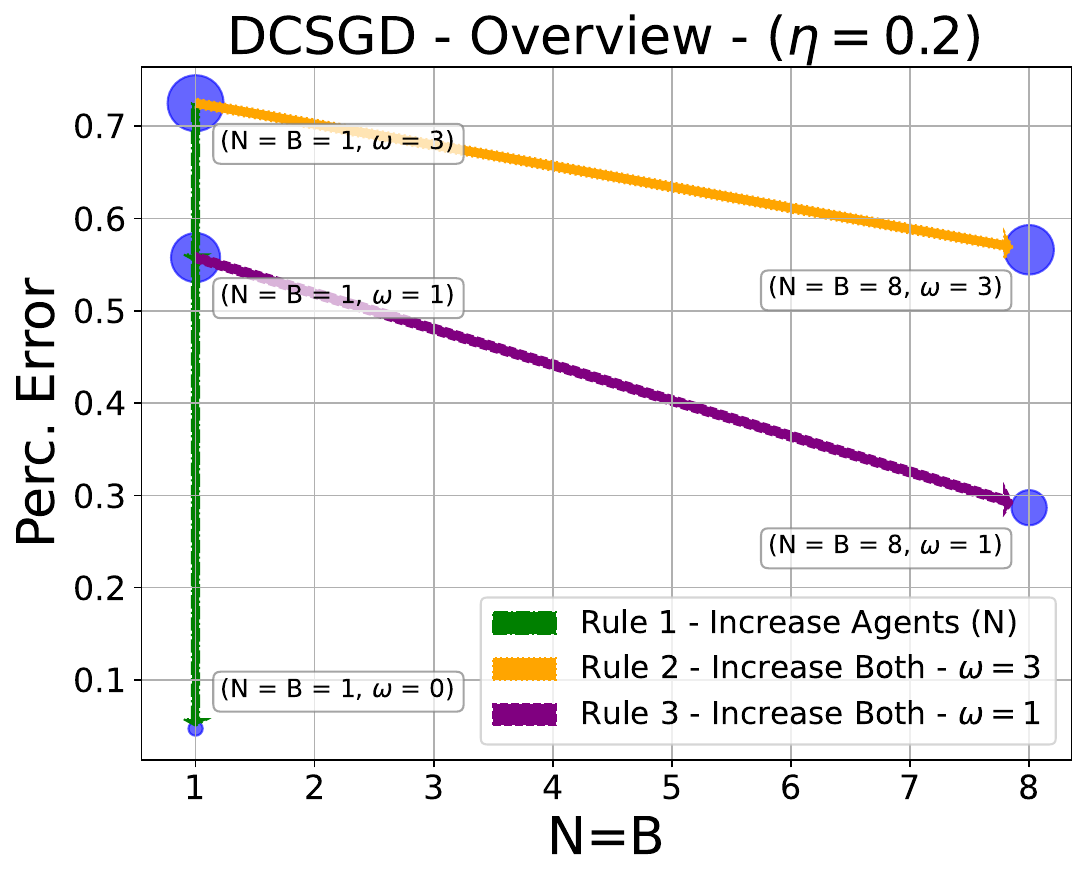} }}
    \caption{Validation of Scaling Rules for DCSGD on a 124M GPT2 model.}
    \label{fig:DCSGD_GPT}
\end{figure}

\end{document}